\documentclass[twoside]{article}

\usepackage{fullpage}
\usepackage[utf8]{inputenc} 
\usepackage[T1]{fontenc}    

\usepackage{hyperref}       
\usepackage{url}            
\usepackage{booktabs}       
\usepackage{amsfonts}       
\usepackage{nicefrac}       
\usepackage{microtype}      
\usepackage{times}

\usepackage[english]{babel}
\usepackage{amsmath}
\usepackage{amsthm}
\usepackage{amssymb}
\usepackage{thm-restate}
\usepackage{graphicx}
\usepackage{color}
\usepackage{nicefrac}
\usepackage[boxed,vlined]{algorithm2e}
\usepackage{algpseudocode}
\usepackage{comment}
\usepackage{enumitem}
\usepackage{framed}
\usepackage{subfigure}
\usepackage{wrapfig}
\usepackage{breakurl}
\usepackage{titlesec}
\usepackage[labelfont=bf]{caption}
\usepackage{color} 
\usepackage[numbers]{natbib}

\usepackage{mdwlist}

\usepackage{tikz}

\renewcommand \paragraph [1]{{\noindent\bf #1:}}

\newcommand{\edited}[1]{{\color{red}}}

\usetikzlibrary{positioning,chains,fit,shapes,calc}

\newtheorem{theorem}{Theorem}

\newtheorem{lemma}[theorem]{Lemma}

\newtheorem{definition}{Definition}

\makeatletter

\newcommand \reals {\mathbb{R}}
\newcommand \NN {\operatorname{NN}}
\newcommand \norm [1]{\Vert #1 \Vert}

\newcommand \setc [2]{\{#1 \,:\, #2\}}
\newcommand \expect {\operatorname*{\mathbb{E}}}
\newcommand \prob {\operatorname*{\mathbb{P}}}

\newcommand \ind [1]{\mathbb{I}\{#1\}}
\newcommand \X {\mathcal{X}}

\newcommand \pdist {\mu}

\author{
Travis Dick \\
Carnegie Mellon University\\
\texttt{tdick@cs.cmu.edu}
\and
Mu Li \\
Carnegie Mellon University\\
\texttt{muli@cs.cmu.edu}
\and
Venkata Krishna Pillutla \\
University of Washington\\
\texttt{pillutla@cs.washington.edu}
\and
Colin White \\
Carnegie Mellon University\\
\texttt{crwhite@cs.cmu.edu}
\and
Maria Florina Balcan \\
Carnegie Mellon University\\
\texttt{ninamf@cs.cmu.edu}
\and
Alex Smola \\
Carnegie Mellon University\\
and AWS Deep Learning\\
\texttt{alex@smola.org}
}
\date{}

\begin{document}

\title{Data Driven Resource Allocation for Distributed Learning%
\footnote{
This work was supported in part by NSF grants CCF-1451177,
CCF-1422910, CCF-1535967, IIS-1618714, IIS-1409802, a Sloan Research
Fellowship, a Microsoft Research Faculty Fellowship, a Google Research
Award, Intel Research, Microsoft Research, and a National Defense Science
\& Engineering Graduate (NDSEG) fellowship.
}}

\maketitle

\begin{abstract}
In distributed machine learning, data is dispatched to multiple
machines for processing.
Motivated by the fact that similar data points often belong to the same
or similar classes, and more generally, classification rules of high accuracy
tend to be ``locally simple but globally complex''~\citep{VB-local}, we propose
data dependent dispatching that takes advantage of such structure.
We present an in-depth analysis of this model, providing new algorithms with
provable worst-case guarantees, analysis proving existing scalable heuristics
perform well in natural non worst-case conditions, and techniques for extending
a dispatching rule from a small sample to the entire distribution.
We overcome novel technical challenges
to satisfy important conditions for accurate distributed learning,
including fault tolerance and balancedness.
We empirically compare our approach with baselines based on random
partitioning, balanced partition trees, and locality sensitive
hashing, showing that we achieve significantly higher accuracy on
both synthetic and real world image and advertising datasets. We
also demonstrate that our technique strongly scales with the
available computing power.
\end{abstract}

\section{Introduction}
\label{sec:intro}
\paragraph{Motivation and Overview}
We consider distributed learning settings where massive amounts of
data are collected centrally, and for space and efficiency reasons
this data must be dispatched to distributed machines in order to
perform the processing needed~\citep{bbfm12,disClustering13,disPCA14,alexps,zdjw12,zdjw13}. The simplest
approach and the focus of past work (both theoretical and empirical) is to perform the dispatching
randomly~\citep{zdjw12,zdjw13}. Random dispatching has the advantage
that dispatching is easy, and because each machine receives data from
the same distribution, it is rather clean to analyze theoretically.

\begin{figure}
  \centering
\includegraphics[width=0.3\textwidth]{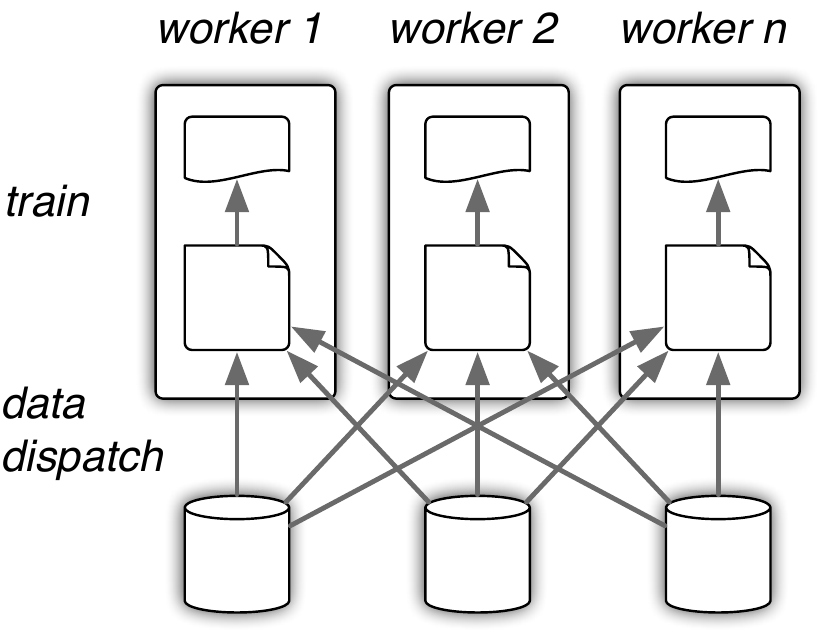}
\caption{Data is partitioned and dispatched to multiple workers. Each
  worker then trains a local model using its local data. There is no
  communication between workers during training.}
\vspace{-2ex}
\label{mu}
\end{figure}
However, since the
distributions of the data on each machine are
identical, such techniques could lead to sub-optimal results in
practice in terms of the accuracy of the resulting learning rule.
Motivated by the fact that in practice, similar data points tend to
have the same or similar classification, and more generally,
classification rules of high accuracy tend to be ``locally simple but
globally complex''~\citep{VB-local}, we propose a new paradigm for
performing {\em data-dependent dispatching} that takes advantage of
such structure by sending similar datapoints to the same machines.
For example, a {\em globally} accurate classification rule may be
complicated, but each machine can accurately classify its {\em local}
region with a simple classifier.

We introduce and analyze dispatching techniques that partition a set
of points such that similar examples end up on the same
machine/worker, while satisfying key constraints present in a real
world distributed system including balancedness and
fault-tolerance. Such techniques can then be used within a simple, but
highly efficient distributed system that first partitions a small
initial segment of data into a number of sets equal to the number of
machines. Then each machine locally and independently applies a
learning algorithm, with no communication between the workers during
training. In other words, the learning is embarrassingly parallel. 
See Figure~\ref{mu} for a schematic representation.
At prediction time, we use a super-fast sublinear algorithm for
directing new data points to the most appropriate machine.

\paragraph{Our Contributions}
We propose a novel scheme for partitioning data which leads to better
accuracy in distributed machine learning tasks, and we give a
theoretical and experimental analysis of this approach.  We present
new algorithms with provable worst-case guarantees, analysis proving
existing scalable heuristics perform well in natural non worst-case
conditions, techniques for extending a dispatching rule from a small
sample to the entire distribution, and an experimental evaluation of
our proposed algorithms and several baselines on both synthetic and
real-world image and advertising datasets.  We empirically show that
our method strongly scales and that we achieve significantly higher
accuracy over baselines based on random partitioning, balanced
partition trees, and locality-sensitive hashing.

In our framework,
a central machine starts by clustering a small sample of data into
roughly equal-sized clusters, where the number of clusters is equal to
the number of available machines.
Next, we extend this clustering into an efficient dispatch rule that
can be applied to new points.
This dispatch rule is used to send the remaining training data to the
appropriate machines and to direct new points at prediction time.
In this way, similar datapoints wind up on the same machine.
Finally, each machine independently learns a classifier using its own data
(in an embarrassingly parallel manner).
To perform the initial clustering used for dispatch,
we use classic clustering objectives
($k$-means, $k$-median, and $k$-center).
However, we need to add novel constraints
to ensure that the clusters give a data partition that respects the
constraints of real distributed learning systems:

\emph{Balancedness:} We need to ensure our dispatching procedure
balances the data across the different machines.
If a machine receives much more data than other machines, then it will
be the bottleneck of the algorithm.
If any machine receives very little data, then its processing power is
wasted.
Thus, enforcing upper and lower bound constraints on the cluster sizes
leads to a faster, more efficient setup.

\emph{Fault-Tolerance:}
In order to ensure that our system is robust to machine failures, we assign each point to multiple distinct clusters.
This way, even if a machine fails, the data on that machine is still present on other machines.
Moreover, this has the added benefit that
our algorithms behave well on points near the boundaries of the clusters.
We say a clustering algorithm satisfies $p$-replication if each point is assigned to $p$ distinct clusters.

\emph{Efficiency:}
To improve efficiency, we apply our clustering algorithms to a small sample of data.
Therefore, we need to be able to extend the clustering to new examples from the same distribution while maintaining a good objective value and satisfying all constraints.
The extension technique should be efficient for both the initial partitioning, and dispatching at prediction time.

When designing clustering algorithms,
adding balancedness and fault tolerance makes the task significantly harder.
Prior work has considered upper bounds on the cluster sizes
\citep{ckc3,byrka2015bi,Cygan_lprounding,Khuller96thecapacitated,li2014,li2014approximating}
and lower bounds \citep{aggarwal2006achieving,clustering-lb},
but no prior work has shown provable guarantees with
upper and lower bounds on the cluster sizes simultaneously.
\footnote{
Note that enforcing only upper (resp.\ lower) bounds
implies a weak lower (resp.\ upper) bound on the cluster sizes,
but this is only nontrivial if the upper (resp.\ lower) bounds are extremely tight or
the number of clusters is a small constant.}
With upper bounds, the objective functions are nondecreasing as the number of clusters $k$
increases, but with lower bounds we show the objective function can oscillate arbitrarily with respect to $k$.
This makes the problem especially challenging from a combinatorial optimization perspective.
Existing capacitated clustering algorithms work by rounding a fractional linear program solution, but the erratic
nature of the objective function makes this task more difficult for us.

The balance constraints also introduce challenges when extending a 
clustering-based partitioning from a small sample to unseen data.
The simple rule that assigns a new point to the cluster with the nearest 
center provides the best objective value on new data,
but it can severely violate the balance constraints.
Therefore, any balanced extension rule must take into account the 
distribution of data.

We overcome these challenges, presenting a variety of complementary
results, which together provide strong justification for our
distributed learning framework.  We summarize each of our main results
below.

$\bullet$ \textbf{Balanced fault-tolerant clustering:} We provide the first clustering algorithms
with provable guarantees that simultaneously handle
upper and lower bounds on the cluster sizes, as well as fault tolerance.
Clustering is NP-hard and adding more constraints makes it significantly harder,
as we will see in Section \ref{sec:structure}.
For this reason, we first devise approximation algorithms with strong worst-case
guarantees, demonstrating this problem is tractable.
Specifically, in Section~\ref{sec:bicriteria} we provide an
algorithm that produces a fault-tolerant clustering that approximately optimizes
$k$-means, $k$-median, and $k$-center objectives while also roughly satisfying
the given upper and lower bound constraints.  At a high level, our algorithm
proceeds by first solving a linear program,
followed by a careful balance and replication aware rounding scheme.
We use a novel min-cost flow technique to finish off rounding the
LP solution into a valid clustering solution.

$\bullet$ \textbf{$k$-means++ under stability:} 
We give complementary results showing that for `typical' problem instances, it
is possible to achieve better guarantees with simpler, more scalable
algorithms.  Specifically, in Section~\ref{sec:stability} we show the
popular $k$-means++ algorithm outputs a balanced clustering with
stronger theoretical guarantees, provided the data satisfies a natural
notion of stability.  We make nontrivial extensions of previous work
to ensure the upper and lower size constraints on the clusters are
satisfied.  No previous work gives provable guarantees while
satisfying both upper and lower bounds on the cluster sizes, and
Sections~\ref{sec:bicriteria} and \ref{sec:stability} may be of
independent interest beyond distributed learning.

$\bullet$ \textbf{Structure of balanced clustering:} 
We show that adding lower bound constraints on the cluster sizes makes
clustering objectives highly nontrivial. Specifically, we show that
for $k$-means, $k$-median, and $k$-center, the objective values may
oscillate arbitrarily with respect to $k$.
In light of this structure,
our aforementioned algorithmic results are more surprising, since
it is not obvious that algorithms with constant-factor guarantees exist.

$\bullet$ \textbf{Efficient clustering by subsampling:} For datasets
large enough to require distributed processing, clustering the entire
dataset is prohibitively expensive. A natural way to avoid this cost
is to only cluster a small subset of the data and then efficiently
extend this clustering to the entire dataset.
In Section~\ref{sec:sampleComplexity} we show that assigning a new
example to the same $p$ clusters as its nearest neighbor in the
clustered subsample approximately preserves both the objective value
and all constraints. We also use this technique at prediction time to
send new examples to the most appropriate machines.

$\bullet$ \textbf{Experimental results:} We conduct experiments with
both our LP rounding algorithms and $k$-means++ together with our
nearest neighbor extension technique. 
We include empirical (and theoretical) 
comparisons which show the effectiveness of both algorithms in different situations. 
The $k$-means++ algorithm is competitive on real world image and
advertising datasets, complementing the results of
Section~\ref{sec:stability} by showing empirically that
$k$-means++ produces high-quality balanced clusterings for `typical'
datasets. 
We then compare the performance of our framework (using
$k$-means++ with nearest neighbor extension) against three baseline
methods (random partitioning, balanced partition trees, and locality
sensitive hashing) in large scale learning experiments where each
machine trains an SVM classifier. We find that for all datasets and
across a wide range of $k$ values, our algorithm achieves higher
accuracy than any of the baselines. Finally, we show that our
technique strongly scales, meaning that doubling the available
computational power while keeping the workload fixed reduces the
running time by a constant factor, demonstrating that our method can
scale to very large datasets.

\paragraph{Related Work}
Currently, the most popular method of dispatch in distributed
learning is random dispatch~\citep{zdjw12,zdjw13}.
This may not produce optimal results because each
machine must learn a global model. Another notion is to dispatch the data to
pre-determined locations e.g., Yahoo!'s geographically distributed database,
PNUTS~\citep{pnuts}. However, it does not look at any properties of the data
other than physical location.

In a recent paper, \citet{wei2015mixed} study partitioning for
distributed machine learning, however, they give no formal guarantees
on balancing the data each machine receives. \citet{you2015} use
$k$-means clustering to distribute data for parallel training of
support vector machines, but their clustering algorithms do not have
approximation guarantees and are applied to the entire dataset, so
their clustering step is much slower than ours.  There is also work on
distributed graph partitioning
\citep{aydin2016distributed,bourse2014balanced,delling2011graph}, in
which the data points are set up in a graph structure, and must be
distributed to different machines, minimizing the number of edges
across machines.  These techniques do not apply more generally for non
graph-based objectives, e.g.\ $k$-means, $k$-median, or $k$-center.

Previous work in theoretical computer science has considered capacitated clustering,
or clustering with upper bounds
\citep{ckc3,byrka2015bi,Cygan_lprounding,Khuller96thecapacitated,li2014,li2014approximating},
and lower bounds \citep{aggarwal2006achieving,clustering-lb},
but our algorithm is the first to solve a more general and challenging question of simultaneously 
handling upper and lower bounds on the cluster sizes, and $p$-replication.
See Section \ref{app:intro} in the Appendix for a more detailed discussion about related work.


\section{Fault Tolerant Balanced Clustering} \label{sec:bicriteria}

In this section, we give an algorithm to cluster a small initial
sample of data to create a dispatch rule that sends similar points to
the same machine.
There are many ways to measure the similarity of points in the same
cluster. We consider three classic clustering objectives
while imposing upper and lower
bounds on the cluster sizes and replication constraints.
It is well-known that solving the objectives optimally are NP-hard, even 
without the capacity and fault tolerance generalizations \cite{jain2003greedy}.
In Section \ref{sec:structure}, we show that the objectives with balance
constraints behave erratically with respect to the number of clusters $k$,
in particular, there may exist an arbitrary number of local minima and maxima.
In light of this difficulty, one might ask whether any
approximation algorithm is possible for this problem.  We
  answer affirmatively, by extending previous work \citep{li2014}
to fit our more challenging constrained
    optimization problem.  Our algorithm returns a clustering whose
cost is at most a constant factor multiple of the optimal solution,
while violating the capacity and replication
  constraints by a small constant factor.
This is the first algorithm with provable guarantees 
to simultaneously handle both upper and lower bounds on the cluster sizes.

	\begin{theorem}\label{thm:bicriteria}
Algorithm \ref{alg:overview} returns a constant factor approximate solution for the
balanced $k$-clustering with $p$-replication problem for $p> 1$, where the upper capacity constraints are violated
by at most a factor of $\frac{p+2}{p}$,
and each point can be assigned to each center at most twice.
\end{theorem}

A clustering instance consists of a set $V$ of $n$ points, and a distance
metric $d$.
Given two points $i$ and $j$ in $V$,
denote the distance between $i$ and $j$ by $d(i,j)$.
The task is to find a set of $k$ centers $C=\{c_1,\dots,c_k\}$
and assignments of each point to $p$ of the centers
 $f: V \rightarrow {{C}\choose{p}}$,
where ${C}\choose{p}$ represents the subset of $C^p$
with no duplicates.
In this paper, we study three popular clustering objectives:

\textit{(1)}~$k$-median: $\min_{C,f} \sum_{i\in V} \sum_{j \in f(i)}  d(i,j)$

\textit{(2)}~$k$-means: $\min_{C,f} \sum_{i\in V} \sum_{j \in f(i)}  d(i,j)^2$

\textit{(3)}~$k$-center: $\min_{C,f} \max_{i\in V}\max_{j\in f(i)} d(i,j)$

In this section, 
we focus on the first two, and put the details for $k$-center in the
appendix.
We add size constraints $0 < \ell \leq L < 1$, also known as capacity
constraints, so each cluster must have a size between $n\ell$ and
$nL$.  For simplicity, we assume these values are integral (or 
replace them by $\lceil n\ell \rceil$ and $\lfloor nL \rfloor$
respectively).


At a high level, our algorithm proceeds by first solving a linear
program, followed by careful rounding.  
In particular, we set up an LP whose optimal integral solution is the optimal clustering. We can
use an LP solver which will give a fractional solution (for example, the LP may open up $2k$ ‘half’ centers).
Then, using a greedy procedure from \citet{charikar1999}, we pick $\leq k$ points (called the `monarchs’) 
which are spread out. Furthermore, the distance from a non-monarch to its closest monarch is a constant-factor
multiple of the non-monarch's connection cost in the LP solution.
The empire of a monarch is defined to be its cell in the Voronoi
partition induced by the monarchs. 
By a Markov inequality, every empire has $\geq p/2$ total fractional centers, 
which is at least one for $p \geq 2$. Then we merely open the
innermost points in the empires as centers, ending with $\leq k$ centers. Once we have the centers, we find the
optimal assignments by setting up a min-cost flow problem. 

The key insight is that
$p$-replication helps to mitigate the capacity violation in the
rounding phase.  Together with a novel min-cost flow technique, this
allows us to simultaneously handle upper and lower bounds on the
cluster sizes.  The procedure is summarized in Algorithm
\ref{alg:overview}, and below we provide details, together with the
key ideas behind its correctness (see Appendix
\ref{app:bicriteria} for the
full details).

\noindent\textbf{Step 1: Linear Program}
The first step is to solve a linear program (LP)
  relaxation of the standard integer program (IP) formulation of our
  constrained clustering problem. The variables are as follows: for
  each $i\in V$, let $y_i$ be an indicator for whether $i$ is opened
  as a center. For $i,j \in V$, let $x_{ij}$ be an indicator for
  whether point $j$ is assigned to center $i$. In the LP, the
  variables may be fractional, so $y_i$ represents the fraction to
  which a center is opened (we will refer to this as the ``opening''
  of $i$), and $x_{ij}$ represents the fractional assignment of $j$ to
  $i$. One can use an LP solver to get a fractional solution which
  must then be rounded. Let $(x,y)$ denote an optimal solution to the LP.  For any
  points $i$ and $j$, let $c_{ij}$ be the cost of assigning point $j$
  to center $i$. That is, for $k$-median, $c_{ij} = d(i, j)$, and for
  $k$-means $c_{ij} = d(i, j)^2$ (we discuss $k$-center in the
  appendix).  Define $C_j=\sum_i c_{ij}x_{ij}$, the
  average cost from point $j$ to its centers in the LP solution
  $(x,y)$.

It is well-known that the LP in Algorithm \ref{alg:overview} has an unbounded integrality gap
(the ratio of the optimal LP solution over the optimal
integral LP solution), even when the capacities are violated by
a factor of $2-\epsilon$ \citep{li2014}. However,
with fault tolerance, the integrality is only unbounded when the capacities are violated by
a factor of $\frac{p}{p-1}$ (see the appendix for the integrality gap).
Intuitively, this is because the $p$ centers can `share' this
violation.

\vspace{0.5em}
\setcounter{figure}{0}
\begin{figure}[htb]
  \begin{framed}
  \begin{enumerate}[noitemsep,nolistsep,leftmargin=*]
  \item Find a solution to the following linear program:
  \vspace{-0.6em}
  \begin{alignat*}{3}
  \operatorname*{min}_{x,y}\sum_{i,j\in V} c_{ij} x_{ij}
  \quad\text{s.t.} 
  &\quad\textbf{(a)}\,\forall j \in V: \sum_{i\in V} x_{ij}=p;
  \quad \quad \quad \quad \,\,\, \textbf{(b)}\,\sum_{i\in V} y_i\leq k; \\
  &\quad\textbf{(c)}\,\forall i \in V: \ell y_i \leq \sum_{j\in V} \frac{x_{ij}}{n} \leq L y_i;
  \quad \textbf{(d)}\,\forall i,j \in V: 0\leq x_{ij}\leq y_i\leq 1.
\end{alignat*}
\vspace{-1.5em}
\item Greedily place points into a set $\mathcal{M}$ from lowest $C_j$
  to highest (called ``monarchs''), adding point $j$ to $\mathcal{M}$
  if it is not within distance $4 C_j$ of any monarch.  Partiton the
  points into coarse clusters (called ``empires'') using the Voronoi
  partitioning of the monarchs.
\item For each empire $\mathcal{E}_u$ with total fractional opening
  $Y_u \triangleq \sum_{i\in\mathcal{E}_u} y_i$, give opening
  $\nicefrac{Y_u}{\lfloor Y_u \rfloor}$ to the $\lfloor Y_u \rfloor$
  closest points to $u$ and all other points opening 0.
\item Round the $x_{ij}$'s by constructing a minimum cost flow problem
  on a bipartite graph of centers and points, setting up demands and
  capacities to handle the bounds on cluster sizes.
\end{enumerate}
\end{framed}
\vspace{-1em}
\renewcommand {\figurename} {Algorithm}
\caption{Balanced clustering with fault tolerance}
\label{alg:overview}

\end{figure}
\setcounter{figure}{1}
\setcounter{algocf}{1}
\vspace{1em}

\noindent\textbf{Step 2: Monarch Procedure}
Next, partition the points into ``empires'' such
that every point is $\leq 4C_j$ from the center of its empire (the ``monarch'')
by using a greedy procedure from \citet{charikar1999} (for an informal description, see step 2 of
Algorithm \ref{alg:overview}).
By Markov's inequality, every empire has total opening $\geq p/2$, which is crucially $\geq 1$ for
$p\geq 2$ under our model of fault tolerance.
We obtain the following guarantees, and we show a proof sketch for
the $k$-median objective, with the full proof in Appendix \ref{app:bicriteria}.

\begin{lemma} \label{lem:monarch_bicriteria}
The output of the monarch procedure satisfies the following properties:
\begin{enumerate}[label=(1\alph*)]
  \item \label{lemsub:partition} The clusters partition the point set;
  \item \label{lemsub:close1} Each point is close to its monarch: $\forall j\in
  \mathcal{E}_u, u\in \mathcal{M}, c_{uj}\leq 4 C_j$;
  \item \label{lemsub:n'far} Any two monarchs are far apart: $\forall u,u' \in
  \mathcal{M}\text{ s.t. } u\neq u', c_{uu'}>4\max\{C_u,C_{u'}\}$;
  \item \label{lemsub:min} Each empire has a minimum total opening: $\forall
  u\in \mathcal{M}, \sum_{j\in \mathcal{E}_u} y_j\geq \frac{p}{2}$.
\end{enumerate}

\end{lemma}

\begin{proof}[Proof sketch.]
The first three properties follow easily from construction (for
property~\ref{lemsub:n'far}, recall we greedily picked monarchs by the value of $C_j$).
For the final property, note that for some $u\in \mathcal{M}$,
if $d(i,u)\leq 2C_u$, then $i\in \mathcal{E}_u$ (from the triangle
inequality and property~\ref{lemsub:n'far}).
Now, note that $C_u$ is a weighted average of costs $c_{iu}$ with weights $x_{iu}/p$,
i.e., $C_u=\sum_i c_{iu}\nicefrac{x_{iu}}{p}$. By Markov's inequality, in any weighted average,
values greater than twice the average have to get less than half the total weight. That is,
\begin{equation*}
    \sum_{j:c_{ju}>2C_u}\frac{x_{ju}}{p} < \sum_{j:c_{ju}>2C_u} \frac{x_{ju}}{p}\cdot\frac{c_{ju}}{2C_u} <
    \frac{C_u} { 2 C_u} = \frac 1 2
\end{equation*}
Combining these two facts, for each $u \in \mathcal{M}$:
\begin{equation*}
    \sum_{j \in \mathcal{E}_u} y_j \geq \sum_{j: c_{ju} \leq 2 C_u} y_j \geq \sum_{j: c_{ju} \leq 2 C_u} x_{ju} \geq \frac p 2 .\quad\qedsymbol
\end{equation*}
\renewcommand{\qedsymbol}{}
\end{proof}

\noindent\textbf{Step 3: Aggregation}
The point of this step is to end up with $\leq k$ centers total.
Since each empire has total opening at least 1, we can aggregate openings within
each empire.
For each empire $\mathcal{E}_u$, we move the openings to the $\lfloor Y_u\rfloor$
innermost points of $\mathcal{E}_u$, where $Y_u=\sum_{i\in \mathcal{E}_u} y_i$.
This shuffling is accomplished by greedily calling a suboperation called \emph{Move},
which is the standard way to transfer openings between points to maintain all LP constraints
\citep{li2014}.
To perform a \emph{Move} from $i$ to $j$ of size $\delta$, set $y_i'=y_i-\delta$ and $y_j'=y_j+\delta$,
and change all $x$'s so that the fractional demand switches from $i$ to $j$: $\forall u\in V$, $x_{iu}'=x_{iu}(1-\nicefrac{\delta}{y_i})$
and similarly increase the demand to all $x_{ju}$.
The \emph{Move} operation preserves all LP constraints, except
we may violate the capacity constraints if we give a center an opening greater than one.

In each empire $\mathcal{E}_u$, start with the point $i$ with nonzero $y_i$
that is farthest away from the monarch $u$. Move its opening to the monarch $u$,
and then iteratively continue with the next-farthest point in $\mathcal{E}_u$ with nonzero opening.
Continue this process until $u$ has opening exactly $\frac{Y_u}{\lfloor Y_u\rfloor}$, and then start moving the 
farthest openings to the
point $j$ closest to the monarch $u$. Continue this until the $\lfloor Y_u\rfloor$ closest points to $u$
all have opening $\frac{Y_u}{\lfloor Y_u\rfloor}$.
Call the new variables $(x',y')$. They have the following properties.

\begin{lemma} \label{lem:step1}
The aggregated solution $(x', y')$ satisfies the following constraints:
\begin{enumerate}[label=(2\alph*)]
  \item \label{eq:lem:step1:a} The opening of each point is either zero or in
  $[1, \frac{p+2}{2}]$: $\forall i\in V,~1\leq y'_i<\frac{p+2}{p} \text{ or
  }y'_i=0$;
  \item \label{eq:lem:step1:b} Each cluster satisfies the capacity constraints:
  $i\in V,~ \ell y'_i\leq \sum_{j\in V}\frac{x'_{ij}}{n}\leq L y'_i$;
  \item \label{eq:lem:step1:c} The total fractional opening is $k$:
  $\sum_{i\in V} y'_i=k$;
  \item \label{eq:lem:step1:d} Points are only assigned to open centers:
  $\forall i,j\in V,~x'_{ij}\leq y'_i$;
  \item \label{eq:lem:step1:e} Each point is assigned to $p$ centers: $\forall
  i\in V,~\sum_j x'_{ji}=p$;
  \item \label{eq:lem:step1:f} The number of points with non-zero opening is at
  most $k$: $|\{i\mid y'_i>0\}|\leq k$. \end{enumerate}
\end{lemma}

\begin{proof}
For the first property, recall that each cluster $\mathcal{E}_u$ has total opening $\geq \frac{p}{2}$,
so by construction, all $i$ with nonzero $y'_i$ has $y'_i\geq 1$.
We also have $\frac{Y_u}{\lfloor Y_u\rfloor}\leq
\frac{\lfloor Y_u\rfloor+1}{\lfloor Y_u\rfloor}\leq
\frac{p+2}{p}$,
which gives the desired bound.

The next four properties are checking that the
LP constraints are still satisfied (except for $y'_i\leq 1$).
These follow from the fact that
\emph{Move} does not violate the constraints.
The last property is a direct result of Properties~\ref{eq:lem:step1:a} and \ref{eq:lem:step1:c}.
\end{proof}

We obtain the following guarantee on the moving costs.

\begin{lemma}\label{lem:costs}
$\forall j\in V \text{ whose opening moved from }i'\text{ to }i$,
\begin{itemize}[itemsep=3pt,topsep=2pt,parsep=0pt,partopsep=0pt]
\item $k$-median: $d(i,j)\leq 3d(i',j)+8 C_j$,
\item $k$-means:    $d(i,j)^2\leq 15d(i',j)^2+80C_j$.
\end{itemize}
\end{lemma}

\begin{proof}
By construction, if the demand of point $j$ moved from $i'$ to $i$,
then $\exists u\in \mathcal{M}$ s.t. $i,i'\in \mathcal{E}_u$ and $d(u,i)\leq d(u,i')$.
Denote $j'$ as the closest point in $\mathcal{M}$ to $j$. Then $d(u,i')\leq d(j',i')$ because $i'\in \mathcal{E}_u$.
Then,
\begin{align*}
d(i,j)&\leq d(i,u)+d(u,i')+d(i',j) \\
&\leq 2d(u,i')+d(i',j) \\
&\leq 2d(j',i')+d(i',j) \\
&\leq 2(d(j',j)+d(j,i'))+d(i',j) \\
&\leq 8 C_j + 3d(i',j).
\end{align*}
\end{proof}

We include the proof for $k$-means in Appendix \ref{app:bicriteria}.

\noindent\textbf{Step 4: Min cost flow}
Now we must round the $x$'s.
We set up a min cost flow problem,
where an integral solution corresponds to an assignment of points
to centers.
We create a bipartite graph with $V$ on the left (each with supply $p$)
and the $k$ centers on the right (each with demand $n\ell$),
and a sink vertex with demand $np-kn\ell$.
We carefully set the edge weights so that the minimum cost
flow that satisfies the capacities corresponds to an optimal
clustering assignment. 
See Figure \ref{fig:flow1x}.

\begin{figure*}
\centering
\includegraphics[scale=0.8]{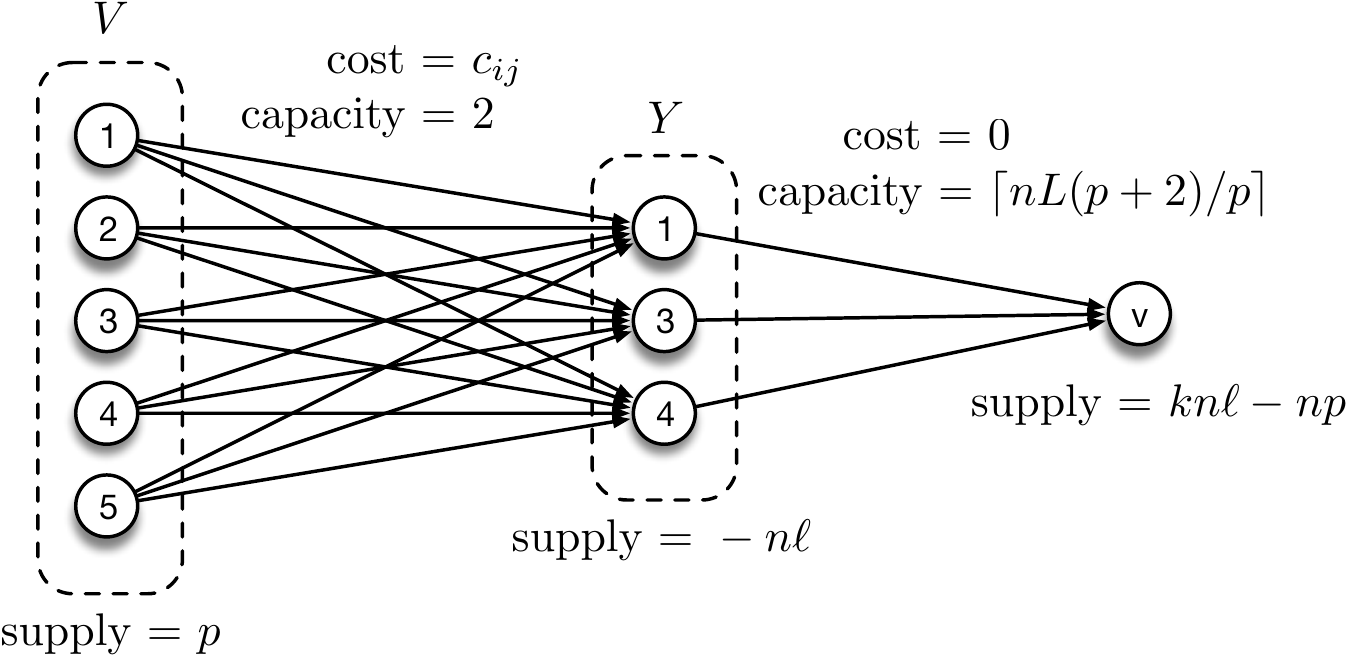}
\caption{{\bf Flow network for rounding the $x$'s}: The nodes in each group all
have the same supply, which is indicated below each group. The edge costs and
capacities are shown above each group. The $y$-rounded solution gives a feasible
flow in this network. By the Integral Flow Theorem, there exists a minimum cost
flow which is integral and we can find it in polynomial time.}
\label{fig:flow1x}
\end{figure*}

Then using the Integral Flow Theorem,
we are guaranteed there is an \emph{integral} assignment
that achieves the same optimal cost
(and finding the min cost flow is a well-studied polynomial time
problem \citep{papadimitriou1998combinatorial}). Thus, we can round the
$x$'s without incurring any additional cost to the approximation factor.
This is the first time this technique has been used in the setting of clustering.

In Section \ref{sec:kcenter} of the appendix, we show a more involved algorithm specifically
for $k$-center which achieves a 6-approximation with \emph{no violation} to the capacity or replication
constraints.

\section{Balanced Clustering under Stability}\label{sec:stability}

In the previous section, we showed an LP-based algorithm which provides theoretical guarantees even
on adversarially chosen data. Often real-world data has inherent structure that allows
us to use more scalable algorithms and achieve even better clusters \citep{as,ostrovsky2006effectiveness}.
In our distributed ML framework, this translates to being able to use a 
larger initial sample for the same computational power 
(Section~\ref{sec:sampleComplexity} analyzes the effect of sample size). 
In this section, we prove the popular $k$-means++ algorithm as well as a greedy thresholding algorithm
output clusters very close to the optimal solution, provided the data satisfies a natural notion of stability
called \emph{approximation stability} \citep{agarwal2015k,balcan2010approximate,as,kcenter-pr,Tim-social-net}.

Specifically,
we show that given a balanced clustering instance in which clusterings close in \emph{value} to $\mathcal{OPT}$
are also close in terms of the clusters themselves, assuming $L\in O(\ell)$,
then $k$-means++ with a simple pruning step \citep{ostrovsky2006effectiveness} outputs a solution close to optimal.
For the $k$-median objective, we show that under the same stability condition,
Balcan et al.'s greedy thresholding algorithm outputs a solution that is close to optimal \citep{as}.
For the $k$-center objective, we show a simple thresholding algorithm is sufficient to
optimally cluster the data, and this result is tight up to the level
of stability. That is,
assuming any strictly weaker version of stability makes the problem NP-hard.
We overcome key challenges that arise when we add upper and lower bounds to the cluster sizes.
We present the intuition here, and give the full 
details in Appendix \ref{app:stability}.

\paragraph{Approximation Stability}
Given a clustering instance $(S,d)$ and inputs $\ell$ and $L$,
let $\mathcal{OPT}$ denote the cost of the optimal balanced clustering.
Two clusterings $\mathcal{C}$ and $\mathcal{C}'$ are \emph{$\epsilon$-close},  if only an
$\epsilon$-fraction of the input points are clustered differently in the two clusterings,
i.e.,\ $\min_\sigma \sum_{i=1}^k |C_i \setminus C'_{\sigma(i)} | \leq \epsilon n$, where
$\sigma$ is a permutation of $[k]$.

\begin{definition}[\citet{as}]
A clustering instance $(S,d)$ satisfies \emph{$(1+\alpha,\epsilon)$-approximation stability}
with respect to balanced clustering if all clusterings $\mathcal{C}$ with
$\text{cost}(\mathcal{C})\leq (1+\alpha)\cdot \mathcal{OPT}$ are $\epsilon$-close to $\mathcal{C}$.
\end{definition}

\paragraph{$k$-means}
We show that sampling $k\log k$ centers using $k$-means++, followed by a greedy center-pruning step,
(introduced by \citet{ostrovsky2006effectiveness})
is sufficient to cluster well with high probability, assuming
$(1+\alpha,\epsilon)$-approximation stability for balanced clustering.
Our results improve over \citet{agarwal2015k}, who showed this algorithm outputs a good clustering with
probability $\Omega(\frac{1}{k})$ for standard (unbalanced) clustering under approximation stability.
Formally, our result is the following.

\begin{theorem} \label{thm:k-means++}
For any $\frac{\epsilon\cdot k}{\alpha}<\rho<1$,
$k$-means++ seeding with a greedy pruning step outputs a solution that is $\frac{1}{1-\rho}$ close to the optimal solution
with probability $>1-O(\rho)$, for clustering instances satisfying $(1+\alpha,\epsilon)$-approximation stability
for the balanced $k$-means objective, with $\frac{L}{\ell}\in O(1)$.
\end{theorem}

Intuitively, $(1 + \alpha,\epsilon)$-approximation stability forces the clusters
to become ``spread out'', i.e., the radius of any cluster is much smaller than
the inter-cluster distances.
This allows us to show for 2-means clustering, the $k$-means++ seeding procedure
will pick one point from each cluster with high probability.
We defer some proofs to Appendix \ref{app:stability}.
Given a point set $S$, let $\Delta_k(S)$ denote the optimal $k$-means cost of $S$.

\begin{lemma}[\citep{ostrovsky2006effectiveness}] \label{lem:mass}
Given a set $S\subseteq\mathbb{R}^d$ and any partition $S_1\cup S_2$ of $S$ with $S_1\neq\emptyset$. Let $s$,$s_1$,$s_2$ denote
the centers of mass of $S$, $S_1$, and $S_2$, respectively. Then
\begin{enumerate}
\item $\Delta_1^2(S)=\Delta_1^2(S_1)+\Delta_1^2(S_2)+\frac{|S_1||S_2|}{|S|}\cdot d(s_1,s_2)^2$
\item $d(s_1,s)^2\leq\frac{\Delta_1^2(S)}{|S|}\cdot\frac{|S_2|}{|S_1|}.$
\end{enumerate}
\end{lemma}

We define $r_i$ as the average radius of cluster $C_i$, i.e.\ $r_i=\frac{1}{|C_i|}\sum_{x\in C_i}d(x,c_i)$.
Given a clustering instance $S$ satisfying $(1 + \alpha,\epsilon)$-approximation stability,
with upper and lower bounds $L$ and $\ell$ on the cluster sizes. We assume that $L\in O(\ell)$.
For convenience, let $|C_i|=n_i$ for all $i$.

\begin{lemma} \label{lem:r_i}
$\max(r_1^2,r_2^2)\leq O(\frac{\epsilon}{\alpha}\cdot\frac{L}{\ell})d(c_1,c_2)^2$.
\end{lemma}

\begin{proof}
From part 2 of Lemma \ref{lem:mass}, we have $\Delta_1^2(X)=\Delta_2^2(X)+\frac{n_1 n_2}{n}\cdot d(c_1,c_2)^2$, which implies that
\begin{equation*}
\frac{n}{n_1 n_2}\cdot\Delta_2^2(X)=d(c_1,c_2)^2\frac{\Delta_2^2(X)}{\Delta_1^2(X)-\Delta_2^2(X)}.
\end{equation*}
Let $c$ denote the center of mass of $X$.
Then 
\begin{align*}
\Delta_1^2(X)&=\sum_{x\in X} d(c,x)^2=n_1 d(c,c_1)^2+n_2 d(c,c_2)^2+\Delta_2^2(X)\\
&>\text{min}(n_1,n_2)(d(c,c_1)^2+d(c,c_2)^2)+\Delta_2^2(X)\geq\text{min}(n_1,n_2)d(c_1,c_2)^2+\Delta_2^2(X).
\end{align*}
Therefore, 
\begin{equation*}
\frac{n}{n_1}r_1^2+\frac{n}{n_2}r_2^2\leq\frac{\Delta_2^2(X)}{\text{min}_i n_i\cdot d(c_1,c_2)}\leq \frac{n}{\text{min}_i n_i}\cdot\frac{w_{avg}^2}{d(c_1,c_2)^2}=\frac{n}{\text{min}_i n_i},
\end{equation*} 
and it follows that $\text{max}(r_1^2,r_2^2)\leq O(\frac{\epsilon}{\alpha}\cdot\frac{L}{\ell})d(c_1,c_2)^2$.
\end{proof}

Let $\rho=\frac{100\epsilon}{\alpha}<1$. Now we define the core of a cluster $C_i$ as $X_i=\{x\in C_i\mid d(x,c_i)^2\leq\frac{r_i^2}{\rho}\}$.
Then by a Markov inequality, $|X_i|\geq (1-\rho)n_i$ for all $i$.
This concentration inequality, along with Lemma \ref{lem:r_i} are the key structures needed to show $k$-means++ produces a good clustering.
Recall in $k$-means++, we pick seeds $\{\hat c_1,\dots,\hat c_k\}$ so that we pick point $x$ for $\hat c_{i+1}$ with probability 
proportional to $d(x,\{\hat c_1,\dots,\hat c_i\})^2$.
We defer the proof of the following lemma to the appendix.

\begin{lemma} \label{lem:k=2}
Assume $k=2$.
For sufficiently small $\epsilon$, 
$Pr[(\hat c_1\in X_1~\wedge~ \hat c_2\in X_2)~ \vee~ (\hat c_2\in X_1~ \wedge~ \hat c_1\in X_2)]=
1-O(\rho)$.
\end{lemma}

This lemma, combined with Lemma 3.3 from \citep{ostrovsky2006effectiveness}, immediately gives us the following theorem.

\begin{theorem}
$k$-means++ seeding with a Lloyd step outputs a solution that is $\frac{1}{1-\rho}$ close to the optimal solution
with probability $>1-O(\rho)$, for clustering instances satisfying $(1+\alpha,\epsilon)$-approximation stability
for the balanced $2$-means objective, with $\frac{L}{\ell}\in O(1)$.
\end{theorem}

Now we have a proof for $k=2$, however,
if we induct on the number of clusters, the probability of success
becomes exponentially small in $k$. We circumvent this issue in a manner similar
to \citet{ostrovsky2006effectiveness}, by sampling $k\log k$ centers, and
carefully deleting centers greedily, until we are left with one center per cluster
with high probability.
The rest of the details in proving Theorem \ref{thm:k-means++} are presented in Appendix \ref{app:stability}.

\paragraph{$k$-median}
We show that the greedy thresholding algorithm of \citet{as} is sufficient to give a good
clustering even for the balanced $k$-median or $k$-means objective, under approximation stability.
At a high level, their algorithm works by first creating a threshold graph for a specific distance, and then
iteratively picking the
node with the highest degree in the threshold graph and removing its neighborhood.
 In Appendix \ref{app:stability}, we show balanced clustering instances where the analysis in \citet{as}
is not sufficient to guarantee good
clusterings are outputted. 
We provide a new technique which overcomes the difficulties in adding
upper and lower balance constraints,
and we obtain the following theorem.

\begin{theorem} \label{thm:as}
There is an efficient algorithm which returns a valid clustering that is
$O(\frac{\epsilon}{\alpha})$-close to the optimal, for balanced $k$-median or $k$-means clustering
under $(1+\alpha,\epsilon)$-approximation stability, provided all clusters are size
$\geq 3\epsilon n(1+\frac{3}{\alpha})$.
\end{theorem}

A key point in the argument of Balcan, Blum, and Gupta is that for 
$(1+\alpha,\epsilon)$-approximation stable
instances, there are less than $\epsilon n$ points $v$ for which the distance between
their closest and second-closest center is small. Otherwise, these points could switch
clusters and cause a contradiction.
With balance constraints, this logic breaks down, since a point may not be able to
switch clusters if the new cluster is at capacity.
We provide a new structural result to handle this problem:
there are only $\frac{\epsilon n}{2}$ distinct
\emph{pairs} of points from different clusters that are close to each other.
This allows us to show the algorithm returns a clustering that is near-optimal.
We provide the details of this argument in Appendix \ref{app:stability}.

\paragraph{$k$-center}
For $k$-center, we obtain a tight result with respect to the level of approximation
stability needed. In particular, we show the exact level of stability for which the
problem of finding the optimal balanced $k$-center solution switches from NP-hard
to efficiently computable. This is shown in the following theorems.
Theorem \ref{thm:kcenter-lb} is a generalization of the lower bound in \citep{kcenter-pr}.

\begin{theorem} \label{thm:kcenter-lb}
There is an efficient algorithm such that if a clustering instance satisfies
$(2,0)$-approximation stability for the balanced $k$-center objective, then the algorithm
will return the optimal clustering. Furthermore, for $\epsilon>0$,
there does not exist an efficient algorithm that
returns optimal clusterings for balanced $k$-center instances under
$(2-\epsilon,0)$-approximation stability, unless $NP=RP$.
\end{theorem}

\section{Structure of Balanced Clustering} \label{sec:structure}
In this section, 
we show that adding lower bounds to clustering makes the problem highly nontrivial. Specifically,
our main result is that the $k$-means, $k$-median, and $k$-center objective values may oscillate arbitrarily
with respect to $k$ (Theorem \ref{thm:localmax}).
In light of this structure,
our results from Sections \ref{sec:bicriteria} and \ref{sec:stability} are more surprising, since
it is not obvious that algorithms with constant-factor guarantees exist.

We give a variety of clustering instances which do not have monotone cost functions with respect to $k$.
For readability and intuition, these examples start out simple, and grow in complexity until we eventually
prove Theorem \ref{thm:localmax}.

First, consider a star graph with $n$ points and lower bound $\ell$, such that $n\ell\geq 3$
(see Figure \ref{fig:stargraph}).
	\begin{figure}
    \centering
		\includegraphics{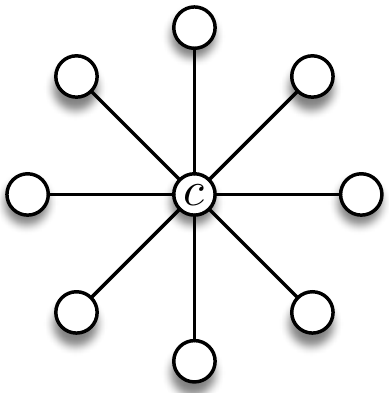}
		\caption{A graph in which the objective function
                  strictly increases with $k$. Each edge signifies
                  distance 1, and all other distances are 2.}
		\label{fig:stargraph}
	\end{figure}
The center $c$ is at distance 1 to the $10n\ell$ leaves,
and the leaves are at distance 2 from each other.
When $k=1$, each point is distance 1 to the center $c$.
However as we increase $k$, the new centers must be leaves, distance 2
from all the other points, so $n\ell-1$ points must pay 2 instead of 1
for each extra center.
It is also easy to achieve an objective that strictly decreases
up to a local minimum $k'$, and then strictly increases onward,
by adding $k'$ copies of the center of the star.

	\begin{lemma} \label{lem:stargraph}
	Given a star graph with parameters $n$ and $\ell$ such that 
	$n\ell\geq 3$, then the cost of the $k$-means and $k$-median objectives
	strictly increase in $k$.
\end{lemma}

	\begin{proof}
Let the size of the star graph be $n$.
Clearly, the optimal center for $k=1$ is $c$. Then $\mathcal{OPT}_1=n-1$.
Then for $k=2$, we must choose another center $p$ that is not $c$. $p$
is distance 2 to all points other than $c$, so the optimal clustering
is for $p$'s cluster to have the minimum of $n\ell$ points, and $c$'s cluster
has the rest.
Therefore, $\mathcal{OPT}_2=n+n\ell-2$.

This process continues; every time a new center is added, the new center
pays 0 instead of 1, but $n\ell-1$ new points must pay 2 instead of 1.
This increases the objective by $n\ell-2$. As long as $n\ell\geq 3$,
this ensures the objective function is strictly increasing in $k$.
\end{proof}

Note for this example, the problem goes away if we are allowed to
place multiple centers on a single point (in the literature,
this is called ``soft capacities'', as opposed to enforcing one
center per point, called ``hard capacities'').
The next lemma shows there can be a local minimum for hard capacities.

\begin{lemma} \label{lem:groups}
For all $k'$, there exists a balanced clustering instance in which the $k$-means or $k$-median objective
as a function of $k$ has a local minimum at $k'$.
\end{lemma}

\begin{proof}
Given $l\geq 3$,
we create a clustering instance as follows. Define $k'$ sets of points
$G_1,\dots,G_{k'}$, each of size $2n\ell-1$.
For any two points in some $G_i$, set their distance to 0. For any two
points
in different sets, set their distance to 1.
Then for $1\leq k\leq k'$, the objective value is equal to
$(k'-k)(2n\ell-1)$, since we can put $k$ centers into $k$ distinct groups,
but $(k'-k)$ groups will not have centers, incurring cost $2n\ell-1$.
When $k>k'$, we cannot put each center in a distinct group, so there is
some group $G_i$ with two centers. Since $|G_i|=2n\ell-1$, the two centers cannot
satisfy the capacity constraint with points only from $G_i$, so the
objective value increases.
\end{proof}

\paragraph{Local maxima}
So far, we have seen examples in which the objective decreases with
$k$, until it hits a minimum (where capacities start to become
violated), and then the objective strictly increases.  The next
natural question to ask, is whether the objective can also have a
local maximum. We show the answer is yes in the following lemma.

\begin{lemma} \label{lem:localmax234}
There exists a balanced clustering instance in which the 
$k$-center, $k$-median, and $k$-means objectives
contain a local maximum with respect to $k$.
\end{lemma}

\begin{figure}
      \centering
        \includegraphics[width=.4\textwidth]{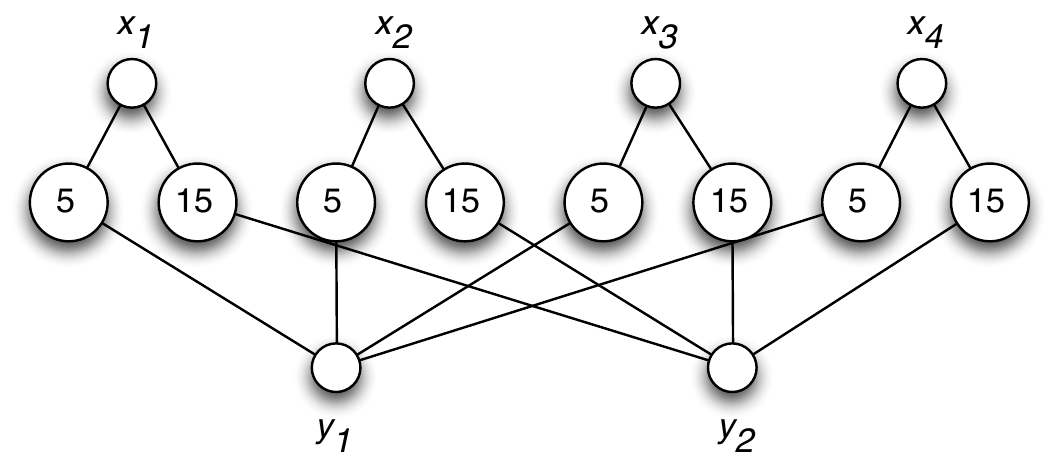}
        \caption{Each edge signifies distance 1, and all other
          distances are 2. The middle points are replicated as many
          times as their label suggests (but each pair of replicated
          points are still distance 2 away).  Finally, add length 1
          edges between all pairs in
          $\{x_1,x_2,x_3,x_4\},\{y_1,y_2\}$.}
        \label{fig:localmax}
        \vspace{-0.5cm}
\end{figure}

\begin{proof}[Proof sketch.]
Consider Figure \ref{fig:localmax}, where $n=86$, and set
$n \ell = 21$.
Since the distances are all 1 or 2,
this construction is trivially a valid
  distance metric. 
	From Figure \ref{fig:localmax}, we see that $k=2$ and $k=4$ have
valid clusterings using only length 1 edges, using centers
$\{y_1,y_2\}$ and $\{x_1,x_2,x_3,x_4\}$, respectively.  But now
consider $k=3$. The crucial property is that by construction, $y_1$
and any $x_i$ cannot simultaneously be centers and each satisfy the
capacity to distance 1 points, because the union of their distance 1
neighborhoods is less than $2n\ell$. 
In the appendix, we carefully check all other sets of 3 centers
do not achieve a clustering with distance 1 edges, which completes the proof.
\end{proof}

The previous example does not work for the case of soft capacities,
since the set of centers $\{x_1,y_2,y_2\}$ allows every point to
have an edge to its center.

Now we prove our main theorem. Note, this theorem holds even for soft capacities.

\begin{theorem} \label{thm:localmax}
For all $m\in \mathbb{N}$, there exists a balanced clustering instance 
in which the $k$-center, $k$-median, and $k$-means objectives contain
$m$ local maxima, even for soft capacities.
\end{theorem}

\begin{proof}

As in the previous lemma, we will construct a set of points in which each pair of points are
either distance 1 or 2.
It is convenient to define a graph on the set of points,
in which an edge signifies a distance of 1, and all
non-edges denote distance 2.
We will construct a clustering instance where the objective value for all even values of $k$ between $10m$ and $12m$
 is low and the objective value for all odd values of $k$ between $10m$ and $12m$ is high.
The $m$ odd values will be the local maxima. We will set the lower bound $n\ell$ to be the product of all
the even integers between $10m$ and $12m$.

We start by creating a distinct set of ``good'' centers, $X_k$, for each even value of $k$ between $10m$ and $12m$.
Let $X$ be the union of these sets. The set $X_k$ contains $k$ points which will be the optimal centers for a $k$-clustering
in our instance.
Then we will add an additional set of points, $Y$, and add edges from $Y$ to the centers in $X$ with the following properties.

\begin{enumerate}

\item For each even value of $k$ between $10m$ and $12m$, there is an assignment of the points in $Y$ to the centers in $X_k$ so that points in $Y$ are only assigned to adjacent centers and the capacity constraints are satisfied.

\item Each of the good centers in $X$ is adjacent to no more than
$\frac{6}{5}\cdot n\ell$ points in $Y$.

\item For each good center $x$ in $X_k$, there is at least one point $x'$ in every other set $X_k'$ (for $k' \neq k$) so that the number of points in $Y$ adjacent to both $x$ and $x'$ is at least $\frac{2}{5} \cdot n\ell$.

\item Any subset of the centers in $X$ that does not contain any complete set of good centers $X_k$ for some even $k$ is non-adjacent to at least one point in $Y$.
\end{enumerate}

Whenever we add a point to $Y$, we give it an edge to exactly one point from each $X_k$. This ensures that each $X_k$ partitions $Y$.
We first create connected components as in Figure \ref{fig:general_backbone} that each share
$\frac{2}{5} \cdot n\ell$ points from $Y$, to satisfy Property 3.

	\begin{figure*}
    \begin{center}
		\includegraphics[width=\textwidth]{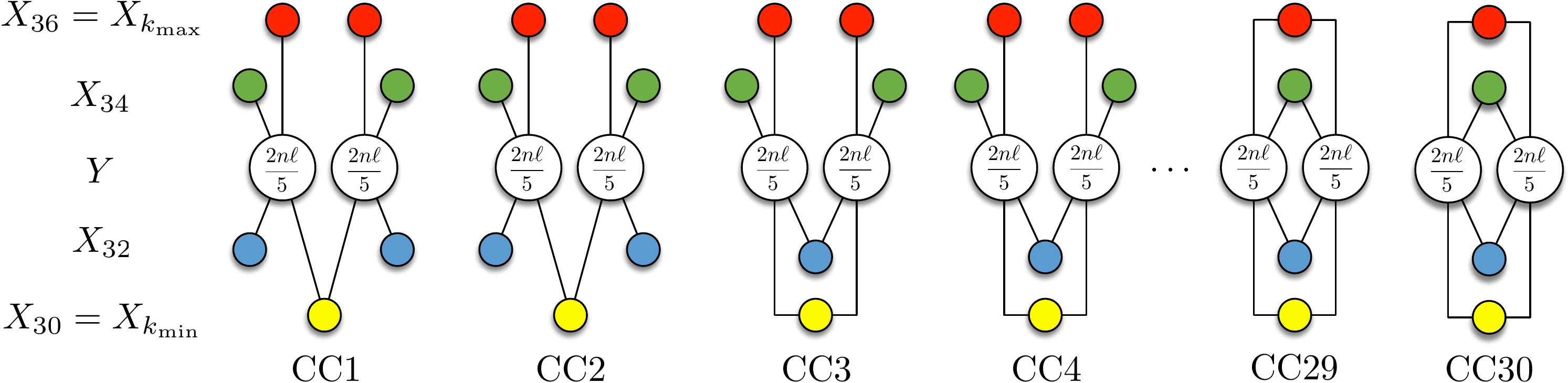}
    \end{center}
		\caption{An example when $m=3$. Each $X_k$ is a different color. 
		Each edge signifies distance 1, and all other
          distances are 2. The middle points are replicated as many
          times as their label suggests (but each pair of replicated
          points are still distance 2 away). }
		\label{fig:general_backbone}
	\end{figure*}

For property 4,  we add one additional point to $Y$ for every combination of picking one point from each $X_k$. This ensures that any set which does not contain at least
one point from each $X_k$ will not be a valid partition for $Y$.
Note that in the previous two steps, we did not give a single center more than
$\frac{6}{5}\cdot n\ell$ edges, satisfying property 2.
Then we add ``filler'' points to bring every center's capacity up to at least $n\ell$, which satisfies property 1.

Now we explain why properties 1-4 are sufficient to finish off the proof.
Property 1 guarantees that the for each even value of $k$ there is a clustering where the cost of each point in
$Y$ is one, which results in a good clustering objective.

Properties 2 and 3 guarantee that any set including a full $X_k$ and a point from a different $X_{k'}$ cannot achieve cost 1 for each point
without violating the capacities.
Property 4 guarantees that any set without a full $X_k$ cannot achieve cost 1 for each point.
This completes the proof sketch.

\end{proof}

\section{Efficient Clustering by Subsampling}
\label{sec:sampleComplexity}
\newcommand \kchp {{k \choose p}}
\newcommand \inv {^{-1}}
\newcommand \setcbigg [2]{\biggl\{#1 \,:\, #2\biggr\}}
\newcommand \dist [1]{d(#1)}

For datasets large enough to require a distributed learning system, it
is expensive to apply a clustering algorithm to the entire dataset. In
this section, we show that we can first cluster a small subsample of
data and then efficiently extend this clustering to the remaining
data. In our technique, each point in the dataset is assigned to the
same $p$ clusters as its nearest neighbor in the clustered
subsample. This dispatch rule extends the clustering to the entire
space $\X$ containing the data (not just to the unused portion of the
training set), so at prediction time it can be used to send query
points to the appropriate machines. We show that the clustering
induced over $\X$ approximately inherits all of the desirable
properties of the clustered subsample: good objective value, balanced
clusters, and replication. Note that the simpler approach of assigning
a new example $x$ to the $p$ clusters with the nearest centers
achieves the lowest cost for new examples, but it may severely violate
the capacity constraints if the data distribution is not evenly
distributed among the centers.

\newcommand \Qmed {Q^{(1)}}
\newcommand \hQmed {\hat{Q}^{(1)}}
\newcommand \Qmea {Q^{(2)}}
\newcommand \hQmea {\hat{Q}^{(2)}}
Each clustering of $\X$ can be represented as a pair $(f, C)$ where
$C = \{c_1, \dots, c_k\}$ is a set of centers and
$f : \X \to {C \choose p}$ is a function that assigns each point in
$\X$ to $p$ of the $k$ centers. We measure the quality of a clustering
of $\X$ as follows: given a data distribution $\pdist$ over $\X$, our
goal is to find a clustering with centers $C$ and an assignment
function $f : \X \to {C \choose p}$ that minimizes either the
$k$-median objective $\Qmed$ or the $k$-means objective $\Qmea$ given
by
\[
  \Qmed(f,C)
  = \expect_{x \sim \pdist} \biggl[\, \sum_{i \in f(x)} \dist{x, c_i} \biggr]
  \qquad \hbox{and} \qquad
  \Qmea(f,C)
  = \expect_{x \sim \pdist} \biggl[\, \sum_{i \in f(x)} \dist{x, c_i}^2 \biggr],
\]
subject to the constraint that each cluster has probability mass
between $\ell$ and $L$. Specifically, we require for each cluster
index $i$ that $\prob_{x \sim \pdist}(i \in f(x)) \in [\ell,
L]$. Throughout this section, we use the notation $Q$, $Q_n$, and
$\hat Q_n$ as a placeholder for either the $k$-median or $k$-means
objective.

In our algorithm, each point $x$ in the subsample $S$ acts as a
representative for those points in $\X$ that are closer to it than any
other sample point (i.e., its cell in the Voronoi partition induced by
$S$). Since each sample point might represent more or less of the
probability mass of $\pdist$, we consider the following weighted
clustering problem of a dataset $S$. A clustering of the data set $S$
is a pair $(g,C)$ for some set of centers $C = \{c_1, \dots, c_k\}$
and an assignment function $g : S \to {C \choose p}$ that assigns each
point of $S$ to $p$ of the centers. The weight for point $x_i$ is
$w_i = \prob_{x \sim \pdist}(\NN_S(x) = x_i)$, where $\NN_S(x)$
denotes the nearest neighbor in $S$ to the point $x$. The weighted
$k$-median and $k$-means objectives on $S$ are given by
\[
\Qmed_n(g,c)
= \sum_{j = 1}^n w_j \sum_{i \in g(x_j)} \dist{x_j, c(i)}
\qquad\hbox{and}\qquad
\Qmea_n(g,c)
= \sum_{j = 1}^n w_j \sum_{i \in g(x_j)} \dist{x_j, c(i)}^2,
\]
where the subscript $n$ refers to the size of the subsample $S$. The
weighted capacity constraints require that the total weight of each
cluster $i$, given by $\sum_{j : i \in g(x_j)} w_j$, is between $\ell$
and $L$. Since the distribution $\pdist$ is unknown, our algorithm
uses a second sample drawn from $\pdist$ to estimate the
weights. Given estimates $\hat w_1$, \dots, $\hat w_n$ of the true
weights, the estimated $k$-median and $k$-means objective functions
are
\[
\hQmed_n(g,c)
=
\sum_{j = 1}^n \hat w_j \sum_{i \in g(x_j)} \dist{x_j, c(i)}
\qquad\hbox{and}\qquad
\hQmea_n(g,c)
=
\sum_{j = 1}^n \hat w_j \sum_{i \in g(x_j)} \dist{x_j, c(i)}^2,
\]
and the estimated weight of a cluster is
$\sum_{j:i \in g(x_j)} \hat w_j$.  Finally, for any clustering $(g,c)$
of $S$, define the nearest neighbor extension to be $(\bar g, c)$
where $\bar g(x) = g(\NN_S(x))$. The assignment function $\bar g$
assigns each point in $\X$ to the same $p$ clusters as its nearest
neighbor in $S$.

Our algorithm first runs a clustering algorithm to approximately
minimize the weighted clustering objective on the sample $S$ (where
the weights are estimated using a second sample drawn from the same
distribution). Then, we extend this clustering to the entire space
$\X$ by assigning a new example to the same $p$ clusters as its
nearest neighbor in the subsample $S$. Psuedocode is given in
Algorithm~\ref{alg:nnextension}. Our main result in this section shows
that the resulting clustering of $\X$ approximately satisfies the
capacity constraints and has a nearly optimal objective value.

\begin{figure}
\begin{framed}
{\bf Input:} Dataset $S = \{x_1, \dots, x_n\}$, cluster parameters $(k,p,\ell, L)$, second sample size $n'$.
\begin{enumerate}[noitemsep,nolistsep,leftmargin=*]
\item Draw second sample $S'$ of size $n'$ iid from $\mu$.
\item For each point $x_i$, set $\hat w_i = |S'_i|/n'$, where
  $S'_i = \setc{x' \in S'}{NN_S(x') = x_i}$
\item Let $C_S = (c_1, \dots, c_k)$ and $g_S : S \to {C \choose p}$ be
  a clustering of $S$ obtained by minimizing $\hat Q_n(g,C)$ subject
  to the approximate weighted capacity constraints.
\item Return $\bar g_S(x) = g_S(\NN_S(x))$ and centers $C_S$.
\end{enumerate}
\end{framed}
\vspace{-1em}
\renewcommand {\figurename} {Algorithm}
\caption{Nearest neighbor clustering extension.}
\label{alg:nnextension}
\vspace{-0.57cm}
\end{figure}
\setcounter{figure}{2}
\setcounter{algocf}{2}

Before stating our main result, we first show that if we take the
second sample size $n'$ to be $\tilde O(n / \epsilon^2)$, then with
high probability the error in any sum of the estimated weights
$\hat w_j$ is at most $\epsilon$.

\begin{restatable}{lemma}{weightApproxLem}
  \label{lem:weightapprox}
  For any $\epsilon > 0$ and $\delta > 0$, if we set
  $n' = O\bigl(\frac{1}{\epsilon^2}(n + \log \frac{1}{\delta})\bigr)$
  in Algorithm~\ref{alg:nnextension}, then with probability at least
  $1-\delta$ we have
  $\bigl| \sum_{i \in I} (w_i - \hat w_i) \bigr| \leq \epsilon$
  uniformly for all index sets $I \subset [n]$.
\end{restatable}
\begin{proof}
  Let $V_i$ be the cell of point $x_i$ in the Voronoi partition
  induced by $S$. For any index set $I \subset [n]$, let $V_I$ denote
  the union $\bigcup_{i \in I} V_i$. Since the sets $V_1$, \dots,
  $V_n$ are disjoint, for any index set $I$ we have that
  $\pdist(V_I) = \sum_{i \in I} w_i$ and
  $\hat \pdist(V_I) = \sum_{i \in I} \hat w_i$, where $\hat \pdist$ is
  the empirical measure induced by the second sample $S'$. Therefore
  it suffices to show uniform convergence of $\hat \mu(V_I)$ to
  $\mu(V_I)$ for the $2^n$ index sets $I$. Applying Hoeffding's
  inequality to each index set and the union bound over all $2^n$
  index sets, we have that
  \[
    \prob \biggl(
    \sup_{I \subset [n]} \biggl| \sum_{i \in I} w_i - \hat w_i \biggr| > \epsilon
    \biggr)
    \leq 2^n e^{-2n' \epsilon^2}.
  \]
  Setting
  $n' = O\bigl(\frac{1}{\epsilon^2}(n + \log \frac{1}{\delta}) \bigr)$
  results in the right hand side being equal to $\delta$.
\end{proof}

Next we relate the weighted capacity constraints and objective over
the set $S$ to the constraints and objective over the entire space
$\X$.

\begin{lemma}
  \label{lem:extensionprops}
  Let $(g,c)$ be any clustering of $S$ that satisfies the weighted
  capacity constraints with parameters $\ell$ and $L$. Then the
  nearest neighbor extension $(\bar g, c)$ satisfies the capacity
  constraints with respect to $\mu$ with the same parameters. For the
  $k$-median objective we have
  \[
    |\Qmed_n(g,c) - \Qmed(\bar g, c)| \leq p \expect_{x \sim \pdist}\bigl[\dist{x ,\NN_S(x)}\bigr],
  \]
  and for the $k$-means objective we have
  \[
    \Qmea_n(g,c) \leq 2\Qmea(\bar g, c) + 2p \expect_{x \sim \pdist}\bigl[\dist{x ,\NN_S(x)}^2\bigr]
    \quad\hbox{and}\quad
    \Qmea(\bar g,c) \leq 2\Qmea_n(g, c) + 2p \expect_{x \sim \pdist}\bigl[\dist{x ,\NN_S(x)}^2\bigr].
  \]
\end{lemma}
\begin{proof}
  The fact that $\bar g$ satisfies the population-level capacity constraints
  follows immediately from the definition of the weights $w_1$, \dots, $w_n$.

  By the triangle inequality, $k$-median objective over $\X$ with
  respect to $\mu$ can be bounded as follows
  \[
    \Qmed(\bar g, c)
    \leq \expect_{x \sim \pdist} \biggl[\, \sum_{i \in \bar g(x)} \dist{x , \NN_S(x)}\biggr]
    + \expect_{x \sim \pdist} \biggl[\, \sum_{i \in \bar g(x)} \dist{\NN_S(x) , c(i)} \biggr]
    = p \expect_{x \sim \pdist} [\dist{x , \NN_S(x)}] + \Qmed_n(g, c).
  \]
  The reverse inequality follows from an identical argument applying
  the triangle inequality to $\Qmed_n$.

  For the $k$-means objective, the result follows similarly by using
  the following approximate triangle inequality for squared distances:
  $\dist{x,z}^2 \leq 2(\dist{x,y}^2 + \dist{y,z}^2)$,
\end{proof}

\paragraph{Main Result} We bound the sub-optimality of the clustering
$(\bar g_n, c_n)$ returned by Algorithm~\ref{alg:nnextension} with
respect to any clustering $(f^*, c^*)$ of the entire space $\X$. The
bound will depend on
\begin{enumerate} \item the quality of the finite-data
  algorithm, \item the ``average radius'' of the Voronoi cells
  $\alpha_1(S) = \expect_{x \sim \pdist} [\dist{x , \NN_S(x)}]$ for
  $k$-median and the ``average squared radius'' for $k$-means
  $\alpha_2(S) = \expect_{x \sim \pdist}[\dist{x,\NN_S(x)}^2]$ ,
  and \item the bias from returning clusterings that are constant over
  the cells in the Voronoi partition induced by $S$. The following
  definition measures the bias for $k$-median
  \[
    \beta_1(S, \ell,L)
    = \min_{h,c}\big\{
    \Qmed(\bar h, c) - \Qmed(f^*, c^*)
    \,\big|\, h\text{ satisfies balance constraints } (\ell ,L)\big\},
  \]
  where the minimum is taken over all clusterings $(h,c)$ of the
  sample $S$ and $(\bar h, c)$ denotes the nearest neighbor
  extension. The bias $\beta_2(S, \ell, L)$ for $k$-means is defined
  analogously.
\end{enumerate}

\begin{theorem}
  \label{th:sampleComplexity}
  For any $\epsilon > 0,\delta > 0$, let $(\bar g_n, c_n)$ be the
  output of Algorithm~\ref{alg:nnextension} with parameters
  $k, p, \ell, L$ and second sample size
  $n' = O\bigl((n + \log 1/\delta)/\epsilon^2\bigr)$. Let $(f^*, c^*)$
  be any clustering of $\X$ and $(g_n^*, c_n^*)$ be an optimal
  clustering of $S$ under $\hat Q_n$ satisfying the estimated weighted
  balance constraints $(\ell, L)$. Suppose the algorithm used to
  cluster $S$ satisfies
  $\hat Q(g_n, c_n) \leq r\cdot \hat Q(g_n^*, c_n^*) + s$. Then
  w.p. $\geq 1-\delta$ over the second sample the output
  $(\bar g_n, c_n)$ will satisfy the balance constraints with
  $\ell' = \ell - \epsilon$ and $L'=L + \epsilon$. For $k$-median we
  have
  \begin{align*}
  \Qmed(\bar g_n, c_n) \leq r\cdot \Qmed(f^*, c^*) + s
                                  + 2(r+1)pD\epsilon
                                  + p(r+1) \alpha_1(S)
                                  + r\beta_1(S, \ell + \epsilon, L - \epsilon),
  \end{align*}
  and for $k$-means we have
  \begin{align*}
      \Qmea(\bar g_n, c_n) \leq 4r\cdot \Qmea(f^*, c^*) + 2s
                                  + 4(r+1)pD^2\epsilon
                                  + 2p(2r+1) \alpha_2(S)
                                  + 4r\beta_2(S, \ell + \epsilon, L - \epsilon).
  \end{align*}
\end{theorem}

\begin{proof}
  Lemma~\ref{lem:weightapprox} guarantees that when the second sample
  is of size $O(\frac{1}{\epsilon^2}(n + \log \frac{1}{\delta}))$ then
  with probability at least $1 - \delta$, for any index set
  $I \subset [n]$, we have
  $\bigl| \sum_{i \in I} w_i - \hat w_i \bigr| \leq \epsilon$. For the
  remainder of the proof, assume that this high probability event
  holds.

  First we argue that the clustering $(g_n, c_n)$ satisfies the true
  weighted capacity constraints with the slightly looser parameters
  $\ell' = \ell - \epsilon$ and $L' = L + \epsilon$. Since the
  clustering $(g_n, c_n)$ satisfies the estimated weighted capacity
  constraints, the high probability event guarantees that it will also
  satisfy the true weighted capacity constraints with the looser
  parameters $\ell' = \ell - \epsilon$ and $L' = L + \epsilon$.
  Lemma~\ref{lem:extensionprops} then guarantees that the extension
  $(\bar g_n, c_n)$ satisfies the population-level capacity
  constraints with parameters $\ell'$ and $L'$.

  Next we bound the difference between the estimated and true weighted
  objectives for any clustering $(g,c)$ of $S$. For each point $x_j$
  in the set $S$, let $C_j = \sum_{i \in g(x_j)} d(x_j, c(i))$ be the
  total distance from point $x_j$ to its $p$ assigned centers under
  clustering $(g,c)$, and let $J$ be the set of indices $j$ for which
  $\hat w_j > w_j$. Then by the triangle inequality, we have the
  following bound for the $k$-median objective:
  \begin{align}
  |\hQmed_n(g, c) - \Qmed_n(g, c)|
  &\leq \biggl|\sum_{j \in J} (\hat w_j - w_j) C_j \biggr|
      + \biggl|\sum_{j \not\in J} (w_j - \hat w_j) C_j \biggr| \nonumber\\
  &\leq \biggl(\biggl|\sum_{j \in J} (\hat w_j - w_j) \biggr|
      + \biggl|\sum_{j \not\in J} (w_j - \hat w_j) \biggr|\biggr) pD \nonumber\\
  &\leq 2pD\epsilon, \label{eq:Qestimate}
  \end{align}
  where the second inequality follows from the fact that $C_j \leq pD$
  and the sum has been split so that $(\hat w_j - w_j)$ is always
  positive in the first sum and negative in the second. For the
  $k$-means objective, the only difference is that our upper bound on
  $C_j$ is $D^2$ instead of $D$, which gives
  $|\hQmea_n(g, c) - \Qmea_n(g, c)| \leq 2pD^2\epsilon$.

  Finally, let $(h_n, c_n')$ be the clustering of $S$ that attains the
  minimum in the definition of
  $\beta(S, \ell + \epsilon, L - \epsilon)$. That is, the clustering
  of $S$ satisfying the capacity constraints with parameters
  $\ell + \epsilon$ and $L - \epsilon$ whose nearest neighbor
  extension has the best objective over $\X$ with respect to $\mu$
  (note that this might not be the extension of $(g_n^*, c_n^*)$).

  Now we turn to bounding the $k$-median objective value of
  $(\bar g_n, c_n)$ over the entire space $\X$. Combining
  Lemma~\ref{lem:extensionprops}, equation \eqref{eq:Qestimate}, the
  approximation guarantees for $(g_n, c_n)$ with respect to
  $\hat Q_n$, and the optimality of $(g_n^*, c_n^*)$, we have the
  following:
  \begin{align*}
    \Qmed(\bar g_n, c_n)
    &\leq \Qmed_n(g_n,c_n) + p \alpha_1(S) \\
    &\leq \hQmed_n(g_n, c_n) + 2pD\epsilon + p\alpha_1(S) \\
    &= \Qmed_n(g_n, c_n) - r\cdot \hQmed_n(g_n^*, c_n^*) + r\cdot \hQmed_n(g_n^*, c_n^*)  + 2pD\epsilon + p\alpha_1(S) \\
    &\leq s + 2pD\epsilon + p\alpha_1(S) + r\cdot \hQmed_n(h_n, c_n') \\
    &\leq s + 2(r+1)pD\epsilon + p\alpha_1(S) + r\cdot \Qmed_n(h_n, c_n') \\
    &\leq s + 2(r+1)pD\epsilon + p(r+1)\alpha_1(S) + r\cdot \Qmed(\bar h_n, c_n') \\
    &\leq s + 2(r+1)pD\epsilon + p(r+1)\alpha_1(S) + r\cdot\beta_1(S, \ell + \epsilon, L - \epsilon) + r \cdot \Qmed(f^*, c^*).
  \end{align*}
  The proof for the case of $k$-means is identical, except we use the
  corresponding result from Lemma~\ref{lem:extensionprops} and the
  alternative version of equation \eqref{eq:Qestimate}.
\end{proof}

The above theorem applies for any set $S$, but the quality of the
bound depends on $\alpha(S)$ and $\beta(S)$, which measure how well
the set $S$ represents the distribution $\pdist$. We now bound
$\alpha(S)$ and $\beta(S)$ when $S$ is a large enough iid sample drawn
from $\mu$ under various conditions on $\mu$ and the optimal
clustering.

\paragraph{Bounding $\alpha(S)$} We bound the sample size required to
make $\alpha(S)$ small when $\X\subseteq\reals^q$ and $S$ is drawn
randomly from an arbitrary $\pdist$. Additionally, when the
distribution has a lower intrinsic dimension, we can do better. The
doubling condition is one such a condition. Let $B(x, r)$ be a ball of
radius $r$ around $x$ with respect to the metric $d$. A measure
$\pdist$ with support $\X$ is said to be a doubling measure of
dimension $d_0$ if for all points $x \in \X$ and all radii $r > 0$ we
have $\pdist(B(x,2r)) \leq 2^{d_0} \pdist(B(x,r))$.

\begin{restatable}{lemma}{lemmaAlpha}
  \label{lemma:alpha}
  For any $\epsilon, \delta > 0$, and $\X \subseteq \reals^q$, if a
  randomly drawn $S$ from $\pdist$ is of size
  $O( {q}^{q/2}{\epsilon}^{-(q+1)}(q \log \frac{\sqrt{q}}{\epsilon} +
  \log \frac{1}{\delta}))$ in the general case, or
  $O(\epsilon^{-d_0}(d_0 \log\frac {1} \epsilon + \log\frac 1
  \delta))$ if $\pdist$ is doubling with dimension $d_0$, then w.p
  $\geq 1-\delta$, $\alpha_1(S) \leq \epsilon D$ and
  $\alpha_2(S) \leq (\epsilon D)^2$.
\end{restatable}

\paragraph{Bounding $\beta(S)$} Bubeck et al.\
\cite{bubeck2009nearest} provide a worst case lower bound when $f^*$
is continuous almost everywhere. Again, one can do better for
well-behaved input. The Probabilistic Lipschitzness (PL) condition
\cite{urner2011access,urner2013plal} says that $f$ is $\phi$-PL if the
probability mass of points that have non-zero mass of differently
labeled points in a $\lambda D$-ball around them is at most
$\phi(\lambda)$. If a clustering function $f$ is PL, it means the
clusters are, in some sense, ``round''- that the probability mass
``close to'' the boundaries of the clusters is small. Under this
condition, we have the following sample complexity result for
$\beta$. We can compare to a clustering with slightly tighter size
constraints:

\begin{restatable}{lemma}{lemmaBeta}
  \label{lem:beta}
  Let $\pdist$ be a measure on $\reals^q$ with support $\X$ of
  diameter $D$.  Let $f^*$ be the optimal clustering of $\pdist$ that
  satisfies capacities $(\ell+\epsilon, L-\epsilon)$ and suppose $f^*$
  is $\phi$-PL. If we see a sample S drawn iid from $\pdist$ of size
  $O\Big(\frac{1}{\epsilon} \big( \frac 1
  {\phi\inv(\epsilon/2)}\big)^q \big(q \log
  \frac{\sqrt{q}}{\phi\inv(\epsilon/2)} + \log \frac{1}{\delta}\big)
  \Big)$ in the general case or
  $O\Big( \big(\frac{1}{\phi\inv(\epsilon)}\big)^{d_0} \big(d_0
  \log\frac {4} {\phi\inv(\epsilon)} + \log \frac{1}{\delta}\big)
  \Big)$ when $\pdist$ is a doubling measure of dimension $d_0$ then,
  w.p. at least $1-\delta$ over the draw of $S$, we have that
  $\beta_1(S, \ell, L) \leq pD\epsilon$ and
  $\beta_2(S, \ell, L) \leq pD^2\epsilon$.
\end{restatable}

\section{Experiments}
\label{sec:expt}
In this section, we present an empirical study of the accuracy and
scalability of our technique using both the LP rounding algorithms and
$k$-means++ together with the nearest neighbor extension. We compare
against three baselines: random partitioning, balanced partition
trees, and locality sensitive hashing (LSH) on both synthetic and real
world image and advertising datasets. Our findings are summarized
below:

\begin{itemize}
\item Using small samples of the given datasets, we compare the
  clusterings produced by our LP rounding algorithms\footnote{ We can
    run the LP rounding algorithm for small $n$, even though there are
    $O(n^2)$ variables.  } and $k$-means++ (with balancing heuristics
  described shortly).  We find that clusterings produced by
  $k$-means++ and the LP rounding algorithms have similar objective
  values and correlate well with the underlying class labels. These
  results complement the results of Section~\ref{sec:stability},
  showing that $k$-means++ produces high quality balanced clusterings
  for `typical' data. This comparison is detailed in
  Sections~\ref{app:experiments} and \ref{app:comparison} of the
  appendix. Based on this observation, our further empirical studies
  use $k$-means++.
  
\item We compare the accuracy of our technique (using $k$-means++ and
  the nearest neighbor extension) to the three baselines for a wide
  range of values of $k$ in large-scale learning tasks where each
  machine learns a local SVM classifier. For all values of $k$ and all
  datasets, our algorithm achieves higher accuracy than all our
  baselines.

\item We show that our framework exhibits strong scaling, meaning that
  if we double the available computing power, the total running time
  reduces by a constant fraction.
\end{itemize}

\paragraph{Experimental Setup} In each run of our experiment, one of
the partitioning algorithms produces a dispatch rule from $10,000$
randomly sampled training points. This dispatch rule is then used to
distribute the training data among the available worker machines. If
the parameter $k$ exceeds the number of machines, we allow each
machine to process multiple partitions independently. Next we train a
one-vs-all linear separator for each partition in parallel by
minimizing the L2-regularized L2-loss SVM objective. This objective is
minimized using Liblinear~\citep{fan2008liblinear} when the data is
small enough to fit in the each worker's memory, and L-BFGS otherwise
(note that both solvers will converge to the same model). The
regularization parameter is chosen via 5-fold cross validation. To
predict the label of a new example, we use the dispatch rule to send
it to the machine with the most appropriate model. All experimental
results are averaged over 10 independent runs.

\paragraph{Details for our technique} Our method builds a dispatch
rule by clustering a small sample of data using $k$-means++ and uses
the nearest neighbor rule to dispatch both the training and testing
data. To ensure a balanced partitioning, we apply the following simple
balancing heuristics: while there is any cluster smaller than $\ell n$
points, pick any such cluster and merge it with the cluster whose
center is nearest. Then each cluster that is larger than $Ln$ points
is randomly partitioned into evenly sized clusters that satisfy the
upper capacity constraint. This guarantees every cluster satisfies the
capacity constraints, but the number of output clusters may differ
from $k$. For the nearest neighbor dispatch, we use the random
partition tree algorithm of \citet{dasgupta2013randomized} for
efficient approximate nearest neighbor search. We set $\ell = 1/(2k)$
and $L = 2/k$ and $p=1$, since our baselines do not support
replication.

\begin{figure*}[h]
  \centering
  \subfigure[][{\scriptsize Accuracy on Synthetic Dataset}]{\includegraphics[width=0.3\textwidth]{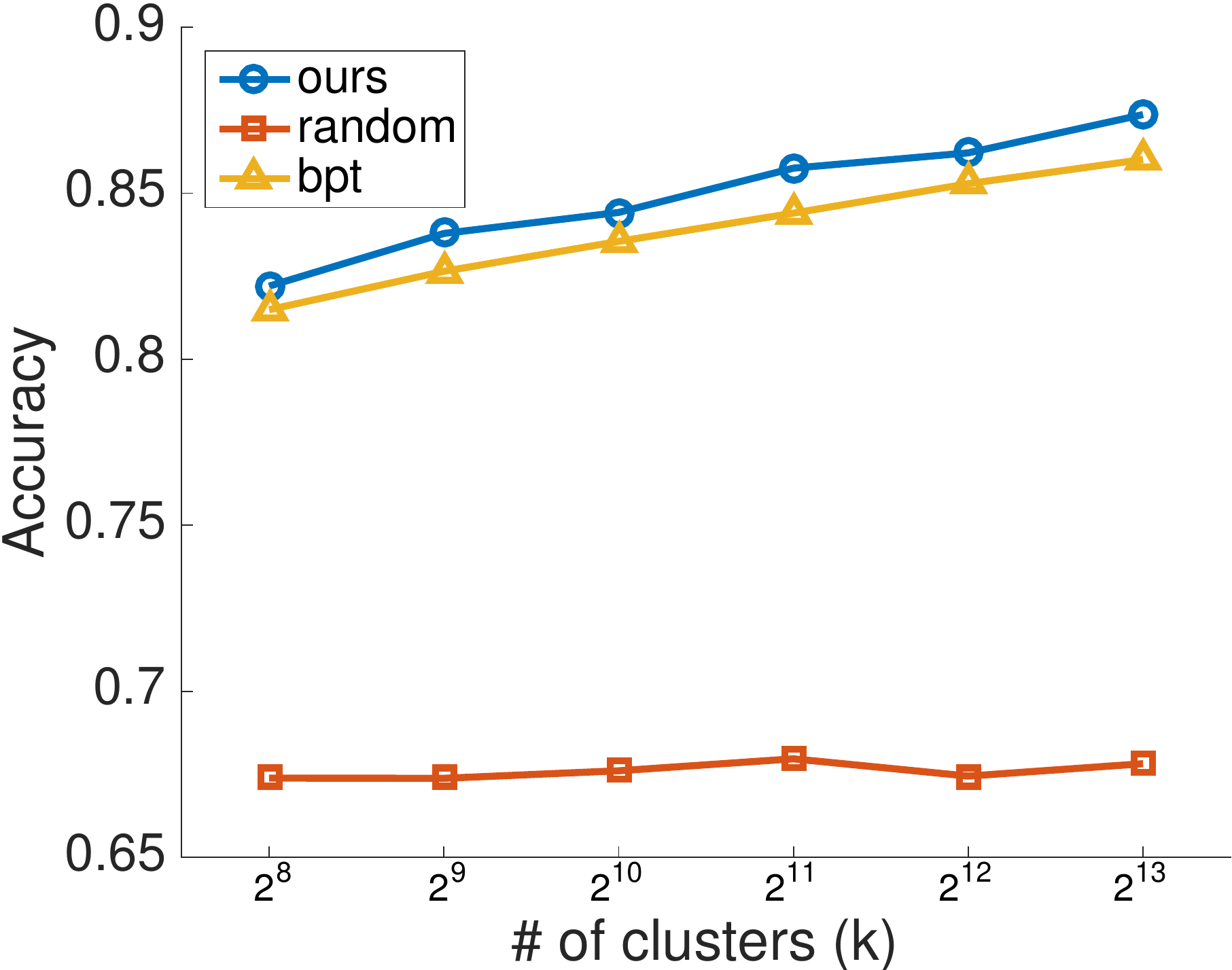}}%
  \subfigure[][{\scriptsize Accuracy on MNIST-8M}]{\includegraphics[width=0.3\textwidth]{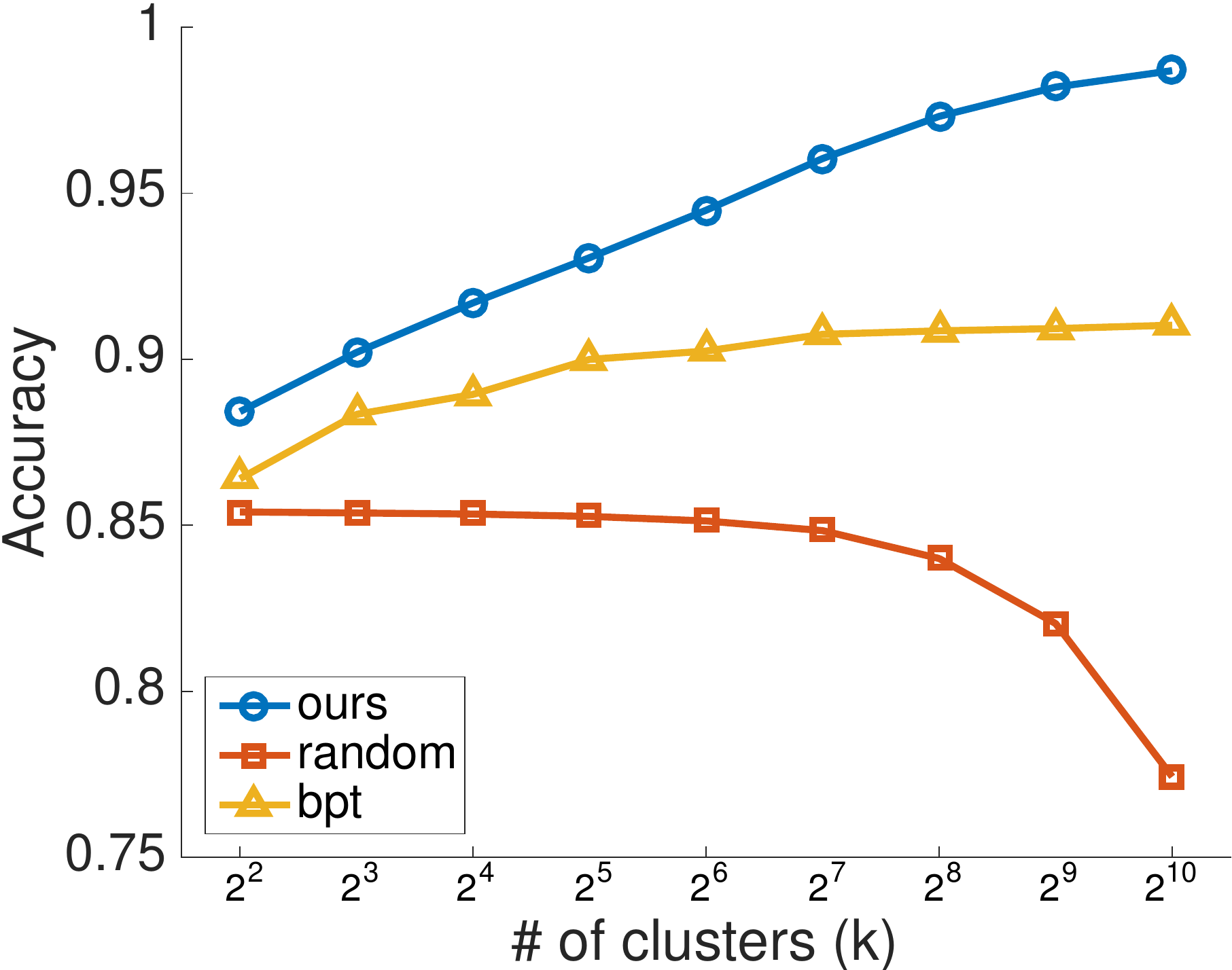}}%
  \subfigure[{\scriptsize Accuracy on CTR
    Dataset}]{\includegraphics[width=0.3\textwidth]{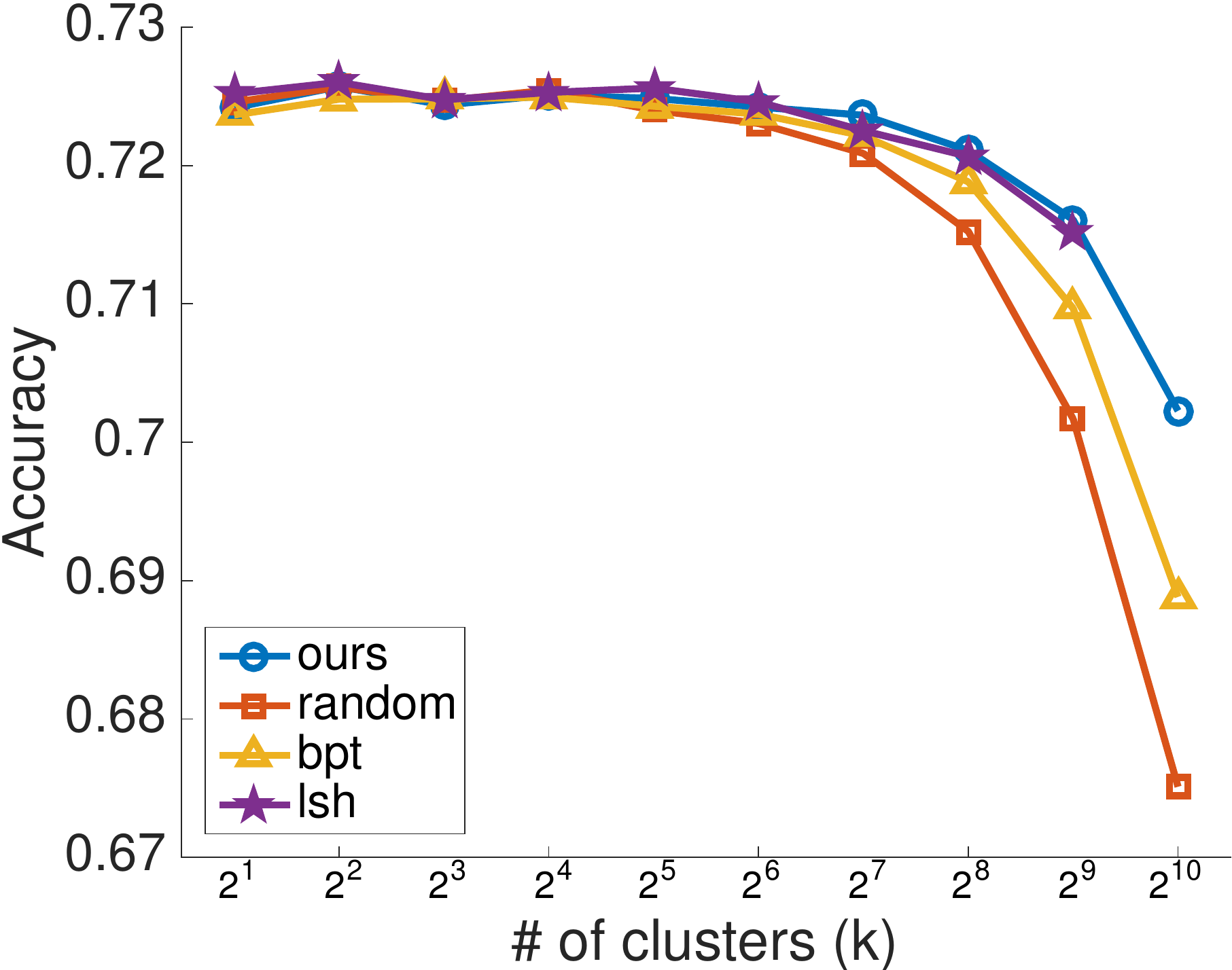}}\\%
  \subfigure[{\scriptsize Accuracy on CIFAR-10 (in3c)}]{\includegraphics[width=0.3\textwidth]{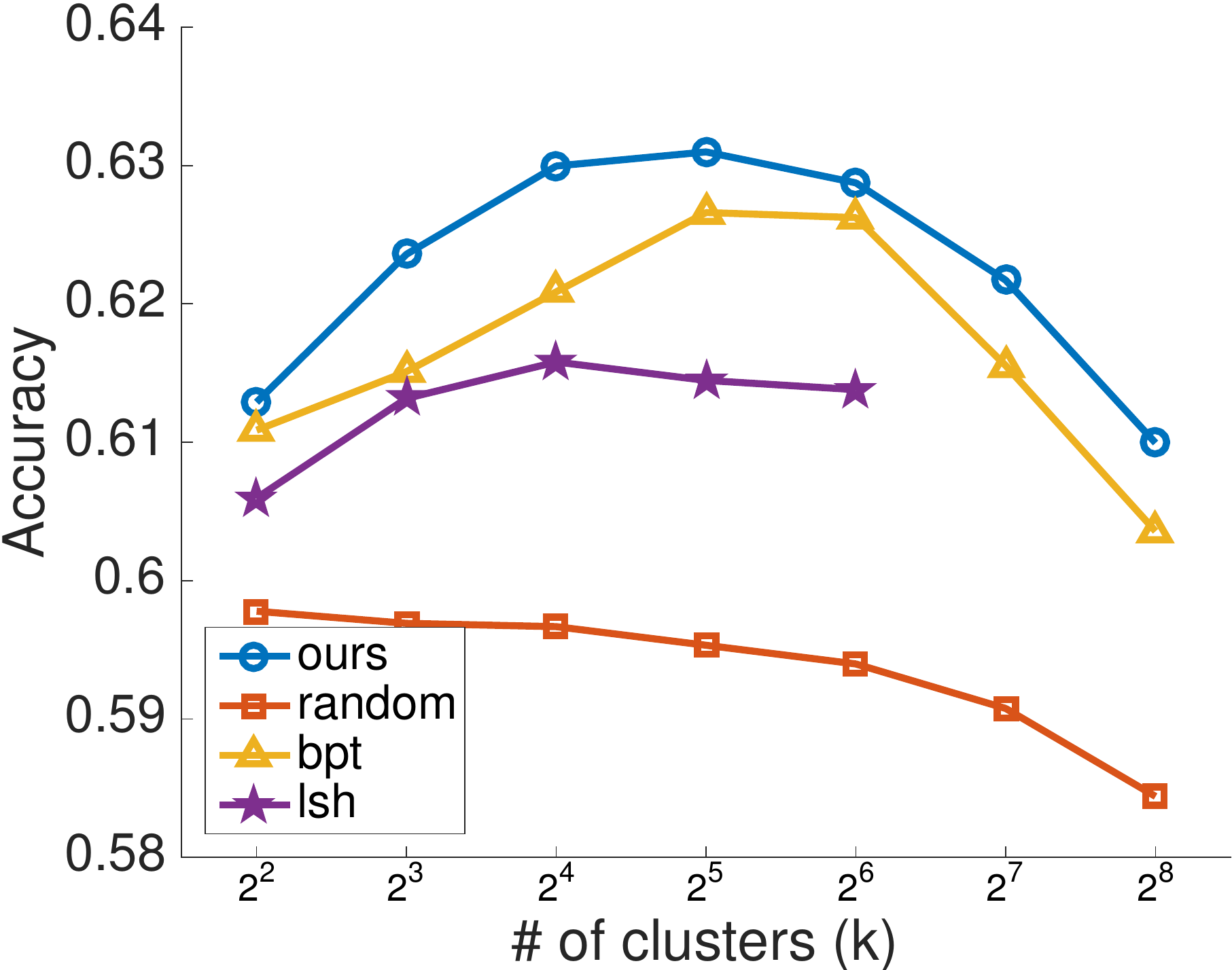}}%
  \subfigure[{\scriptsize Accuracy on CIFAR-10 (in4d)}]{\includegraphics[width=0.3\textwidth]{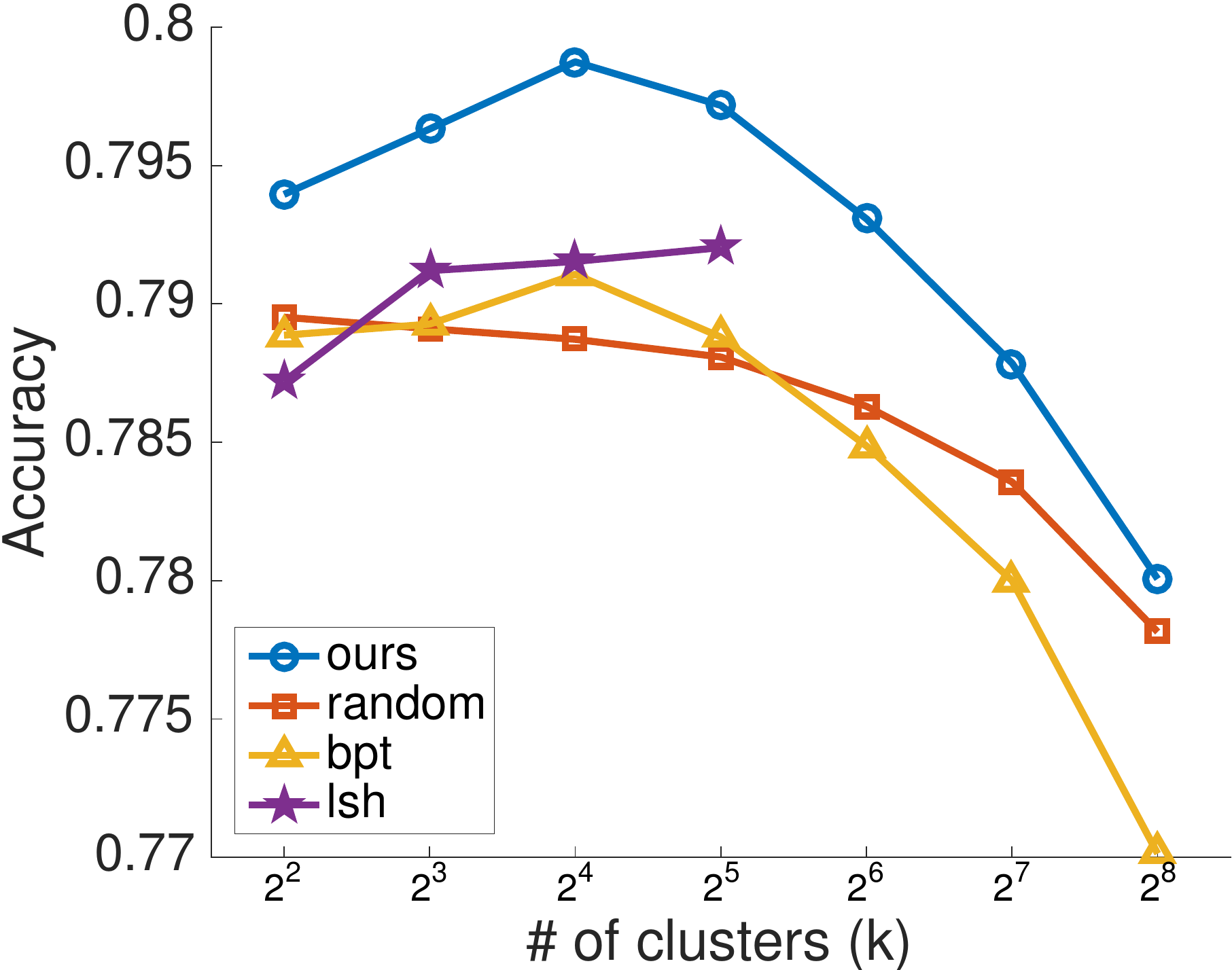}}%
  \subfigure[{\scriptsize \bf Strong Scaling}]{ \includegraphics[width=0.3\textwidth]{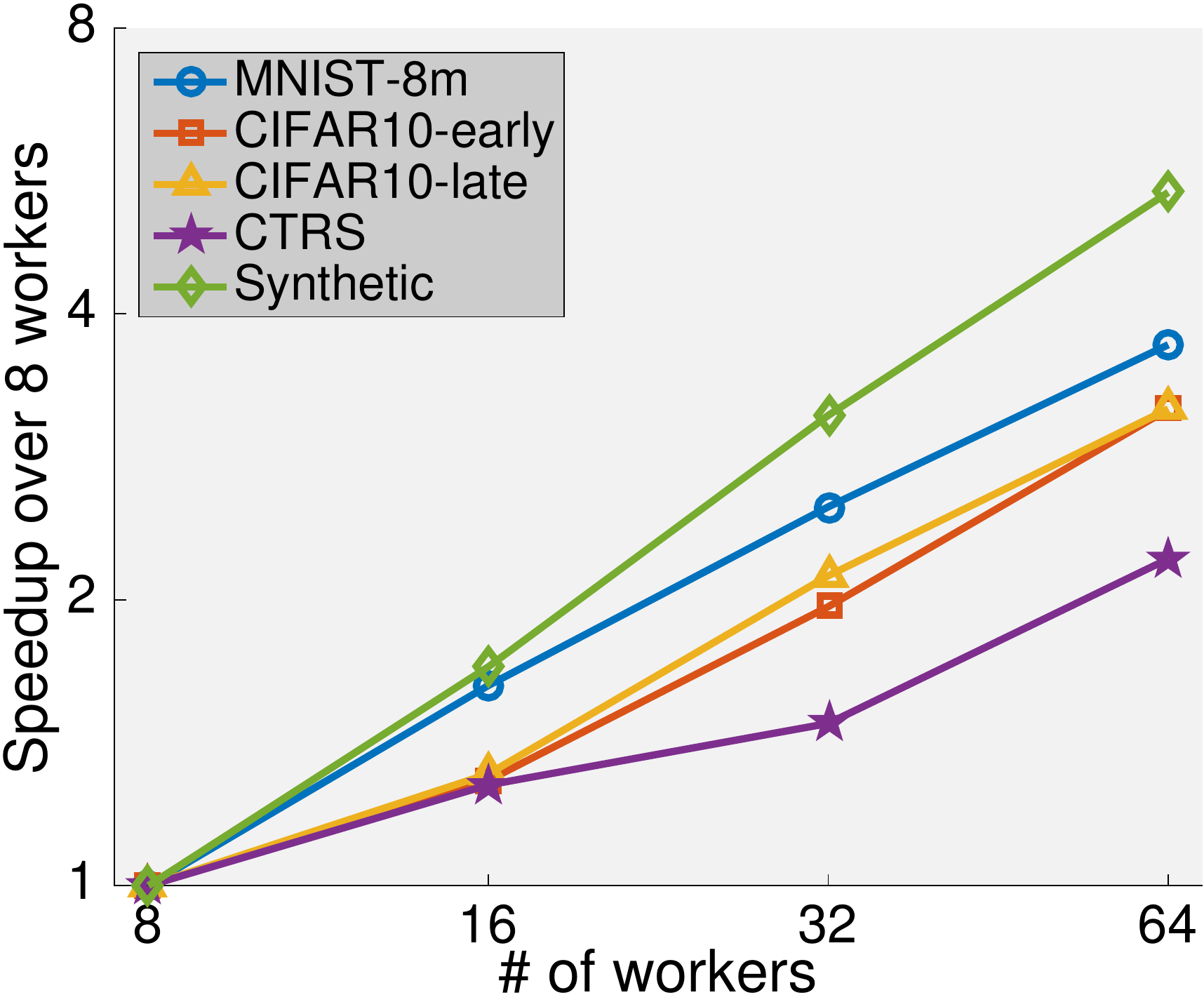}%
    \label{fig:expt:scaling}}%
  \caption{\small Figures (a) through (e) show the effect of $k$ on the classification
    accuracy.
    Figure (f) shows the speedup factor as we increase the number of
    workers from 8 to 64 for each dataset.}
  \label{fig:experiments}
\end{figure*}

\paragraph{Baselines} We compare against the following
baselines:\footnote{Since our framework does not communicate during
  training, we do not compare against algorithms that do,
  e.g. distributed boosting \citep{bbfm12}.}

\begin{description}

\item[Random Partitioning] Points are dispatched uniformly at
  random. This baseline produces balanced partitions but does not send
  similar examples to the same machine.

\item[Balanced Partition Trees] Similarly to a $kd$-tree, this
  partitioning rule recursively divides the dataset by splitting it at
  the median point along a randomly chosen dimension. This is repeated
  until the tree has $k$ leaves (where we assume $k$ is a power of
  2). This baseline produces balanced partitions and improves over
  random partitioning because each machine learns a local model for a
  different subset of the space. The drawback is that the partitioning
  may result in subsets that do not contain similar data points.

\item[LSH Partitioning] This baseline uses locality sensitive hash
  functions \citep{andoni2006near} to dispatch similar points to the
  same machine. Given an LSH family $H$, we pick a random hash
  $h : \reals^d \to \mathbb{Z}$. Then a point $x$ is assigned to
  cluster $h(x) \mod k$. In our experiments, we use the concatenation
  of 10 random projections followed by binning \citep{datar2004}. See
  Section~\ref{app:experiments} in the appendix for details of the
  construction. This baseline sends similar examples to the same
  machine, but does not balance the cluster sizes (which is essential
  for practical data distribution).

\end{description}

\paragraph{Datasets}
We use the following datasets:

\begin{description}
\item[Synthetic:] We use a 128 GB synthetic dataset with 30 classes
  and 20 features. The data distribution is a mixture of 200 Gaussians
  with uniformly random centers in $[0,1]^{20}$ with covariance
  $0.09 I$. Labels are assigned so that nearby Gaussians have the same
  label.

\item[MNIST-8M:] We use the raw pixels of the MNIST-8M dataset
  \citep{loosli2007training}. It has $8M$ examples and $784$ features.

\item[CIFAR-10:] The CIFAR-10 dataset \citep{CIFAR10} is an image
  classification task with 10 classes.  Following
  \citet{krizhevsky2012} we include 50 randomly rotated and cropped
  copies of each training example to get a training set of $2.5$
  million examples. We extract the features from the Google Inception
  network \citep{Szegedy_2015_CVPR} by using the output of an early
  layer (in3c) and a later layer (in4d).

\item[CTR:] The CTR dataset contains ad impressions from a commercial
  search engine where the label indicates whether the ad was
  clicked. It has 860K examples with 232 features.
\end{description}

\paragraph{Results} Our empirical results are shown in Figure
\ref{fig:experiments}. We do not report accuracies when the
partitioning is imbalanced, specifically when the largest $k/2$
clusters contain more than 98\% of the data. In practice, such a
partitioning is essentially using only half of the available computing
power. For all values of $k$ and all datasets, our method has higher
accuracy than all three baselines. The balanced partition tree is the
most competitive baseline, but in Section~\ref{app:experiments} we
present an additional synthetic distribution for which our algorithm
drastically outperforms the balanced partition tree. For all datasets
except CTR, the accuracy of our method increases as a function of $k$,
until $k$ is so large that each cluster becomes data starved.

Our method combines the good aspects of both the balanced partition
tree and LSH baselines by simultaneously sending similar examples to
the same machines and ensuring that every machine gets roughly the
same amount of data. In contrast, the balanced partition tree baseline
both produce balanced clusters, but do not send similar examples to
the same machines, and the LSH baseline sends similar examples to the
same machine, but makes no attempt at balancing the partitions. The
fact that we get higher accuracy than the LSH baseline demonstrates
that it is not enough to send similar examples to the same machines
without balancing, and that we get higher accuracy than balanced
partition trees shows that simply balancing the cluster sizes is not
sufficient.

Figure~\ref{fig:expt:scaling} shows the speedup obtained when running
our system using $16$, $32$, or $64$ workers compared to using $8$.
We clock the time taken for the entire experiment: the time for
clustering a subsample, dispatch, training and testing.  In all cases,
doubling the number of workers reduces the total time by a constant
factor, showing that our framework strongly scales and can be applied
to very large datasets.

\section{Conclusion}
In this work, we propose and analyze a new framework for distributed learning.
Given that similar points tend to have similar classes, we partition the data so that similar examples go to the same machine.
We cast the dispatching step as a clustering problem combined with novel fault tolerance and balance constraints necessary for distributed systems.
We show the added constraints make the objective highly nontrivial, yet we provide LP rounding algorithms with provable guarantees.
This is complemented by our results showing that the $k$-means++ algorithm is competitive on `typical' datasets.
These are the first algorithms with provable guarantees under both upper and lower capacity constraints, and may be of interest beyond distributed learning.
We show that it is sufficient to cluster a small subsample of data and use a nearest neighbor extension technique to efficiently dispatch the remaining data.
Finally, we conduct experiments for all our algorithms that support our theoretical claims, show that our framework outperforms several baselines and strongly scales.

\bibliographystyle{plainnat}
\bibliography{bib/bib,bib/proposal-ana}

\newpage
\section*{Appendix}

\section{Related Work Continued} \label{app:intro}

\subsubsection*{Capacitated $k$-center}
The (uniform) capacitated $k$-center problem is to minimize the maximum distance
between a cluster center and any point in its cluster subject to the constraint
that the maximum size of a cluster is $L$. It is NP-Hard, so research has
focused on finding approximation algorithms. Bar-Ilan et al.\
\citep{Barilan1993385} introduced the problem and presented the first  constant
factor polynomial time algorithm achieving a factor of 10, 
using a combinatorial algorithm which moves around clients until the capacities
are satisfied, and the objective is approximately satisfied.
The approximation factor was improved by
Khuller et al.\ \citep{Khuller96thecapacitated}.
Cygan et al.\ \citep{Cygan_lprounding} give the first algorithm for 
capacitated $k$-center
with non-uniform capacities by using an LP rounding algorithm.
The appoximation factor is not explicitly computed, although it is mentioned to be
in the order of hundreds.
\citep{ckc3} follows a similar procedure but with a dynamic rounding procedure,
and they improve to an approximation factor of 8. 
Further, for the special case of uniform capacities,
they show a 6-approximation.

\subsubsection*{Capacitated $k$-median}
$k$-median with capacities is a notoriously difficult problem in clustering. It
is much less understood than $k$-center with capacities, and uncapacitated
$k$-median, both of which have constant factor approximations. Despite numerous
attempts by various researchers, still there is no known constant factor
approximation for capacitated $k$-median (even though there is no better lower
bound for the problem than the one for uncapacitated $k$-median). As stated
earlier, there is a well-known unbounded integrality gap for the standard LP
even when violating the capacity or center constraints by a factor of
$2-\epsilon$ \citep{aardal2015approximation}.

Charikar et al.\ gave a 16-approximation when constraints are violated by a
factor of 3 \citep{charikar1999}. Byrka et al.\ improved this violation to
$2+\epsilon$, while maintaining an $O(\frac{1}{\epsilon^2})$ approximation
\citep{byrka2015}. Recently, Li improved the latter to $O(\frac{1}{\epsilon})$,
specifically, when constraints are violated by $2+\frac{2}{\alpha}$ for
$\alpha\geq 4$, they give a $6+10\alpha$ approximation \citep{li2014}. These
results are all for the \emph{hard} capacitated $k$-median problem. In the
\emph{soft} capacities variant, we can open a point more than once to achieve
more capacity, although each extra opening counts toward the budget of $k$
centers. In hard capacities, each center can only be opened once. The hard
capacitated version is more general, as each center can be replicated enough
times so that the soft capacitated case reduces to the hard capacitated case.
Therefore, we will only discuss the hard capacitated case.

All of the algorithms for capacitated $k$-median mentioned above share the same
high-level LP rounding and aggregation idea but with different refinements 
in the algorithm and analysis.

\subsubsection*{Universal and load balanced facility location}

In the facility location problem, we are given a set of demands and a set of
possible locations for facilities. We should open facilities at some of these
locations, and connect each demand point to an open facility so as to minimize
the total cost of opening facilities and connecting demands to facilities.
Capacitated facility location is a variant where each facility can supply only a
limited amount of the commodity. This and other special cases are captured by
the Universal Facility Location problem where the facility costs are general
concave functions. Local search techniques \citep{Mahdian03universalfacility}
have been proposed and applied successfully. Also, LP rounding techniques suffer
from unbounded integrality gap for capacitated facility location
\citep{Mahdian03universalfacility}.

Load-balanced facility location \citep{conf/focs/KargerM00},
\citep{conf/focs/GuhaMM00}, is yet another variant where every open facility must
cater to a minimum amount of demand. An unconstrained facility location problem
with modified costs is constructed and solved. Every open facility that does not
satisfy the capacity constraint is closed and the demand is rerouted to nearby
centers. The modified problem is constructed so as to keep this increase in cost
bounded.

\section{Details from Section \ref{sec:bicriteria}} \label{app:bicriteria}

In this section, we provide the formal details for the bicriteria algorithm presented
in Section \ref{sec:bicriteria}.

For convenience, we restate Theorem~\ref{thm:bicriteria} here.

\medskip
\noindent \textbf{Theorem~\ref{thm:bicriteria} (restated).}
\emph{
Algorithm \ref{alg:overview} returns a constant factor approximate solution for the
balanced $k$-clustering with $p$-replication problem for $p> 1$, where the upper capacity constraints are violated
by at most a factor of $\frac{p+2}{p}$,
and each point can be assigned to each cluster at most twice.
}
\medskip

\paragraph{Step 1 details}

We restate the LP for $k$-means and $k$-median for completeness, labeling each constraint
for the proofs later.

  \begin{subequations}
  \label{LP kmed}
  \begin{align}
  \text{min } &\sum_{i,j\in V} c_{ij} x_{ij}  \tag{LP.1}\label{lp:constraint1} \\
  \text{subject to: } & \sum_{i\in V} x_{ij}=p, \,\, \forall j\in V  \tag{LP.2}\label{lp:constraint2}\\
   &\ell y_i \leq \sum_{j\in V} \frac{x_{ij}}{n} \leq L y_i, \,\, \forall i\in V \tag{LP.3}\label{lp:constraint3} \\
  &\sum_{i\in V} y_i\leq k; \tag{LP.4}\label{lp:constraint4} \\
  &0\leq x_{ij}\leq y_i\leq 1, \,\, \forall i,j\in V. \tag{LP.5}\label{lp:constraint5}
  \end{align}
  \end{subequations}

As mentioned in Section \ref{sec:bicriteria}, it is well-known that the standard capacitated $k$-median LP 
(this LP, without the lower bound constraint and with $p=1$)
has an unbounded integrality gap, 
even when the capacities are violated by
a factor of $2-\epsilon$ \citep{aardal2015approximation}.
The integrality gap is as follows.
$k=2nL-1$, and there are $nL$ groups of size $2nL-1$.
Points in the same group are distance 0, and points in different groups are distance 1.
Fractionally, we can open $2-\frac{1}{nL}$ facilities in each group to achieve cost 0.
But integrally, some group contains at most 1 facility, and thus the capacity violation must
be $2-\frac{1}{nL}$.

However,
with $p$ replication, there must be $p$ centers per group, so the balance violation can
be split among the $p$ centers.
Therefore, the integrality is only unbounded when the capacities are violated by
a factor of $\frac{p}{p-1}$

The $k$-center LP is a little different from the $k$-median/means LP.
As in prior work
\citep{ckc3,Cygan_lprounding,Khuller96thecapacitated},
we guess the optimal radius, $t$.
Since there are a polynomial number of choices for $t$, we can try all of them
to find the minimum possible $t$ for which the following program
is feasible.
Here is the LP for $k$-center.

\begin{subequations}
\label{LP kcenter}
\footnotesize{
\begin{align}
&  \sum_{i\in V} x_{ij}=p, & \forall j\in V \\
& n\ell y_i \leq \sum_{j\in V} x_{ij} \leq nL y_i, &\forall i\in V \\
&\sum_{i\in V} y_i\leq k; \\
&0\leq x_{ij}\leq y_i, &\forall i,j\in V \\
& x_{ij} = 0 & \text{if } d(i,j) > t \label{eqn:lp2_threshold}.
\end{align}
}
\end{subequations}

For $k$-median and $k$-means,
let $C_{LP}$ denote the objective value.
For $k$-center, $C_{LP}$ would be the smallest threshold $t$ at which
the LP is feasible, however
we scale it as $C_{LP} = tnp$ for consistency with the other objectives.
For all $j\in V$, define the connection cost $C_j$ as the
average contribution of a point to the objective.
For $k$-median and $k$-means,
it is $C_j=\frac{1}{p}\sum_{i\in V} c_{ij} x_{ij}$.
That is, for $k$-median, it is the average distance
of a point to its fractional centers while for $k$-means,
it is the average squared distance of a point to
its fractional centers.
For $k$-center,
$C_j$ is simply the threshold $C_j=t$.
Therefore, $C_{LP}=\sum_{j \in V} p C_j$ in all cases.

The notation is summarized in table \ref{tab:notation}.

\begin{table*}
\caption{ Notation table}
\begin{center}
\begin{tabular}{|c||c|c|c|c|}
	\hline
	Symbol & Description & $k$-median & $k$-means & $k$-center \\
	\hline
	$y_i$ & Fractional opening at center $i$ & \multicolumn{3}{c|}{-}\\
	\hline
	$x_{ij}$ & Fractional assignment of point $j$ to center $i$ &\multicolumn{3}{c|}{-}\\
	\hline
	$c_{ij}$ & Cost of assigning $j$ to center $i$ & $d(i,j)$ & $d(i,j)^2$ & $t$ \\
	\hline
	$C_{j}$ & Avg cost of assignment of point $j$ to all its centers & \multicolumn{2}{c|}{$ \sum_i c_{ij} x_{ij} / {p} $} & $t$\\
	\hline
	$C_{LP}$ & Cost of LP & \multicolumn{3}{c|}{$\sum_j pC_{j} $} \\
	\hline
	$\rho$ & parameter for monarch procedure & 2 & 4 & 1 \\
	\hline
\end{tabular}
\end{center}
\label{tab:notation}
\end{table*}

\paragraph{Step 2 details}
Let $\mathcal{M}$ be the set of monarchs, and
for each $u \in \mathcal{M}$, denote $\mathcal{E}_u$ as the empire
of monarch $u$.
Recall that the contribution of an assignment to the objective
$c_{ij}$ is $d(i,j)$ for $k$-median, $d(i,j)^2$ for $k$-means,
and $t$ for $k$-center.
We also define a parameter $\rho=1$ for $k$-center, $\rho=2$ for $k$-median,
and $\rho=4$ for $k$-means, for convenience.

Initially set $\mathcal{M}=\emptyset$.
Order all points in nondecreasing order of $C_i$.
For each point $i$, if $\exists j\in \mathcal{M}$ such that $c_{ij}\leq 2t C_i$,
continue. Else, set $\mathcal{M}=\mathcal{M}\cup\{i\}$.
At the end of the for loop,
assign each point $i$ to cluster $\mathcal{E}_u$ such that $u$ is the closest point in $\mathcal{M}$ to $i$.
See Algorithm \ref{algo:monarchs_bicriteria}.

\begin{algorithm}
\DontPrintSemicolon
\KwIn{$V$ and fractional $(x,y)$}
\KwOut{Set of monarchs, $\mathcal{M}$, and empire $\mathcal{E}_j$ for each monarch $j \in \mathcal{M}$}
\nl $\mathcal{M} \gets \emptyset$\;
\nl Order all points in non-decreasing order of $C_i$\;
\nl // Identify Monarchs \;
\nl \ForEach {$i \in V$} {
\nl	\If {\upshape$\nexists j\in \mathcal{M} \text{ such that } c_{ij}\leq 2\rho C_i$}{
\nl	$\mathcal{M}\gets \mathcal{M}\cup\{i\}$\;
}
}
\nl // Assign Empires as Voronoi partitions around monarchs \;
\nl \ForEach {$j \in V$} {
\nl	Let $u \in \mathcal{M}$ be the closest monarch to $j$\;
\nl $\mathcal{E}_u \gets \mathcal{E}_u \cup \{j\}$\;
}
\caption{{\bf Monarch procedure for coarse clustering}: Greedy algorithm to create monarchs and assign empires}
\label{algo:monarchs_bicriteria}
\end{algorithm}

We restate Lemma \ref{lem:monarch_bicriteria} and give a full proof.

\medskip
\noindent \textbf{Theorem~\ref{lem:monarch_bicriteria} (restated).}
\emph{
Let $\rho$ be a parameter such that $\rho=1$ for $k$-center, $\rho=2$ for $k$-median,
and $\rho=4$ for $k$-means.
The output of the monarch procedure satisfies the following properties:
\begin{enumerate}[label=(1\alph*)]
  \item \label{lemsub:partition} The clusters partition the point set;
  \item \label{lemsub:close1} Each point is close to its monarch: $\forall j\in
  \mathcal{E}_u, u\in \mathcal{M}, c_{uj}\leq 2\rho C_j$;
  \item \label{lemsub:n'far} Any two monarchs are far apart: $\forall u,u' \in
  \mathcal{M}\text{ s.t. } u\neq u', c_{uu'}>4\max\{C_u,C_{u'}\}$;
  \item \label{lemsub:min} Each empire has a minimum total opening: $\forall
  u\in \mathcal{M}, \sum_{j\in \mathcal{E}_u} y_j\geq \frac{p}{2}$
	(or for $k$-center, $\sum_{j\in \mathcal{E}_u} y_j\geq p$).
\end{enumerate}
}
\medskip

\begin{proof}
The first three properties follow easily from construction (for the third property, recall we ordered the points at the start of the monarch procedure). Here is the proof of the final property,
depending on the objective function.

For $k$-center and $k$-median,
it is clear that for some $u \in \mathcal{M}$, if $d(i, u) \leq \rho C_u$,
then $i \in \mathcal{E}_u$ (from the triangle inequality and Property \ref{lemsub:n'far}).
For $k$-means, however: if $d(i,u)^2\leq 2C_u$, then $i\in \mathcal{E}_u$.
Note that the factor is $\rho/2$ for $k$-means. This is because of the triangle
inequality is a little different for squared distances.

To see why this is true for $k$-means, assume towards contradiction that
$\exists i\in V$, $u,u'\in \mathcal{M}$, $u\neq u'$ such that $u\in \mathcal{E}_{u'}$ and
$d(i,u)^2\leq 2 C_u$.
Then $d(i,u')\leq d(i,u)$ by construction.
Therefore, $d(u,u')^2\leq (d(u,i)+d(i,u'))^2\leq 4d(i,u)^2\leq 8 C_u$, and we have reached a contradiction
by Property \ref{lemsub:n'far}.

Now, to prove property \ref{lemsub:min}:

\paragraph{$k$-center}
From the LP constraints, for every $u$, $\sum_{j \in V} x_{ju} = p$. But $x_{ju}$ is non-zero only they are separated by at most $t$, the threshold.
Combining this with the fact that if $d(j, u) \leq C_u = t$, then $j \in \mathcal{E}_u$, we get, for each $u \in \mathcal{M}$:
\begin{equation*}
	\sum_{j \in \mathcal{E}_u} y_j \geq \sum_{j \in \mathcal{E}_u} x_{ju} = p
\end{equation*}

\paragraph{$k$-median and $k$-means}
Note that $C_u$ is a weighted average of costs $c_{iu}$ with weights $x_{iu}/p$,
i.e., $C_u=\sum_i c_{iu}\nicefrac{x_{iu}}{p}$.
By Markov's inequality,
\begin{equation*}
	\sum_{j:c_{ju} > 2 C_u} \frac{x_{ju}}{p} <
	\frac{C_u} { 2 C_u} = \frac 1 2
\end{equation*}
Combining this with the fact that if $c_{ju} \leq 2 C_u$, then $j \in \mathcal{E}_u$ for both $k$-median and $k$-means
, we get, for each $u \in \mathcal{M}$:
\begin{equation*}
	\sum_{j \in \mathcal{E}_u} y_j \geq \sum_{j: c_{ju} \leq 2 C_u} y_j \geq \sum_{j: c_{ju} \leq 2 C_u} x_{ju} \geq \frac p 2.
\end{equation*}

\end{proof}

\paragraph{Step 3 Details}
First we define the suboperation Move \citep{li2014}:
\begin{definition} [Operation ``Move'']
The operation ``Move'' moves a certain opening $\delta$ from $a$ to $b$.
Let $(x', y')$ be the updated $(x, y)$ after a movement of $\delta \leq y_a$ from $a$ to $b$.
Define
	\begin{align*}
		y_a' &= y_a - \delta \\
		y_b' &= y_b + \delta \\
		{\forall u\in V,~}x_{au}' &= x_{au} (1- \nicefrac{\delta}{y_a}) \\
		{\forall u\in V,~}x_{bu}' &= x_{bu} + x_{au} \cdot \nicefrac{\delta}{y_a}
	\end{align*}
\end{definition}

It has been proven in previous work that the move operation does not violate
any of the LP constraints except the constraint that $y_i \leq 1$ \citep{li2014}. 
We provide a proof below for completeness.
Should we require $\delta \leq \min(y_a, 1-y_b)$, the constraint $y_i \leq 1$ would not be violated. But to get a bicriteria approximation, we allow this violation.
The amount by which the objective gets worse can then be bounded by the triangle inequality.

\begin{lemma}
\label{lem:move_properties}
The operation Move does not violate any of the LP constraints except possibly the constraint $y_i \leq 1$ and the threshold constraint \ref{eqn:lp2_threshold} of $k$-center.
\end{lemma}

\begin{proof}

To show that the \emph{Move} operation satisfies all the LP constraints,
first note that the only quantities that change are $y_a, y_b, x_{au}, x_{bu},~\forall u\in V$.
Further, $x, y$ satisfy all the constraints of the LP. Using this,
\begin{itemize}
	\item Constraint \ref{lp:constraint1}:
		For every $u$, $\sum_i x_{iu}' = \sum_i x_{iu} = p$.
	\item Constraint \ref{lp:constraint2} (1):
		\begin{align*}
		\sum_u x_{au}' &= \sum_u x_{au} (1- \nicefrac{\delta}{y_a}) \leq nLy_a (1- \nicefrac{\delta}{y_a}) = nLy_a' \\
		\sum_u x_{bu}' &= \sum_u x_{bu} + \sum_u x_{au} \cdot \nicefrac{\delta}{y_a}\\  &\leq nLy_{b} + nLy_{a} \cdot \nicefrac{\delta}{y_a} = nLy_b'
		\end{align*}
	\item Constraint \ref{lp:constraint2} (2):
		\begin{align*}
		\sum_u x_{au}' &= \sum_u x_{au} (1- \nicefrac{\delta}{y_a}) \geq n\ell y_a (1- \nicefrac{\delta}{y_a}) = n\ell y_a' \\
		\sum_u x_{bu}' &= \sum_u x_{bu} + \sum_u x_{au} \cdot \nicefrac{\delta}{y_a}\\  &\geq n\ell y_{b} + n\ell y_{a} \cdot \nicefrac{\delta}{y_a} = n\ell y_b'
		\end{align*}
	\item Constraint \ref{lp:constraint3}: $\sum_i y_i' = \sum_i y_i \leq k$
	\item Constraint \ref{lp:constraint4} (1):
		\begin{align*}
			x_{au}' &= x_{au} (1- \nicefrac{\delta}{y_a}) \leq y_a  (1- \nicefrac{\delta}{y_a}) = y_a' \\
						 x_{bu}' &= x_{bu} + x_{au} \cdot \nicefrac{\delta}{y_a} \leq
						y_{b} + y_{a} \cdot \nicefrac{\delta}{y_a} = y_b'.
		\end{align*}
	\item Non-negative constraint: this is true since $\delta \leq y_a$.
\end{itemize}
\end{proof}

See Algorithm \ref{algo:aggregation} for the aggregation procedure.

\begin{algorithm}
\DontPrintSemicolon
\KwIn{$V$, fractional $(x,y)$, empires $\{\mathcal{E}_j\}$}
\KwOut{updated $(x,y)$}
\nl \ForEach {$\mathcal{E}_u$} {
\nl Define $Y_u=\sum_{i\in \mathcal{E}_u} y_i$, $z_u=\frac{Y_u}{\lfloor Y_u\rfloor}$.\\
\nl \While { $\exists v$ s.t. $y_v\neq z_u$} {
\nl Let $v$ be the point farthest from $u$ with nonzero $y_v$.\\
\nl Let $v'$ be the point closest to $j$ with $y_{v'}\neq z_u$.\\
\nl Move $\min\{y_v,z_u-y_{v'}\}$ units of opening from $y_v$ to $y_{v'}$.
}
}
\caption{{\bf Aggregation procedure}}
\label{algo:aggregation}
\end{algorithm}

We restate Lemma \ref{lem:costs} and give a full proof.

\medskip
\noindent \textbf{Theorem~\ref{lem:costs} (restated).}
\emph{
$\forall j\in V \text{ whose opening moved from }i'\text{ to }i$,
\begin{itemize}[itemsep=3pt,topsep=2pt,parsep=0pt,partopsep=0pt]
\item $k$-center: $d(i,j)\leq 5t$,
\item $k$-median: $d(i,j)\leq 3d(i',j)+8 C_j$,
\item $k$-means:    $d(i,j)^2\leq 15d(i',j)^2+80C_j$.
\end{itemize}
}
\medskip

\begin{proof}

{\bf $k$-center.} Use the fact that
all $C_j = t$, and $x_{ij}>0\implies d(i,j)\leq t$
with property \ref{lemsub:close1} to get:
\begin{align*}
	d(i,j)&\leq d(i,u)+d(u,i')+d(i',j) \\
&\leq 2C_i+2C_{i'}+d(i',j)\leq 5t.
\end{align*}

{\bf $k$-median.}
By construction, if the demand of point $j$ moved from $i'$ to $i$,
then $\exists u\in \mathcal{M}$ s.t. $i,i'\in \mathcal{E}_u$ and $d(u,i)\leq d(u,i')$.
Denote $j'$ as the closest point in $\mathcal{M}$ to $j$. Then $d(u,i')\leq d(j',i')$ because $i'\in \mathcal{E}_u$.
Then,
\begin{align*}
d(i,j)&\leq d(i,u)+d(u,i')+d(i',j) \\
&\leq 2d(u,i')+d(i',j) \\
&\leq 2d(j',i')+d(i',j) \\
&\leq 2(d(j',j)+d(j,i'))+d(i',j) \\
&\leq 8 C_j + 3d(i',j).
\end{align*}

\paragraph{$k$-means}
The argument is similar to $k$-median, but with a bigger constant factor
because of the squared triangle inequality.
\begin{align*}
d(i,j)^2&\leq (d(i,u)+d(u,i')+d(i',j))^2 \\
&\leq (2d(u,i')+d(i',j))^2 \\
&\leq 4d(u,i')^2+d(i',j)^2+4d(u,i')d(i',j) \\
&\leq 4d(u,i')^2+d(i',j)^2+4d(u,i')d(i',j)\\
&\quad+(2d(i',j)-d(u,i))^2 \\
&\leq 5d(u,i')^2+5d(i',j)^2 \\
&\leq 5d(j',i')^2+5d(i',j)^2 \\
&\leq 5(d(j',j)+d(j,i'))^2+5d(i',j)^2 \\
&\leq 5d(j',j)^2+10d(i',j)^2+10d(j',j)d(i',j)\\
&\leq 5d(j',j)^2+10d(i',j)^2+10d(j',j)d(i',j)\\
&\quad+5(d(j',j)-d(i',j))^2\\
&\leq 10d(j',j)^2+15d(i',j)^2 \\
&\leq 80 C_j+15d(i',j)^2.
\end{align*}

\end{proof}

\definecolor{myblue}{RGB}{80,80,160}
\definecolor{mygreen}{RGB}{80,160,80}
\definecolor{myblack}{RGB}{10,10,10}

\paragraph{Step 4 details}
Set $\{i\mid y_i\neq 0\}=Y$. We show details of the min cost flow network in Algorithm \ref{algo:flow}.

\begin{algorithm}
\DontPrintSemicolon
\KwIn{$V$, $(x,y)$, $y$ are integral }
\KwOut{updated $(x,y)$ with integral $x$'s and $y$'s}
\nl Create a flow graph $G=(V',E)$ as follows.\\
\nl Add each $i\in V$ to $V'$, and give $i$ supply $p$.\\
\nl Add each $i\in Y$ to $V'$, and give $i$ demand $n\ell$.\\
\nl Add a directed edge $(i,j)$ for each $i\in V$, $j\in Y$, with capacity
2 and cost $c_{ij}$ (for $k$-center, make the edge weight $5t$ if $d(i,j)\leq 5t$
and $+\infty$ otherwise.\\
\nl Add a sink vertex $v$ to $V'$, with demand $np-kn\ell$.\\
\nl Add a directed edge $(i,v)$ for each $i\in Y$, with capacity
$\lceil\frac{p+2}{p}nL\rceil-n\ell$ and cost 0.\\
\nl Run an min cost integral flow solver on $G$.\\
\nl Update $x$ by setting $x_{ij}$ to 0, 1, or 2 based on the amount of
flow going from $i$ to $j$.
\caption{{\bf Min cost flow procedure}: Set up flow problem to round $x$'s}
\label{algo:flow}
\end{algorithm}

\begin{lemma} \label{lem:x's}
There exists an integral assignment of the $x'_{ij}$'s such that $\forall i,j\in V$,
$ x'_{ij}\leq 2$ and it can be found in polynomial time.
\end{lemma}

\begin{proof}

See Algorithm \ref{algo:flow} and Figure \ref{fig:flow1x} for the details of the flow construction.

In this graph, there exists a feasible flow: $\forall i,j\in V$,
send $x'_{ij}$ units of flow along the
edge from $i$ to $j$, and send $\sum_{j\in V} x_{ij}$ units of flow along the edge from $i$ to $v$.
Therefore, by the integral flow theorem, there exists
a maximal integral flow which we can find in polynomial time.
Also, by construction, this flow corresponds to an integral assignment of the $x'_{ij}$'s
such that $x'_{ij}\leq 2$.
\end{proof}

Now we are ready to prove Theorem \ref{thm:bicriteria}.
The approximation ratios are 5, 11, and 95 for $k$-center, $k$-median, and $k$-means, respectively.

\begin{proof}[Proof of Theorem \ref{thm:bicriteria}]
$\,$\\
\textbf{$k$-center:}
Recall that we defined $C_{LP} = tnp$, where $t$ is the threshold for the k-center LP.
From Lemma \ref{lem:costs}, when we reassign the demand of point $j$ from $i'$ to $i$,
$d(i,j)\leq 5t$. In other words, the y-rounded solution is feasible at threshold $5t$.
Then the $k$-center cost of the new $y$'s is $np(5t)=5C_{LP}$.
From Lemma \ref{lem:x's}, we can also round the $x$'s at no additional cost.

\textbf{$k$-median:}
From Property \ref{lem:costs}, when we reassign the demand of point $j$ from $i'$ to $i$,
$d(i,j)\leq 3d(i',j)+8 C_j$.
Then we can bound the cost of the new assignments with respect to the original LP solution
as follows.
\begin{align*}
\sum_{i\in V}\sum_{j\in V} d(i,j) x'_{ij}&\leq \sum_{i\in V}\sum_{j\in V} (8C_j+3d(i,j))x_{ij} \\
&\leq \sum_{i\in V}\sum_{j\in V} 8 C_j x_{ij}\\
&\quad+\sum_{i\in V}\sum_{j\in V} 3d(i,j) x_{ij} \\
&\leq \sum_{j\in V}8 C_j \sum_{i\in V} x_{ij} + 3 C_{LP} \\
&\leq \sum_{j\in V}8p C_j+3 C_{LP}\leq 11 C_{LP}.
\end{align*}

Then from Lemma \ref{lem:x's}, we get a solution of cost at most $11 C_{LP}$, which also has
integral $x$'s.

\textbf{$k$-means:}
The proof is similar to the $k$-median proof.
From lemma \ref{lem:costs}, when we reassign the demand of point $j$ from $i'$ to $i$,
$d(i,j)^2\leq 15d(i',j)^2+80C_j$.
Then we can bound the cost of the new assignments with respect to the original LP solution
as follows.
\begin{align*}
\sum_{i\in V}\sum_{j\in V} d(i,j)^2 x'_{ij}&\leq \sum_{i\in V}\sum_{j\in V} (80C_j+15d(i',j)^2)x_{ij} \\
&\leq \sum_{i\in V}\sum_{j\in V} 80 C_j x_{ij}+\sum_{i\in V}\sum_{j\in V} 15d(i,j)^2 x_{ij} \\
&\leq \sum_{j\in V}80 C_j \sum_{i\in V} x_{ij} + 15 C_{LP} \\
&\leq \sum_{j\in V}80p C_j+15 C_{LP}\leq 95 C_{LP}.
\end{align*}

Then from Lemma \ref{lem:x's}, we get a solution of cost at most
$95 C_{LP}$, which also has
integral $x$'s.
\end{proof}

See Algorithm \ref{alg:bicriteria} for the final algorithm.

\begin{algorithm}
\DontPrintSemicolon
\KwIn{$V$ }
\KwOut{Integral $(x,y)$ corresponding to bicriteria clustering solution}
\nl Run a solver for the LP relaxation for $k$-median, $k$-means, or $k$-center, output $(x,y)$.\\
\nl Run Algorithm \ref{algo:monarchs_bicriteria} with $V$, $(x,y)$, output set of empires $\{\mathcal{E}_j\}$.\\
\nl Run Algorithm \ref{algo:aggregation} with $V$, $\{\mathcal{E}_j\}$, $(x,y)$, output updated $(x,y)$.\\
\nl Run Algorithm \ref{algo:flow} with $V$, $(x,y)$, output updated $(x,y)$.
\caption{{\bf Bicriteria approximation Algorithm for $k$-median, $k$-means, and $k$-center}}
\label{alg:bicriteria}
\end{algorithm}

\section{$k$-center} \label{sec:kcenter}
In this section, we present a more complicated algorithm that is specific
to $k$-center, which achieves a true approximation algorithm -
the capacity and replication constraints are no longer violated.

\subsection* {Approach}

As in the previous section and in prior work
\citep{ckc3,Cygan_lprounding,Khuller96thecapacitated},
 we start off by guessing the optimal distance $t$. Since
there are a polynomial number of possibilities, it is still only polynomially expensive.
We then construct the threshold graph $G_t = (V, E_t)$,
with $j$ being the set of all points, and $(x, y) \in E_t$ iff $d(x, y) \leq t$.

A high-level overview of the rounding algorithm that follows is given in
Algorithm \ref{algo:overview_balanced_clustering}.

\paragraph{Connection to the previous section} The algorithm here is similar to the
bicriteria algorithm presented previously.
There are, however, two differences. Firstly, we work only
with connected components of the threshold graph.
This is necessary to circumvent the unbounded integrality gap of the LP \citep{Cygan_lprounding}.
Secondly, the rounding procedure of the $y$'s can now move opening
across different empires. Since the threshold graph is connected,
the distance between any two adjacent monarchs is bounded and turns out to exactly be
thrice the threshold. This enables us to get a constant factor approximation
without violating any constraints.

\begin{algorithm}
\DontPrintSemicolon
\KwIn{$V$: the set of points, $k$: the number of clusters, $(\ell,L)$: min and max allowed cluster size }
\KwOut{A $k$-clustering of $V$ respecting cluster size constraints, $p$: replication factor}
\SetKwFunction{procGlobal}{balanced-k-center}
\SetKwFunction{proc}{LPRound}\SetKwFunction{proclocal}{yRound}  
\SetKwProg{myproc}{Procedure}{}{} 

\myproc{\procGlobal{$V, k, p, \ell, L$}}{

	\ForEach {\upshape threshold $t$}{
		Construct the threshold graph $G_t$\;
		\ForEach {\upshape connected component $G^{(c)}$ of $G_t$}{
			\ForEach{\upshape $k'$ in $1,...k$}{
				// {\it Solve balanced $k'$-clustering on $G^{(c)}$}\;
				Solve \proc{$G^{(c)}, k', p, \ell, L$}\;
			}
	  	 }
		Find a solution for each $G^{(c)}$ with $k_c$ centers such that $\sum_c k_c = k$ by linear search; call is $s$\;
		\lIf{\upshape no such a solution exists}{
			\Return ``No Solution Found''
		}
		\lElse {
			\Return  solution $s$
		}\;
	}
}
\myproc{\proc{$G, k, p, \ell, L$}}{ 
	$(x, y) \gets $ relaxed solution of LP in equation \ref{LP}\;
	$(x', y') \gets \proclocal(G, x, y)$\;
	Round $x'$ to get $x''$ from theorem \ref{thm:xround}\;
	\Return $(x'', y')$
}
\myproc{\proclocal{$G, x, y$}}{ 
	Construct coarse clustering to get a tree of clusters from algorithm \ref{algo:monarchs} \;
	Round clusters in a bottom up manner in the tree, moving mass around to nodes within a distance of 5 away (algorithm \ref{algo:round})\;
	\Return rounded solution with integral $y$ \;
}
\caption{Algorithm overview}
\label{algo:overview_balanced_clustering}
\end{algorithm}

\subsection* {The Algorithm}

\subsubsection*{Intuition}
	The approach is to guess the optimal threshold, construct the threshold graph
	at this threshold, write and round several LPs for each connected component
	of this graph for different values of $k$.
	The intuition behind why this works is that at the optimal threshold,
	each cluster is fully contained within a connected component
	(by definition of the threshold graph).

	We the round the opening variables, but this time, open exactly $k$ centers.
	Most of the work goes into rounding the openings, and showing that it is correct.
	Then, we simply round the assignments using a minimum cost flow again.

\subsubsection*{Linear Program}

As earlier, let $y_i$ be an indicator variable to denote whether vertex $i$ is a center, and $x_{ij}$
be indicators for whether $j$ belongs to the cluster centered at $i$.
By convention, $i$ is called a facility and $j$ is called a client.

Consider the following LP relaxation for the IP for each connected component of $G$.
Note that it is exactly the same as the one from the previous section,
except it is described in terms of the threshold graph $G$.
Let us call it {\tt LP-$k$-center}$(G)$:

\begin{subequations}
\label{LP}
\begin{align}
\sum_{i \in V} y_i &= k & \label{eq:lp_1} \\
x_{ij} &\leq y_i & \forall i,j \in V \label{eq:lp_2} \\
\sum_{j: ij \in E} x_{ij} & \leq nL y_i  &\forall i \in V \label{eq:lp_3} \\
\sum_{j: ij \in E} x_{ij} & \geq n\ell y_i  &\forall i \in V \label{eq:lp_3.5} \\
\sum_{i: ij \in E} x_{ij} &= p    &\forall j \in V \label{eq:lp_4} \\
x_{ij} &=0 &\forall ij \notin E \label{eq:lp_5}\\
0 \leq x, y & \leq 1  \label{eq:lp_6}
\end{align}
\end{subequations}

Once we have the threshold graph, for the purpose of $k$-center,
all distances can now be measured in terms of the length
of the shortest path in the threshold graph.
Let $d_G(i,j)$ represent the distance between $i$ and $j$
measured by the length of the shortest path between $i$ and $j$ in $G$.

\subsubsection*{Connected Components}
It is well known \citep{Cygan_lprounding} that even without lower bounds and replication,
the LP has unbounded integrality gap for general graphs. However, for connected components
of the threshold graph, this is not the case.

To begin with, we show that it suffices to be able to do the LP rounding procedure for only connected
threshold graphs, even in our generalization.

\begin {theorem} \label{thm:conn}
If there exists an algorithm that takes as input a connected graph $G$, capacities $\ell,L$,
replication $p$, and $k$ for which {\tt LP-$k$-center}$(G_t)$ is feasible,
and computes a set of $k$ centers to open and an assignment of
 every vertex $j$ to $p$ centers $i$ such that $d_G(i,j) \leq r$ satisfying the
 capacity constraints, then we can obtain a $r$-approximation algorithm to the balanced $k$-centers problem with $p$-replication.
\end {theorem}

\begin{proof}
	Let connected component $i$ have $k_i$ clusters.
	For each connected component, do a linear search on the range $[1,\dots,k]$
	to find values of $k_i$ for which the problem is feasible.
	These feasible values will form a range, if size constraints are to be satisfied.
	To see why this is the case, note that if $(x_1, y_1)$ and $(x_2, y_2)$ are
	fractional solutions for $k = k_1$ and $k = k_2$ respectively, then
	$((x_1+x_2)/2, (y_1 + y_2)/2)$ is a valid fractional soluion for
	$k = (k_1 + k_2)/2$.

	Suppose the feasible values of $k_i$ are $m_i \leq k_i \leq M_i$. If $\sum_{i} m_i > k$ or $\sum_i M_i < k$,
	return NO (at this threshold $t$). Otherwise, start with each $k_i$ equal to $m_i$.
	Increase them one by one up to $M_i$ until $\sum_i k_i = k$. This process takes polynomial time.
\end{proof}

From now on, the focus is entirely on { a single connected component}.

\subsubsection*{Rounding $y$}

Given an integer feasible point to the IP for each connected component, we can obtain the
desired clustering. Hence, we must find a way to obtain an integer feasible point from any feasible
point of {\tt LP-$k$-center}.

To round the $y$, we follow the approach of An et al.\citep{ckc3}. The basic idea is to create a coarse clustering
of vertices, and have the cluster centers form a tree. The radius of each cluster will be at most 2, and the
length of any edge in the tree will exactly be three, by construction.

Now, to round the $y$, we first start from the
leaves of the tree, moving opening around in each coarse cluster such that at most one node (which we pick to be the center, also called the monarch). In subsequent steps, this fractional opening is passed to the parent cluster, where the
same process happens. The key to getting a constant factor approximation is to ensure that fractional openings that
transferred from a child cluster to a parent cluster are not propagated further. Note that the bicriteria algorithm did
not move opening from one coarse cluster (empire) to another because we didn't have an upper bound
of the cost incurred by making this shift.

\paragraph{Preliminaries.} We start with some definitions.

\begin{definition}[$\delta$-feasible solution \citep{Cygan_lprounding}]
	A solution $(x, y)$ feasible on $G^{\delta}$, the
	graph obtained by connecting all nodes within $\delta$ hops away from each other.
\end{definition}

Next, we introduce the notion of a distance-$r$ shift.
Intuitively, a distance-$r$ shift is a series of movements of openings,
none of which traverses a distance more than $r$ in the threshold graph.
Note that the definition is similar to what is used in An et al.\citep{ckc3}.

\begin{definition}[Distance-$r$ shift ]
	Given a graph $G = (V, E)$ and $y, y' \in \mathbb{R}^{|V|}_{\geq 0}$,
	$y'$ is a distance-$r$ shift of $y$ if
	$y'$ can be obtained from $y$ via a series of disjoint  movements
	of the form {  ``Move $\delta$ from $i$ to $i'$'' where $\delta \leq \min(y_i, 1-y_{i'})$}
	and every $i$ and $i'$ are at most a distance $r$ apart in the threshold graph $G$.
	Further, if all $y'$ are zero or one, it is called an { integral distance-$r$ shift}.
\end{definition}

Note that, by the definition of a distance-$r$ shift,
 each unit of $y$ moves only once and if it moves more than once,
 all the movements are put together as a single, big movement,
 and this distance still does not exceed $r$.

 \begin{lemma}[Realizing distance-$r$ shift]
 	\label{lemma:transfer_validity}
 	For every distance-$r$ shift $y'$ of $y$ such that $0\leq y_i' \leq 1~ \forall i\in V$,
	we can find $x'$ in polynomial time such that
	$(x', y')$ is $(r+1)$-feasible.
 \end{lemma}
\begin{proof}

We can use the Move operation described earlier and in Cygan et al.\ \citep{Cygan_lprounding}
to change the corresponding $x$ for each such a movement to
ensure that the resulting $(x', y')$ are $(r+1)$-feasible.
The additional restriction $\delta \leq 1 - y_b$ ensures that $y\leq 1$.
Since each unit of $y$ moves only once, all the movements put together will also lead
a solution feasible in $G^{r+1}$, i.e. we get a $(r+1)$-feasible solution.

\end{proof}

From here on, we { assume that $x_{ij}, x_{i'j}$ are adjusted as described
above for every movement between $i$ and $i'$}.


The algorithm to round $y$ \citep{ckc3} proceeds in two phases. In the first phase,
we cluster points into a tree of coarse clusters (monarchs) such that nearby clusters are connected
using the monarch procedure of Khuller et al \citep{Khuller96thecapacitated}. In the second phase,
fractional opening are aggregated to get an integral distance-5 shift.

\paragraph{Monarch Procedure.}
The monarch procedure is presented a little differently but is very similar to
the monarch procedure presented earlier. Since the threshold graph is
connected, we can get guarantees on how big the distance between two
monarchs is.

Algorithm \ref{algo:monarchs} describes the first phase where we
construct a tree of monarchs and assign empires to each monarch.
Let $\mathcal{M}$ be the set of all monarchs. For some monarch, $u\in \mathcal{M}$,
let $\mathcal{E}_u$ denote its empire.
For each vertex $i$, let $m(i)$ denote the the monarch $u$ to whose empire $\mathcal{E}_u$,
$i$ belongs.

\begin{algorithm}
\DontPrintSemicolon
\KwIn{$G = (V, E)$ }
\KwOut{Tree of monarchs, $T=(\mathcal{M}, E')$, and empires for each monarch}
\nl $\texttt{Marked} \gets \emptyset$\;
\nl \ForEach {$j \in V$} {\nl initialize $\texttt{ChildMonarchs}(j)$ and $\texttt{Dependents}(j)$ to $\emptyset$}
\nl Pick any vertex $u$ and make it a monarch\;
\nl $\mathcal{E}_u \gets N^+(u)$; Initialize $T$ to be a singleton node $u$\;
\nl $\texttt{Marked} \gets \texttt{Marked} \cup \mathcal{E}_u$\;
\nl \While{$\exists w \in (V \setminus \texttt{\upshape Marked}) \text{ \upshape   such that } d_G(w, \texttt{\upshape Marked}) \geq 2$} {
\nl	Let $u \in (V \setminus \texttt{Marked})$ and $v \in \texttt{Marked}$ such that $d_G(u, v) = 2$\;
\nl	Make $u$ a monarch and assign its empire to be $\mathcal{E}_u \gets N^+(u)$\;
\nl	$\texttt{Marked} \gets \texttt{Marked} \cup \mathcal{E}_u$\;
\nl	Make $u$ a child of $m(v)$ in $T$\;
\nl  $\texttt{ChildMonarchs}(v) \gets \texttt{ChildMonarchs}(v) \cup \{u\}$\;
}
\nl \ForEach {$v \in (V \setminus \texttt{\upshape Marked})$} {
\nl	Let $u \in \texttt{Marked}$ be such that $d_G(u, v)$ = 1\;
\nl	$\texttt{Dependents}(u) \gets \texttt{Dependents}(u) \cup \{v\}$\;
\nl	$\mathcal{E}_{m(u)} \gets \mathcal{E}_{m(u)} \cup \{v\}$\;
}
\caption{Monarch Procedure: Algorithm to construct tree of monarchs and assign empires}
\label{algo:monarchs}
\end{algorithm}

The guarantees now translate to the following (Lemma \ref{lemma:monarchs}):
\begin{itemize}
\item Empires partition the point set.
\item The empire includes \emph{all} immediate neighbors of a monarch
	and additionally, some other nodes of distance two (called dependents).
\item Adjacent monarchs are exactly distance 3 from each other.
\end{itemize}

\begin{lemma}
\label{lemma:monarchs}
Algorithm \ref{algo:monarchs}, the monarch procedure is well-defined and its output satisfies the following:
\begin{itemize}
	\item $\mathcal{E}_u \cap \mathcal{E}_{u'} = \emptyset$.
	\item $\forall u \in \mathcal{M}: \text{ } \mathcal{E}_u = N^+(u) \cup (\bigcup_{j \in N^+(u)} \texttt{\upshape Dependents}(j) )$.
	\item The distance between a monarch and any node in its empire is at most 2.
	\item Distance between any two monarchs adjacent in $T$ is exactly 3.
	\item If $\texttt{\upshape ChildMonarchs}(j) \neq \emptyset$ or $\texttt{\upshape Dependents}(j) \neq \emptyset$,
			then $j$ is at distance one from some monarch.
\end{itemize}
\end{lemma}
\begin{proof}
	Note that the whole graph is connected and $V \neq \emptyset$.
	For the while loop, if there exists $w$ such that $d_G(w, \texttt{Marked}) \geq 2$,
	there exists $u$ such that $d_G(u, \texttt{Marked}) = 2$ because the graph is connected.
	By the end of the while loop, there are no vertices at a distance 2 or more from $\texttt{Marked}$.
	Hence, vertices not in $\texttt{Marked}$, if any, should be at a distance 1 from $\texttt{Marked}$.
	Thus, the algorithm is well defined.

	Each time a new monarch $u$ is created, $N^+(u)$ is added to its empire.
	This shows the first statement.
	The only other vertices added to any empire are the dependents in the foreach loop.
	Each dependent $j$ is directly connected to $i$, a marked vertex. Hence, $i$ has to
	be a neighbor of a monarch. If $i$ were a monarch, $j$ would have been marked in the
	while loop. Thus, $d_G(j, m(i))=2$.

	If the first statement of the while loop, $v$ is a marked vertex,
	and has to be a neighbor of some monarch $m(v)$.
	New monarch $u$ is chosen such that $d_G(u, v) = 2$.
	The parent monarch of $u$ is $m(v)$ and $d_G(u, m(v)) = d_G(u, v)+d_G(v, m(v))=3$.

\end{proof}

\paragraph{Initial Aggregation.}
Now, we shall turn to the rounding algorithm of An et al \citep{ckc3}.
The algorithm begins with changing $y_u$ of every monarch $u\in \mathcal{M}$ to 1.
Call this the { initial aggregation}. It requires transfer of at most distance
one because the neighbors of the monarchs has enough opening.

\begin{lemma}
	\label{lemma:initial_aggregation}
	The initial aggregation can be implemented by a distance-1 shift.
\end{lemma}
\begin{proof}
	For every vertex $u \in V$, we have
	$\sum_{j \in N(u)} y_j \geq \sum_{j \in N(u)} x_{uj} = p \geq 1$.
	Hence, there is enough $y$-mass within a distance of one from $u$.
	The actual transfer can happen by letting $\delta = \min(1-y_u, y_j)$
	for some neighbor $j$ of $u$ and then transferring $\delta$ from
	$j$ to $u$. That is, $y_j = y_j - \delta$ and $y_u = y_u + \delta$.

\end{proof}

\paragraph{Rounding.}
The rounding procedure now proceeds in a bottom-up manner on the tree
of monarchs, rounding all $y$ using movements of distance 5 or smaller.
After rounding the leaf empires, all fractional opening, if any is at the monarch.
For internal empires, the centers of child monarch (remnants of previous rounding
steps) and dependents are first rounded. Then the neighbors of the
monarch are rounded to leave the entire cluster integral except the monarch.
The two step procedure is adopted so that the opening propagated from this monarch
to its parent originates entirely from the 1-neighborhood of the monarch.

Formally, at the end of each run of round on $u \in \mathcal{M}$,
all the vertices of the set $I_u$ are integral,
where
$I_u := (\mathcal{E}_u \setminus u) \cup  (\bigcup_{j \in N(u)} \texttt{\upshape ChildMonarchs}(j) )$.

\begin{algorithm}
\DontPrintSemicolon 
\KwIn{Tree of monarchs, $T$, and empires for each monarch after the initial aggregation }
\KwOut{$y'$, an integral distance-5 shift of $y$}
\SetKwFunction{proc}{Round}\SetKwFunction{proclocal}{LocalRound}
\SetKwProg{myproc}{Procedure}{}{} 
\nl \myproc{\proc{\upshape Monarch $u$}}{
\nl //Recursive call \;
\nl	\lForEach{\upshape child $w$ of $u$ in $T$ } {{\tt Round}$(w)$}
\nl //Phase 1 \;
\nl	\ForEach{$j \in N(u)$} {
\nl		$X_j \gets \{j\} \cup \texttt{ChildMonarchs}(j) \cup \texttt{Dependents}(j)$ \;
\nl		$W_j \gets \{ \lfloor y(X_j) \rfloor$ nodes from $X_j \}$; (Avoid picking $j$ if possible)\;
\nl		\proclocal{$W_j, X_j, \emptyset$}\;					\label{algo_line:xu1}
\nl		\proclocal{$\{j\}, X_j \setminus W_j, \emptyset$}\; 		\label{algo_line:xu2}
	}
\nl //Phase 2 \;
\nl	$F = \{ j | j \in N(u) \text{ and } 0 < y_j < 1 \}$ \;
\nl	$W_F \gets \{$ any $\lfloor y(F) \rfloor$ nodes from $F \}$\;
\nl	\proclocal{$W_F, F , \emptyset$}\;			\label{algo_line:f1}
\nl //Residual \;
\nl	\If {$y(F \setminus W_F) > 0$}{
\nl		Choose $w^* \in F \setminus W_F$\;
\nl		\proclocal{$\{w^*\}, F \setminus W_F, u$}\; 	\label{algo_line:wstar}
	}
}

\nl \myproc{\proclocal{$V_1, V_2, V_3$}}{
\nl	\While {$\exists i \in V_1 \text{ \upshape such that }  y_i < 1$} {
\nl		Choose a vertex $w$ with non-zero opening from $V_2 \setminus V_1$ \;
\nl		if there exists none, choose $j$ from $V_3 \setminus V_1$ \;
\nl		$\delta \gets \min(1-y_i, y_j)$\;
\nl 		{\tt Move} $\delta$ from $j$ to $i$ \;
	}
}
\caption{Algorithm to round $y$}
\label{algo:round}
\end{algorithm}

The rounding procedure is described in detail in Algorithm \ref{algo:round}.
The following lemma states and proves that algorithm \ref{algo:round}
 rounds all points and doesn't move opening very far.

\begin{lemma}[Adaptation of Lemma 19 of An et al \citep{ckc3}]
	Let $I_u := (\mathcal{E}_u \setminus u) \cup  (\bigcup_{j \in N(u)} \texttt{\upshape ChildMonarchs}(j) )$.
	\begin{itemize}
	\item	$\texttt{\upshape Round}(u)$ makes the vertices of $I_u$ integral
	with a set of opening movements within $I_u \cup \{u\}$.
	\item This happens with no incoming movements to the monarch $u$ after the initial aggregation.
	\item The maximum distance of these
	movements is five, taking the initial aggregation into account.
	\end{itemize}

\end{lemma}

\begin{proof}
	{\bf Integrality.}
	From lemma \ref{lemma:monarchs}, it can be seen that $X_j, j \in N(u)$ above form a partition
	of $I_u$. Hence, it suffices to verify that each node of every $X_j$ is integral.

	At the end of line \ref{algo_line:xu1}, the total non-integral opening in $X_j$
	is $y(X_j) - \lfloor y(X_j) \rfloor$, and is hence smaller than one. Line \ref{algo_line:xu2}
	moves all these fractional openings to $j$. By now, all openings of $X_j \setminus \{j \}$ are integral.

	Now, $F$ is the set of all non-integral $j \in N(u)$. So, by the end of line \ref{algo_line:f1},
	the total non-integral opening in $N(u)$ (and hence in all of $I_u$) is
	$y(F\setminus W_F) = y(F) - \lfloor y(F) \rfloor$,
	and is again smaller than one. If this is zero, we are done.

	Otherwise, we choose a node $w^*$, shift this amount to $w^*$ in line \ref{algo_line:wstar}.
	To make this integral, this operations also transfers the \emph{remaining amount, i.e. $1-y(F\setminus W_F)$
	from the monarch $u$}. If this happens, the monarch $u$'s opening is no longer integral, but $I_u$'s is.

	This shows the first bullet.
	For the second one, notice that after the initial aggregation,
	this last operation is the only one involving the monarch $u$ and
	hence, there are no other incoming movements into $u$.

	{\bf Distance.}
	In the first set of transfers in line \ref{algo_line:xu1}
	the distance of the transfer is at most 4. This is because
	dependents are a distance one away from $j$ and child monarchs
	are at a distance two away. The maximum distance is when the transfer happens
	from one child monarch to another, and this distance is 4
	(recall that  there are no incoming movements into monarchs).

	The transfers in line \ref{algo_line:xu2} moves openings from a child monarch
	or a dependent to $j$. The distances are 2 and 1 respectively.
	Accounting for the initial
	aggregation, this is at most 3.

	The rounding on line \ref{algo_line:f1} moves openings between neighbors of the monarch,
	i.e. from some $j$ to $j'$ where
	$j, j' \in N(u)$. So, the distance between $j$ and $j'$ is at most 2. From the
	preceding transfers, the openings at $j$ moved a distance of
	at most three to get there, and thus, we conclude
	that openings have moved at most a distance of 5 so far.

	The first step of rounding on line \ref{algo_line:wstar} moves openings from some $j$ to $w^*$, where
	$j, w^* \in N(u)$. As above, the maximum distance in this case is 5. The second step
	of rounding on  line \ref{algo_line:wstar} moves opening from the
	monarch $u$ to its neighbor $w^*$. This distance
	is one, and after accounting for the initial aggregation, is 2.

	From this, we see that the maximum distance any opening has to move is 5.
\end{proof}

The algorithms, their properties in conjunction with lemma \ref{lemma:transfer_validity}
leads to the following theorem, which also summarizes this subsection.
\begin{theorem}
	\label{thm:yround}
	There exists a polynomial time algorithm to find a 6-feasible solution with
	all $y$ integral.
\end{theorem}

\subsubsection*{Rounding $x$}

Once we have integral $y$, rounding the $x$ is fairly straight-forward,
without making the approximation factor any worse.
Exactly the same procedure used in bicriteria algorithms works here too.
But, we can have an easier construction since for $k$-center since
we can use distances in the threshold graph instead.

\begin{theorem} \label{thm:xround}
There exists a polynomial time algorithm that given a $\delta$-feasible solution
$(x, y)$ with all $y$ integral, finds a $\delta$-feasible solution $(x', y)$
with all $x'$ integral.
\end{theorem}

\begin{proof}
  \begin{figure*}
    \centering
    \includegraphics[width=0.7\textwidth]{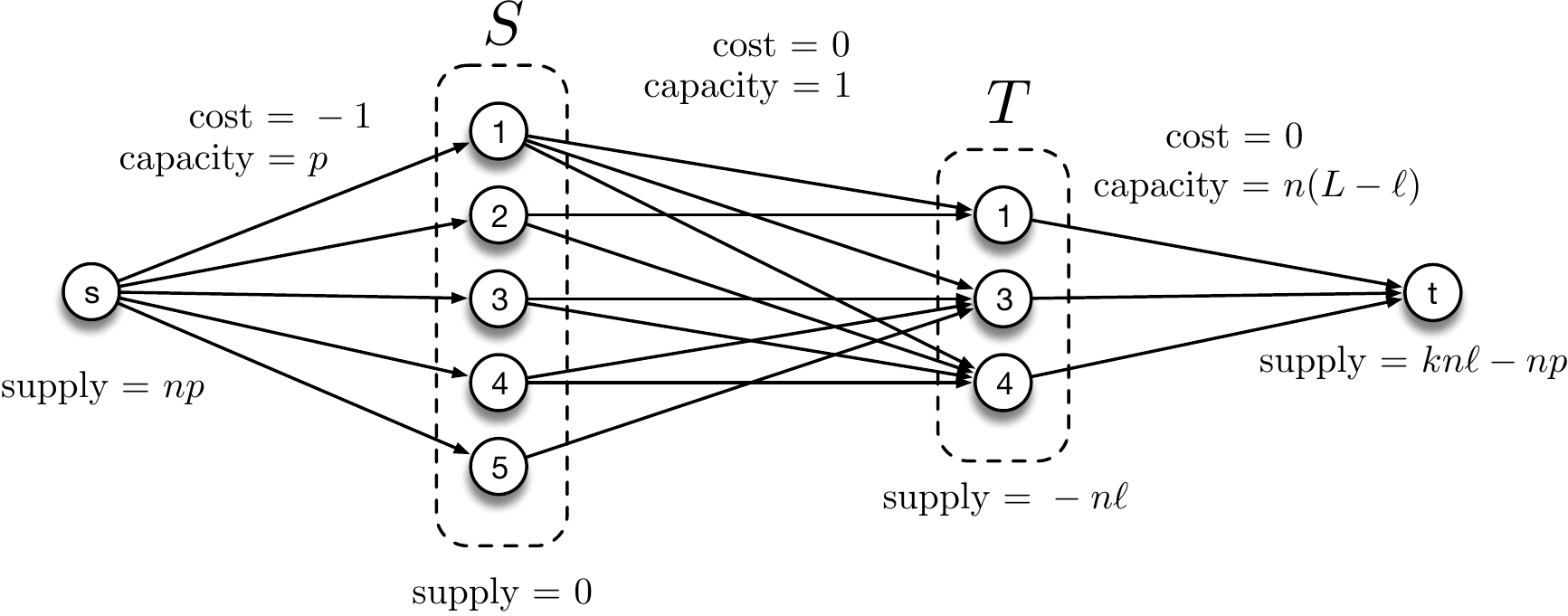}
    \caption{Minimum cost flow network to round $x$'s. Each node in a group has
      the same supply, which is indicated below. The cost and capacity of each
      edge is indicated above.}
    \label{fig_xround}
  \end{figure*}
  We shall use a minimum cost flow network to this.  Consider a
  directed bipartite graph $(S, T, E')$, where $S = V$ and
  $T = \{i:y_i=1\}$ and $j \rightarrow i \in E'$ iff $x_{ij} > 0$.
  Add a dummy vertex $s$, with edges to every vertex in $S$, and $t$
  with edges from every vertex in $T$.  In this network, let every
  edge of the bipartite graph have capacity 1.  Further, all the
  $s\rightarrow S$ edges have capacity $p$.  $s$ supplies a flow of
  $np$ units, while each $u \in T$ has a demand of $l$ units. To
  ensure no excess demand or supply, $t$ has a demand of $np-kl$. All
  the $t \rightarrow T$ edges have a capacity of $(L-\ell)$.

	 All the $s\rightarrow S$ edges have a cost of $-1$ and every other
	 edge has a cost of zero.
	 See figure \ref{fig_xround}.

	 Clearly, a feasible assignment $(x, y)$ to {\tt LP-$k$-center}$(G^\delta)$ with integral $y$ is a
	 feasible flow in this network. In fact, it is a minimum-cost flow in this network.
	 This can be verified by the absence of negative cost cycles in the residual graph (because
	 all negative cost edges are at full capacities).

	 Since, the edge capacities are all integers, there exists a minimum cost integral flow by
	  the Integral Flow Theorem. This flow can be used to
	 fix the cluster assignments.

\end{proof}

Piecing together theorems \ref{thm:yround} and \ref{thm:xround}, we have the following theorem:

\begin{theorem}

Given an instance of the $k$-centers problem with $p$-replication and  for a connected graph $G$,
and a fractional feasible solution to {\tt LP-$k$-center}$(G)$, there exists a polynomial time algorithm to
obtain a 6-feasible integral solution. That is,
for every $i,j$ such that $x_{ij} \neq 0$, we have $d_G(i, j) \leq 6$.

\end{theorem}

\section{Proofs from Section \ref{sec:stability}} \label{app:stability}
\paragraph{$k$-means++}

We provide the full details for Theorem \ref{thm:k-means++}.

\medskip
\noindent \textbf{Lemma~\ref{lem:k=2} (restated).}
\emph{
Assume $k=2$.
For sufficiently small $\epsilon$, 
$Pr[(\hat c_1\in X_1~\&\&~ \hat c_2\in X_2)~ ||~ (\hat c_2\in X_1~ \&\&~ \hat c_1\in X_2)]=
1-O(\rho)$. 
}

\begin{proof}
Wlog, we set $d(c_1,c_2)^2=1$ for ease of computation (scaling all distances does not affect the optimal clustering).
Let $A=\sum_{x\in X_1,y\in X_2} d(x,y)^2$ and $B=\sum_{x,y\in X} d(x,y)^2$. Then the probability is $A/B$.
Let $c_1'$ and $c_2'$ denote the centers of mass of $X_1$ and $X_2$, respectively. 
By Lemma \ref{lem:mass}, $d(c_1,c_1')\leq\frac{\Delta_1(C_i)}{|C_i|}\cdot\frac{10\epsilon/\alpha|C_i|}{(1-10\epsilon/\alpha)|C_i|}\leq
r_i^2 O(\frac{\rho}{1-\rho})\leq \frac{L}{\ell}\cdot O(\frac{\epsilon}{\alpha})(1-O(\frac{\epsilon}{\alpha}))^{-2}d(c_1,c_2)$,
so $d(c_1',c_2')\geq d(c_1,c_2)-d(c_1,c_1')-d(c_2,c_2')\geq 1-O(\frac{L}{\ell}\cdot\frac{\epsilon}{\alpha})$.

Therefore,
$A=\sum_{x\in X_1,y\in X_2} d(x,y)=|X_1|\Delta_1(X_2)+|X_2|\Delta_1(X_1)+|X_1||X_2|d(c_1',c_2')^2
\geq |X_1||X_2|d(c_1',c_2')\geq n_1 n_2 (1-O(\frac{\epsilon}{\alpha}))^2 (1-O(\frac{L}{\ell}\cdot\frac{\epsilon}{\alpha}))^2\geq n_1 n_2
(1-O(\frac{L}{\ell}\cdot\frac{\epsilon}{\alpha}))$.

$B=\sum_{x,y\in X} d(x,y)^2=n^2 w_{avg}^2+n_1 n_2\leq n_1 n_2 (\frac{5}{4}\cdot\frac{\epsilon}{\alpha}\cdot\frac{n^2}{n_1 n_2}+1)\leq
n_1 n_2 (1+\rho \frac{L^2}{\ell^2})$.
Therefore, $\frac{A}{B}\geq \frac{1-O(\frac{L}{\ell}\rho)}{1+O(\frac{L}{\ell}\cdot\rho)}$.
\end{proof}


Now we prove Theorem \ref{thm:k-means++}.
We assume $k\leq\frac{\epsilon}{\alpha}$, which is reasonable for
real-world data which does not have too many clusters and is sufficiently stable. 
We need this assumption so that we can set $\rho>\frac{\epsilon\cdot k}{\alpha}$.
We still assume that $L\in O(\ell)$.
Assume that we sample $k$ points $\hat c_1,\dots\hat c_k$. 
We start with a lemma similar to Lemma \ref{lem:k=2}.

\begin{lemma}
The probability that $\hat c_1$ and $\hat c_2$ lie in the cores of different clusters is $1-O(\rho)$.
\end{lemma}

\begin{proof}
\begin{equation*}
A=\sum_{x\in X_i,~y\in X_j} d(x,y)^2=|X_i|\Delta_1(X_2)+|X_2|\Delta_1(X_1)+|X_1||X_2|d(c_i',c_j')^2\geq n_1 n_2 (1-O(\rho))^2.
\end{equation*}
\begin{align*}
B&=\sum_{x\in C_i,y\in C_j} d(x,y)^2=(n_i+n_j)\Delta_1(C_i\cup C_j)=
(n_i+n_j)^2 \left(\frac{w_i^2 n_i+w_j^2 n_j}{n_i+n_j}\right)+n_i n_j d(c_i,c_j)^2\\
&=n_i n_j \left(\frac{(n_i+n_j)^2}{n_i n_j}\left(\frac{w_i^2 n_i+w_j^2 n_j}{n_i+n_j}\right)+d(c_i,c_j)^2\right).
\end{align*}
Summing over all $i,j$, we have the following.
$$A/B=\frac{\sum_{i,j}n_i n_j}{\sum_{i,j} n_i n_j \left(1+\frac{(w_i n_i +w_j n_j)(n_i+n_j)}{n_i n_j}\right)}.$$

When $L\in O(\ell)$, we simplify the denominator:
\begin{align*}
\sum_{i,j} n_i n_j \left(1+\frac{(w_i n_i +w_j n_j)(n_i+n_j)}{n_i n_j}\right)
&=\sum_{i,j} n_i n_j +\sum_{i,j} (n_i+n_j)(w_i n_i +w_j n_j)\\
&=\sum_{i,j} n_i n_j +O(\frac{L}{\ell})\cdot\frac{n}{k}\sum_{i,j}(w_i n_i +w_j n_j)\\
&=\sum_{i,j} n_i n_j +O(\frac{L}{\ell})\cdot\frac{n}{k}\sum_{i,j}2n\frac{\epsilon}{\alpha}\\
&=O(1+\rho)\sum_{i,j} n_i n_j
\end{align*}
when $\rho>\frac{\epsilon}{\alpha}$.

Therefore, $A/B=1-O(\rho)$.
\end{proof}

\begin{lemma}
\begin{align*}
Pr\left[\hat c_{i+1}\in \bigcup_{j\notin 1,\dots, i} X_j\mid \hat c_1,\dots,\hat c_i
 \text{ lie in the cores of }
 X_1,\dots X_i\right] 
=1-O(\rho).
\end{align*}
\end{lemma}

\begin{proof}
$A=\sum_{j=m+1}^k\sum_{x\in X_j} d(x,\hat C)^2$, and $B=\sum_{j=1}^k\sum_{x\in C_j} d(x,\hat C)$.
Let $$\phi=\max_{j\geq m+1}[\max_{x\in X_j}d(x,C_j)/d(c_j,\hat C)],$$
From Lemma \ref{lem:r_i}, we have that
$$\phi\leq \max_{i,j}\frac{\O(\sqrt{\frac{\epsilon}{\alpha}})d(c_i,c_j)}{d(c_i,c_j)-O(\sqrt{\frac{\epsilon}{\alpha}})d(c_i,c_j)}
\leq \frac{O(\sqrt{\frac{\epsilon}{\alpha}})}{1-O(\sqrt{\frac{\epsilon}{\alpha}})}\leq 1$$.
Then for all points in the core of a cluster, $d(x,\hat C)\geq (1-\phi) d(c_j,\hat C)$.
Then
$$A\geq \sum_{j=m+1}^k (1-O(\rho))n_j(1-\phi)^2 d(c_j,\hat C)\geq (1-\rho-2\phi)
\sum_{j=m+1}^k n_j d(c_j,\hat C)^2.$$
\begin{align*}
B&\leq \sum_{j=1}^k (\Delta_1(C_j)+n_j d(c_j,\hat c_{p_j})^2)\leq \Delta_k(V)+\sum_{j=1}^m n_j d(c_j,x_j)^2+
\sum_{j=m+1}^k n_j d(c_j,\hat C)^2\\
&\leq\Delta_k(V)+\frac{1}{\rho}\sum_{j=1}^m\Delta_1(C_j)+\sum_{j=m+1}^k n_j d(c_j,\hat C)
\leq \frac{1}{\rho}\Delta_k(V)+\sum_{j=m+1}^k n_j d(c_j,\hat C).
\end{align*}
If we set $\rho=\Omega(\frac{L}{\ell}\cdot{\epsilon k}{\alpha})$, then $A/B=1-O(\rho)$.
\end{proof}

\begin{lemma} \label{lem:unconditional}
Given we have sampled points $\{\hat x_1,\dots,\hat x_i\}$, and let $C_1,\dots C_m$ be all the clusters whose outer core
contains some point $\hat x_j$. Then $\text{Pr}[\hat x_{i+1}\in\cup_{j=m+1}^k X_j]\geq 1-5\rho$.
\end{lemma}

\begin{proof}
This follows from the previous lemma.
\end{proof}

Now we analyze the case where we sample more than $k$ points.
Let $N=\frac{2k}{1-5\rho}$.

\begin{lemma}
Say we sample $N$ points. The probability that for all $i\in[k]$, there is some sampled point in $X_i'$, is $\geq 1-\rho$.
\end{lemma}

\begin{proof}
The proof follows from Lemma \ref{lem:unconditional}.
\end{proof}

Finally, we perform a greedy removal phase. We refer the reader to \citep{ostrovsky2006effectiveness}
since the analysis is the same. This finishes the proof of Theorem \ref{thm:k-means++}.


\paragraph{$k$-median}
Given a clustering instance $(S,d)$ which satisfies $(1+\alpha,\epsilon)$-approximation
stability with respect to balanced $k$-median, for some $(l,L,k)$.
We denote $\mathcal{C}=\{C_1,\dots,C_k\}$ as the optimal partition and $c_1,\dots,c_k$ as the optimal centers.
Denote $w_{avg}$ as the average distance from a point
to its center in the optimal solution.
Given a point $x$, define $w(x)$ as the distance to its center, and in general, for all $i$,
denote $w_i(x)$ as the distance to the center $c_i$ of cluster $C_i$ in the optimal solution.
Note, we will discuss the $p$-replication variant at the end of this section.

\begin{lemma} \label{lem:as-struct}
Assume the size of the optimal clusters are $>4\epsilon n(1+\frac{6}{\alpha})$. Then,
\begin{itemize}
\item For $<\frac{\epsilon n}{2}$ points $x$, there exists a set of 
$\geq \frac{\epsilon n}{2}$ points $y$
from different clusters such that $d(x,y)\leq \frac{\alpha w_{avg}^2}{\epsilon}$.
\item $<\frac{6\epsilon n}{\alpha}$ points $x$ have $w(x)>\frac{\alpha w_{avg}^2}
{6\epsilon}$.
\end{itemize}
\end{lemma}

\begin{proof}

\begin{itemize}
\item Assume the statement is false. Then there exist $\frac{\epsilon n}{2}$ pairs of
distinct points $(x,y)$ such that $d(x,y)\leq \frac{\alpha w_{avg}^2}{\epsilon}$
(for example, because of Hall's theorem).
Now we create a new partition $\mathcal{C}'$
by swapping the points in each pair, i.e., for a pair $x\in C_i$ and $y\in C_j$, put $x$
into $C_j$ and $y$ into $C_i$. Then for a single swap, the increase in cost is
$w_j(x)-w_i(x)+w_i(y)-w_j(y)\leq w_j(y)+\frac{\alpha w_{avg}^2}{\epsilon}-w_i(x)
+w_i(x)+\frac{\alpha w_{avg}^2}{\epsilon}-w_j(y)=\frac{2\alpha w_{avg}^2}{\epsilon}$.
Therefore, the total difference in cost from $\mathcal{OPT}_{\Phi}$
to $\Phi(\mathcal{C'})$ is then $\alpha\mathcal{OPT}_{\Phi}$.
Note that since we only performed swaps, $\mathcal{C}'$ still satisfies all the capacity constraints
of $\Phi$. Furthermore, $\mathcal(C)'$ is $\epsilon$-close to $\mathcal{OPT}_{\Phi}$ 
(since all clusters are size $>2\epsilon n$,
the optimal bijection between $\mathcal(C)'$ and $\mathcal{OPT}_{\Phi}$ is the identity).
This contradicts the approximation stability property.
\item This follows simply from Markov's inequality.
\end{itemize}
\end{proof}

Define $x\in C_i$ as a good point if there are less than $\frac{\epsilon}{2}$ points 
$y\notin C_i$ such that $d(x,y)\leq\frac{\alpha w_{avg}^2}{\epsilon}$, and 
$w(x)<\frac{\alpha w_{avg}^2}{6\epsilon}$.
If a point is not good, then call it bad. Denote the set of bad points by $B$.
For all $i$, denote $X_i=\{x\in C_i\mid x\text{ is good}\}$.
Then Lemma \ref{lem:as-struct} implies that 
if the optimal clusters are size $>2\epsilon n$, $|B|<\epsilon n(1+\frac{6}{\alpha})$.

Now we show that the optimal centers are sufficiently far apart.

\begin{lemma} \label{lem:far-centers}
Assume for all $i$, $|C_i|>\frac{7\epsilon n}{\alpha}$.
Then for all $i,j$, $d(c_i,c_j)>\frac{2}{3}\cdot\frac{\alpha w_{avg}^2}{\epsilon}$.
\end{lemma}

\begin{proof}
Given a cluster $C_i$. Since there are $<\frac{\epsilon n}{2}$ and $<\frac{6\epsilon n}{\alpha}$ 
points which do not satisfy properties 1 and 2 from Lemma \ref{lem:as-struct}, then at least one point
from each cluster satisfies both properties, i.e., is a good point.
Then given $i,j$, let $x_i\in C_i$ and $x_j\in C_j$ be good points.
Then $d(c_i,c_j)>d(x_i,x_j)-d(x_i,c_i)-d(x_j,c_j)\geq \frac{\alpha w_{avg}^2}{\epsilon}-
2\cdot\frac{\alpha w_{avg}^2}{6\epsilon}=\frac{2}{3}\cdot\frac{\alpha w_{avg}^2}{\epsilon}$.
\end{proof}

Combining Lemmas \ref{lem:as-struct} an \ref{lem:far-centers} implies that the threshold graph with distance
$\frac{2}{3}\cdot\frac{\alpha w_{avg}^2}{\epsilon}$ will contain mostly ``good'' edges between good points
from the same cluster. The rest of the argument for showing correctness of the algorithm is similar to
the analysis in \citep{as}. We include it here for completeness.

The following lemma is similar to Lemma 3.5 from \citep{as}. 
We need to assume the clusters are larger, but
our proof generalizes to \emph{capacitated} $k$-median.

\begin{lemma} \label{lem:threshold}
Assume the optimal cluster sizes are $\geq \frac{\epsilon n}{2}(1+\frac{3}{\alpha})$.
For $\tau=\frac{\alpha w_{avg}^2}{3\epsilon}$, the threshold graph $G_{\tau}$ has the following properties:
\begin{itemize}
\item There is an edge between any two good points $x,y$ in the same cluster $C_i$,
\item There is not an edge between any two good points $x,y$ in different clusters; furthermore, these points do not even share a neighbor in $G_{\tau}$.
\end{itemize}
\end{lemma}

\begin{proof}
\begin{itemize}
\item Since $x$ and $y$ are both good, 
$d(x,y)\leq w(x)+w(y)\leq \frac{\alpha w_{avg}^2}{6\epsilon}+
\frac{\alpha w_{avg}^2}{6\epsilon}\leq\frac{\alpha w_{avg}^2}{3\epsilon}$ by the triangle inequality.
\item Assume $x$ and $y$ have a common neighbor $z$. Consider a point $y_2$ from $y$'s
cluster such that $w(y_2)\leq \frac{\alpha w_{avg}^2}{6\epsilon}$.
By assumption, there are at least $\frac{\epsilon n}{2}$ such points.
Furthermore,
$d(x,y_2)\leq d(x,z)+d(z,y)+d(y,y_2)\leq 2\tau+w(y)+w(y_2)\leq 
\frac{2\alpha w_{avg}^2}{3\epsilon}+\frac{\alpha w_{avg}^2}{6\epsilon}+
\frac{\alpha w_{avg}^2}{6\epsilon}=\frac{\alpha w_{avg}^2}{\epsilon}$.
Since $x$ is close to at least $\frac{\epsilon n}{2}$ points from different clusters,
$x$ cannot be a good point, so we have reached a contradiction.
\end{itemize}
\end{proof}

Then the threshold graph is as follows. Each $X_i$ forms a clique, the neighborhood of $X_i$
is exactly $X_i\cup B$, and for all $i\neq j$, $N(X_i)\cup N(X_j)=\emptyset$.
This facilitates an algorithm for clustering in this setting, following analysis that is similar to
\citep{as}.

\begin{lemma}[\citep{as}] \label{lem:b+2}
There is an efficient algorithm such that, given a graph $G$ satisfying the properties
of Lemma \ref{lem:threshold}, and given $b\geq |B|$ such that each $|X_i|\geq b+2$,
outputs a $k$-clustering with each $X_i$ in a distinct cluster.
\end{lemma}

The proof of this theorem, given in \citep{as}, depends solely on the properties of
Lemma \ref{lem:threshold}.

\begin{lemma} \label{lem:capviolate}
There is an efficient algorithm such that if a clustering instance satisfies
$(1+\alpha,\epsilon)$-approximation stability
for the balanced $k$-median objective and all clusters are size
$\geq 3\epsilon n(1+\frac{3}{\alpha})$, then the algorithm will produce a clustering
that is $O(\frac{\epsilon}{\alpha})$-close to the optimal clustering.
\end{lemma}

\begin{proof}
First we use the proof of Theorem 3.7 from \citep{as} (which assumes $w_{avg}^2$ is known) 
to the point where they achieve error
$O(\frac{\epsilon}{\alpha})$ by using their version of Lemma \ref{lem:b+2}.
They go on to lower the error to $\epsilon$, but this part of the proof breaks
down for us. 
For the case of unknown $w_{avg}^2$, we can use the same fix as in the proof of
Theorem 3.8 from \citep{as}.
\end{proof}

However, there is a problem with Lemma \ref{lem:capviolate}: even though it returns
a clustering with small error, the capacity constraints might
be violated. We can fix the capacity violations if we double our error.

\begin{lemma}
Given an $\epsilon'>0$ and a clustering $\mathcal{C'}$, if $\mathcal{C'}$ is $\epsilon'$-close to the optimal clustering, then
it is possible in polynomial time to construct a \emph{valid} clustering $\mathcal{C''}$ which is $2\epsilon'$-close to the optimal clustering.
\end{lemma}

\begin{proof}
For each cluster $C_i'$ in $\mathcal{C}$, let $v_i$ be the number of points for which
$C_i'$ violates the capacity, i.e., $|C_i'|-L$ or $l-|C_i'|$ or 0. Clearly,
$\sum_i v_i\leq \epsilon' n$, or else $\mathcal{C'}$ would not be $\epsilon'$-close
to the optimal. Then we can iteratively take a point from the largest cluster, and
place it into the smallest cluster. In each iteration, $\sum_i v_i$ is reduced by at
least 1. So in $\leq \epsilon' n$ rounds, we reach a valid clustering, and the error
is only $\epsilon'$ more than $\mathcal{C'}$.
\end{proof}

Theorem \ref{thm:as} follows immediately from the previous lemma.

Note that all of the proofs in this section can be trivially extended to the case where there is $p$-replication.
However, $p$-replication does not mesh well with stability assumptions.
Although Theorem \ref{thm:as} works completely under $p$-replication, all but an $\epsilon$-fraction of the
data appears to have ``trivial'' replication, in which there are $\frac{k}{p}$ groups of data, each of which
have nearly the same $p$ centers. This makes the problem similar to a $\frac{k}{p}$-clustering problem, up to
$\epsilon n$ points.
The reason for this phenomenon is the following. If good points $x_i$ and $x_j$ share center $c$ in addition to
having other centers $c_i$ and $c_j$, then by the triangle inequality, $c_i$ and $c_j$ are close together. This
would contradict Lemma \ref{lem:far-centers} unless each pair of good points either have $p$ centers in common, 
or no centers in common.

\paragraph{Examples of balanced approximation stable instances}

In this section, we explicitly construct clustering instances which satisfy approximation stability for balanced
clustering, but not for approximation stability for unbalanced clustering.

Given $n,\alpha,\epsilon,\ell,L$, and let $k=2$. Denote the two optimal centers as $c_1$ and $c_2$. Let $d(c_1,c_2)=1.9$.
We place $x\leq \frac{\epsilon n}{\alpha}$ points at distance 1 from $c_1$, and $.9$ from $c_2$. Call this set of points $A$.
Then we place $\ell n-x$ points at distance 0 from $c_1$ (denote by $B_1$), and we place the rest of the points at distance 0 from $c_2$ (denote by $B_2$).
We need to assume that $\ell n-x\geq 0$.

Then for balanced clustering, $C_1=A\cup B_1$, and $C_2=B_2$, because $C_1$ must contain at least $\ell n$ points.
The optimal cost is $x$.
For standard clustering, $C_1=B_1$, and $C_2=A\cup B_2$, and the optimal cost is $.9x$.
This clustering instance is not $(\frac{10}{9},\frac{x}{n})$approximation stable for standard clustering: all points in $A$ can move from $C_2$ to $C_1$, incurring a cost of $.1x$ to the objective.

However, this clustering instance is $(1+\alpha,\epsilon)$-approximation stable for the balanced clustering objective.
Moving any point to a different cluster incurs a cost of at least 1.
Given a partition with cost $(1+\alpha)x=x+\alpha\cdot x\leq x+\alpha\frac{\epsilon n}{\alpha}\leq x+\epsilon n$.
Then less than $\epsilon n$ points have switched clusters.

\paragraph{$k$-center}

Now we will prove Theorem \ref{thm:kcenter-lb}.
Given a clustering instance, denote its optimal balanced $k$-center radius by $r^*$.

\begin{lemma} \label{lem:kcenter-as}
Given a clustering instances satisfying $(2,0)$-approximation stability for balanced $k$-center,
then for all $p\in C_i$, $q\in C_j$, $i\neq j$, $d(p,q)>r^*$.
\end{lemma}

\begin{proof}
Assume the claim is false. Then $\exists p\in C_i,~q\in C_j$ such that $d(p,q)\leq r^*$ and $i\neq j$.
Then consider the optimal clustering, except switch $p$ for $q$. 
So the clusters are $C_i\cup\{q\}\setminus\{p\}$, $C_j\cup\{p\}\setminus\{q\}$, and the other $k-2$ clusters are the same as in $\mathcal{OPT}$.
This clustering achieves a balanced $k$-center radius of $2r^*$. The only points with a new distance to their
center are $p$ and $q$. By the triangle inequality, $d(c_i,q)\leq d(c_i,p)+d(p,q)\leq 2r^*$ and
$d(c_j,p)\leq d(c_j,q)+d(q,p)\leq 2r^*$. Furthermore, since the updated clusters are still the same size as
$C_i$ and $C_j$, all the balance constraints still hold.
But this gives us a contradiction under $(2,0)$-approximation stability, since we have found a valid clustering
of cost $2r^*$ which is different from $\mathcal{OPT}$.
\end{proof}

From this lemma, there is a simple algorithm to optimally cluster instances satisfying $(2,0)$-approximation stability for balanced $k$-center. We merely need to create the threshold graph for threshold distance $r^*$ and
output the connected components. Since every point is distance $\leq r^*$ to their center by definition, all
points in an optimal cluster will be in the same connected component. By Lemma \ref{lem:kcenter-as}, no two
points from different clusters will appear in the same connected component.

Now we prove the second part of Theorem \ref{thm:kcenter-lb}.
Note that there does not exist an efficient algorithm for $(2-\epsilon)$-approximation stability,
if the algorithm takes in $\ell$ and $L$ as parameters. This is a corollary of Theorem 10 in \citep{kcenter-pr},
by setting $\ell=0$ and $L=1$.

However, we can show something stronger: we can show no algorithm exists, even for the special case when
$\ell=L$. We show the analysis from \citep{kcenter-pr} carries over in this case.

Given a hard balanced-3-Dimensional Matching instance $X_1,X_2,X_3,T$, where $|X_1|=|X_2|=|X_3|=m$,
and each $t\in T$ is a triple $t=(x_1,x_2,x_3)$, $x_1\in X_1$, $x_2\in X_2$, $x_3\in X_3$.
We modify the reduction to balanced-perfect-dominating set as follows. We start by converting 
it to a graph $G=(V,E)$ in the same way as \citep{kcenter-pr}. Now we make four copies of this graph, $G_1,G_2,G_3,G_4$.
For each vertex $v$ in $G_1$ corresponding to an element in $T$ (denote this set by $GT_1$), we add edges from $v$ to its other three copies
in $G_2,G_3$ and $G_4$. Call the resulting graph $G'$. Note this reduction is still parsimonious.
It has the additional property that if a dominating set of size $|T|+3|M|$ exists, 
then each vertex in the dominating set hits exactly 4 vertices.
First, assume a 3-matching for the 3DM instance exists. Then we can achieve a dominating set of size $|T|+3|M|$.
Pick the vertices corresponding to the 3-matching, there are $4|M|$ such vertices, each of which have edges
to the 3 elements they represent.
We also put in the dominating set, the $|T|-|M|$ elements in $GT_1$ that are not in the 3-matching.
Each of these elements have edges to their 3 copies in $G_2,G_3$, and $G_4$.
This makes a full dominating set of size $|T|+3|M|$.
If there is no 3-matching for the 3DM instance, then the dominating set must be strictly larger.
Finally, the reduction from Unambiguous-Balanced Perfect Dominating Set to clustering is the exact same proof,
but now since each vertex in the dominating set hits 4 vertices, we get that each cluster is size exactly 4.

\section{Structure of Balanced Clustering} \label{app:structure}

\medskip
\noindent \textbf{Lemma~\ref{lem:localmax234} (restated).}
\emph{
There exists a balanced clustering instance in which the 
$k$-center, $k$-median, and $k$-means objectives
contain a local maximum with respect to $k$.
}

\begin{proof}
Consider the graph in Figure \ref{fig:localmax}, 
where $n=86$, and set $n\ell=21$.
Since the distances are all 1 or 2,
this construction is trivially a valid
  distance metric. 
	From Figure \ref{fig:localmax}, we see that $k=2$ and $k=4$ have
valid clusterings using only length 1 edges, using centers
$\{y_1,y_2\}$ and $\{x_1,x_2,x_3,x_4\}$, respectively.
But now consider $k=3$.
The crucial property is that by construction, $y_1$ and any $x_i$
cannot simultaneously be centers and each satisfy the capacity
to distance 1 points, because the union of their distance 1
neighborhoods is less than $2n\ell$. 
So we cannot just take the centers from $k=2$ and add a center
from $k=4$.
Formally, we show no possible set of 3 centers can be distance
1 to all points without violating the lower bound on the cluster sizes.

Case 1: the set of centers includes a point $p$ not in
$\{x_1,x_2,x_3,x_4,y_1,y_2\}$. The rest of the points are only
distance 1 from exactly two points, so $p$ cannot hit the lower
bound of 21 using only distance 1 assignments.

Case 2: the set of centers is a subset of $\{x_1,x_2,x_3,x_4\}$.
Then there are clearly 20 points which are not distance 1 from
the three centers.

Case 3: the set of centers includes both $y_1$ and $y_2$.
Then we need to pick one more center, $x_i$.
$x_i$ is distance 1 from 20 middle points, plus
$\{x_1,x_2,x_3,x_4,y_1,y_2\}$, so 26 total. $y_1$ is also distance
1 from 20 middle points and $\{x_1,x_2,x_3,x_4,y_1,y_2\}$.
$y_1$ and $x_i$ share exactly 5 neighbors from the middle points,
plus $\{x_1,x_2,x_3,x_4,y_1,y_2\}$ as neighbors.
Then the union of points that
$x_i$ and $y_1$ are distance 1 from, is $26+26-11=41$, which
implies that $x_i$ and $y_1$ cannot simultaneously reach the lower
bound of $21$ with only distance 1 points.

Case 4: the set of centers does not include $x_i$ nor $y_j$.
By construction, for each pair $x_i$ and $y_j$,
there exists some middle points which are only
distance 1 from $x_i$ and $y_j$.

These cases are exhaustive, so we conclude $\mathcal{OPT}_3$
must be strictly larger than $\mathcal{OPT}_2$ and
$\mathcal{OPT}_4$ (no matter what objective we use).
\end{proof}

Now we give the full proof for Theorem \ref{thm:localmax}.

\medskip
\noindent \textbf{Lemma~\ref{thm:localmax} (restated).}
\emph{
For all $m\in \mathbb{N}$, there exists a balanced clustering instance 
in which the $k$-center, $k$-median, and $k$-means objectives contain
$m$ local maxima, even for soft capacities.
}

\begin{proof}

\textbf{Setup.}
Set $k_{min}=10\cdot m$, and $k_{max}=12m$.
Define $K_{good}=\{k\mid k_{min}\leq k\leq k_{max}\text{ and }2\mid k\}$.
Similarly, let
$K_{bad}=\{k\mid k_{min}\leq k\leq k_{max}\text{ and }2\nmid k\}$.
Note $|K_{bad}|=m$ and $|K_{good}|=m+1$.
For all $k\in K_{good}$,
define $X_{k}=\{x_1^{(k)},\dots, x_{k'}^{(k)}\}$.
Let $X=\bigcup_{k} X_{k}$.

Define $G=(V,E)$, $V=X\cup Y$, $X\cap Y=\emptyset$.
Just like in the last proof, the edges later correspond to a distance of 1,
and all other distances are 2.
We will construct $Y$ and $E$ such that for all $k\in K_{good}$,
all the neighbors of $X_{k}$ form a partition of $Y$, i.e.\
$\forall k\in K_{good},~\bigcup_i N(x_i^{(k)})=Y$ and
$N(x_i^{(k)})\cap N(x_j^{(k)})=\emptyset$ for all $i\neq j$.
So taking $X_{k}$ as the centers
corresponds to a $k$-clustering in which all points are distance 1 from
their center.
We will also show that for all $k\in K_{bad}$, it is not possible to
find a valid set of centers for which every point has an edge to its center,
unless the capacities are violated.
This implies that all $m$ points in $K_{bad}$ are local maxima.

For all $k\in K_{good}$, $X_{k'}$ will have exactly
$\frac{k_{max}}{k'}l$ edges in $Y$
Thus, set $n\ell=\prod_{k\in K_{good}}k$ to make all of these
values integral.
Note that some points (those in $X_{k_{max}}$) have exactly $n\ell$ edges,
and all points have $\leq\frac{6}{5}n\ell$ edges
(which is tight for the points in $X_{k_{min}}$).

Now we define the main property which drives the proof.
We say $x_{i_1}^{(j_1)}$ \emph{overlaps} with $x_{i_2}^{(j_2)}$ if
$N(x_{i_1}^{(j_1)})\cup N(x_{i_2}^{(j_2)})>\frac{2}{5}n\ell$.
Note this immediately implies
it is not possible to include them in the same set of
centers such that each point has an edge to its center,
since $N(x_{i_1}^{(j_1)})\cup N(x_{i_2}^{(j_2)})\leq
N(x_{i_1}^{(j_1)})+N(x_{i_2}^{(j_2)})-N(x_{i_1}^{(j_1)})\cap N(x_{i_2}^{(j_2)})
< 2\cdot\frac{6}{5}n\ell-\frac{2}{5}n\ell=2n\ell$.

\textbf{Outline}.
We will construct $Y$ in three phases. First, we add edges to ensure
that for all $x_{i_1}^{(j_1)}$, for all $j_2\neq j_1$, there exists
an $i_2$ such that $x_{i_1}^{(j_1)}$ overlaps with $x_{i_2}^{(j_2)}$.
It follows that if we are trying to construct a set of centers from $X$ for
$k'\in K_{bad}$, we will not be able to use any complete $X_{k'}$ as a subset.
These are called the \emph{backbone} edges.

The next phase is to add enough edges among points in different $X_k$'s so that
no subset of $X$ (other than the $X_{k'}$'s) is a complete partition of $Y$.
We will accomplish this by adding a bunch of points to $Y$ shared by various
$x\in X$, so that each $x$ has edges to $k_{max}$ points in $Y$.
These are called the \emph{dispersion} edges.

The final phase is merely to add edges so that all points reach their assigned
capacity. We do this arbitrarily. These are called the \emph{filler} edges.

Note whenever we add a point to $Y$, for all $k\in K_{good}$,
we need to add an edge to exactly one $x\in X_k$, which will ensure that
all $X_k$'s form a partition of $Y$.

\textbf{Phase 1: Backbone edges}.
Recall that for $k,k'\in K_{good}$, we want $\forall i$, $\exists j$ such
that $x_i^{(k)}$ overlaps with $x_j^{(k')}$.
Since $k_{max}=\frac{6}{5}k_{min}$, some $x$'s will be forced to overlap with
two points from the same $X_k$.
However, we can ensure no point overlaps with three
points from the same $X_k$.

We satisfy all overlappings naturally by creating
$k_{min}$ components, $CC_1$ to $CC_{k_{min}}$.
Each component $CC_i$ contains point $x_i^{(k_{min})}$.
The rest of the sets $X_{k}$ are divided so that one or two points are in
each component, as shown in Figure \ref{fig:general_backbone}.
Formally,
in component $CC_i$, sets $X_{k_{min}}$ to
$X_{k_{min}+\lceil\frac{i}{2}\rceil}$
have one point in the component,
and all other sets have two points in the component.

For each component $CC_i$, we add $\frac{4}{5}n\ell$ points to $Y$, split into
two groups of $\frac{2}{5}n\ell$.
The points from
sets $X_{k_{min}+\lceil\frac{i}{2}\rceil}$ have edges to all $\frac{4}{5}n\ell$
 points,
and the points from the rest of the sets (since there are two from each set)
have edges to one group of $\frac{2}{5}n\ell$ points.
Therefore, for all $k,k'\in K_{good}$, each point $x\in X_k$ belongs to
some component $CC_i$, and overlaps with some $x'\in X_{k'}$,
so all of the overlapping requirements are satisfied (only using points
within the same component).

This completes phase 1.
Each point in $X$ had at most $\frac{4}{5}n\ell$ edges
added, so every point can still take at least $\frac{n\ell}{5}$ more edges in
subsequent phases.

\textbf{Phase 2: Dispersion edges}.
Now we want to add points to $Y$ to
ensure that no set of at most $k_{max}$ points from $X$
create a partition of $Y$, except sets that completely contain some $X_k$.

We have a simple way of achieving this. For every
$(x_1,x_2,\dots,x_{m+1})\in
 X_{k_{min}}\times X_{k_{min}+2}\times\cdots\times X_{k_{max}}$,
add one point to $Y$ with edges to $x_1,x_2,\dots,x_{m+1}$.
Then we have added $\prod_{k\in K_{good}} k$ total points
to $Y$ in this phase.

This completes phase 2.

\textbf{Phase 3: Filler edges}.
The final step is just to fill in the leftover points,
since we want every point
$x_i^{(k)}$ to have $\frac{k_{min}}{k}l$ points total.
All of the mechanisms for the proof have been
set up in phases 1 and 2, so these final points can be arbitrary.

We greedily assign points. Give each point $x_i^{(k)}\in X$ a number
 $t_{x_i^{(k)}}=\frac{k_{min}}{k}n\ell-N(x_i^{(k)})$, i.e., the number of extra points it needs.
Take the point $x\in X_{k}$ with the minimum $t$, and create $t$ points in
$Y$ with $x$. For each layer other than $X_{k}$,
add edges to the point with the smallest number.
Continue this process until $t=0$ for all points.

\textbf{Final Proof.}
Now we are ready to prove that $G$ has $m$ local maxima.
By construction, for all $k\in K_{good}$, $X_k$ is a set of centers
which satisfy the capacity constraints, and every point has an edge
to its center.
Now, consider a set $C$ of centers of size $k'\in K_{bad}$.
We show in every case, $C$ cannot satisfy the capacity constraints
with all points having edges to their centers.

Case 1: $C$ contains a point $y\in Y$. $y$ only has $m$ edges, which is
much smaller than $n\ell$.

Case 2: There exists $k\in K_{good}$ such that $X_k\subseteq C$.
Then since $|C|\notin K_{good}$, $\exists x\in C\setminus X_k$.
By construction, there exists $x_i^{(k)}\in X_k$ such that
$x$ and $x_i^{(k)}$ are overlapping.
Therefore, both centers cannot satisfy the capacity constraints with points
they have an edge to.

Case 3: For all $k\in K_{good}$, there exists $x\in X_k$ such that
$x\notin C$. Take the set of all of these points,
$x_1,x_2,\dots,x_{m+1}$.
By construction, there is a point $y\in Y$ with edges to only these points.
Therefore, $y$ will not have an edge to its center in this case.

This completes the proof.

\end{proof}

\section{Proofs from Section \ref{sec:sampleComplexity}} \label{app:nndispatch}
We begin by introducing some notation for the classes of clusterings
that satisfy the capacity constraints over the entire space with
respect to $\mu$, the weighted capacity constraints over the set $S$,
and the estimated capacity constraints on the set $S$: For any lower
bound $\ell$ and upper bound $L$ on the cluster capacities, we denote
the set of cluster assignments over $\X$ that satisfy the capacity
constraints by
\[
F(\ell, L) = \setcbigg{f : \X \to \kchp}{\prob_{x \sim \pdist}(i \in f(x)) \in [\ell, L] ~\forall i \in [k]}.
\]
Similarly, for the samples $S$, for true and estimated weights, define
the sets of feasible assignments respectively as:
\begin{align*}
G_n(\ell, L)
&= \setcbigg{g : S \to \kchp}{\sum_{j : i \in g(x)} w_j \in [\ell, L] ~\forall i \in [k]} \\
\hat G_n(\ell, L)
&= \setcbigg{g : S \to \kchp}{\sum_{j : i \in g(x)} \hat w_j \in [\ell, L] ~\forall i \in [k]}.
\end{align*}

\paragraph{Bounding $\alpha(S)$}

In this section, we prove the following Lemma bounding the $\alpha$
term from Theorem~\ref{th:sampleComplexity}:

\lemmaAlpha*

First we show that when the set $\X$ is bounded in $\reals^q$, then
for a large enough sample $S$ drawn from $\pdist$, every point
$x \in \X$ will have a close neighbor uniformly with high probability.

\begin{lemma}
  \label{lem:general_covering}
  For any $r > 0$ and any $\epsilon > 0$, there exists a subset
  $\mathcal{Y}$ of $\X$ containing at least $1-\epsilon$ of the
  probability mass of $\pdist$ such that, for any $\delta > 0$, if we
  see an iid sample $S$ of size
  $n = O( \frac{1}{\epsilon} (\frac{D \sqrt{q}}{r})^q(q \log
  \frac{D\sqrt{q}}{r} + \log \frac{1}{\delta}))$ drawn from $\pdist$,
  then with probability at least $1-\delta$ we have
  $\sup_{x \in \mathcal{Y}} \dist{x,\NN_S(x)} \leq r$ and
  $\sup_{x \in \mathcal{Y}} \dist{x, \NN_S(x)} \leq r^2$.
\end{lemma}
\begin{proof}
  Let $C$ be the smallest cube containing the support $\X$. Since the
  diameter of $\X$ is $D$, the side length of $C$ is at most $D$. Let
  $s = r/\sqrt{q}$ be the side-length of a cube in $\reals^q$ that has
  diameter $r$. Then it takes at most $m = \lceil D / s \rceil^q$
  cubes of side-length $s$ to cover the set $C$. Let $C_1$, \dots,
  $C_m$ be such a covering of $C$, where each $C_i$ has side length
  $s$.

  Let $C_i$ be any cube in the cover that has probability mass at
  least $\epsilon / m$ under the distribution $\pdist$. The
  probability that a sample of size $S$ drawn from $\pdist$ does not
  contain a sample in $C_i$ is at most $(1-\epsilon/m)^n$. Let $I$
  denote the index set of all those cubes with probability mass at
  least $\epsilon / m$ under $\pdist$. Applying the union bound over
  the cubes indexed by $I$, the probability that there exists a cube
  $C_i$ with $i \in I$ that does not contain any sample from $S$ is at
  most $m(1-\epsilon/m)^n \leq me^{-n\epsilon/m}$. Setting
  $n = \frac{m}{\epsilon}(\ln m + \log \frac{1}{\delta}) = O(
  \frac{1}{\epsilon} (\frac{D \sqrt{q}}{r})^q(q \log
  \frac{D\sqrt{q}}{r} + \log \frac{1}{\delta}))$ results in this upper
  bound being $\delta$. For the remainder of the proof, suppose that
  this high probability event occurs.

  Define $\mathcal{Y} = \bigcup_{i \in I} C_i$. Each cube from our
  cover not used in the construction of $\mathcal{Y}$ has probability
  mass at most $\epsilon/m$ and, since there are at most $m$ such
  cubes, their total mass is at most $\epsilon$. It follows that
  $\prob_{x \sim \pdist}(x \in \mathcal{Y}) \geq
  1-\epsilon$. Moreover, every point $x$ in $\mathcal{Y}$ belongs to
  one of the cubes, and every cube $C_i$ with $i \in I$ contains at
  least one sample point. Since the diameter of the cubes is $r$, it
  follows that the nearest sample to $x$ is at most $r$ away.

  Setting $r = D\epsilon$, we obtain the results for both $k$-median
  and $k$-means in the first half of Lemma~\ref{lem:alpha}.
\end{proof}

For the remainder of this section, suppose that $\pdist$ is a doubling
measure of dimension $d_0$ with support $\X$ and that the diameter of
$\X$ is $D > 0$.  First, we shall prove general lemmas about doubling
measures. They are quite standard, and are included here for the sake
of completion. See, for example,
\citep{krauthgamer2004navigating,kpotufe2010curse}.
\begin{lemma}
  \label{lem:problb}
  For any $x \in \X$ and any radius of the form $r = 2^{-T}D$ for some
  $T \in \mathbb{N}$, we have
  \[
  \pdist(B(x,r)) \geq (r/D)^{d_0}.
  \]
\end{lemma}
\begin{proof}
  Since $\X$ has diameter $D$, for any point $x \in \X$ we have that
  $\X \subset B(x,D)$, which implies that $\pdist(B(x,D)) =
  1$. Applying the doubling condition $T$ times gives
  $\pdist(B(x,r)) = \pdist(B(x,2^{-T}D)) \geq 2^{-Td_0} =
  (r/D)^{d_0}$.
\end{proof}

\begin{lemma}
  \label{lem:covering}
  For any radius of the form $r = 2^{-T}D$ for some
  $T \in \mathbb{N}$, there is a covering of $\X$ using balls of
  radius $r$ of size no more than $(2D/r)^{d_0}$.
\end{lemma}
\begin{proof}
  Consider the following greedy procedure for covering $\X$ with balls
  of radius $r$: while there exists a point $x \in \X$ that is not
  covered by our current set of balls of radius $r$, add the ball
  $B(x,r)$ to the cover. Let $C$ denote the set of centers for the
  balls in our cover. When this procedure terminates, every point in
  $\X$ will be covered by some ball in the cover.

  We now show that this procedure terminates after adding at most
  $(2D/r)^{d_0}$ balls to the cover. By construction, no ball in our
  cover contains the center of any other, implying that the centers
  are at least distance $r$ from one another. Therefore, the
  collection of balls $B(x,r/2)$ for $x \in C$ are pairwise
  disjoint. Lemma~\ref{lem:problb} tells us that
  $\pdist(B(x,r/2)) \geq (r/2D)^{d_0}$, which gives that
  $1 \geq \pdist\biggl(\,\bigcup_{x \in C} B(x,r/2)\biggr) = \sum_{x
    \in C} \pdist(B(x,r/2)) \geq |C| (r/2D)^{d_0}$.  Rearranging the
  above inequality gives $|C| \leq (2D/r)^{d_0}$.
\end{proof}

The next lemma tells us that we need a sample of size
$O\bigl( (\frac{D}{r})^{d_0} (d_0 \log \frac{D}{r} +
\log\frac{1}{\delta})\bigr)$ in order to ensure that there is a
neighbor from the sample no more than $r$ away from any point in the
support with high probability. The second half of
Lemma~\ref{lemma:alpha} is an easy corollary.

\begin{lemma}
  \label{lem:hpcovering}
  For any $r > 0$ and any $\delta > 0$, if we draw an iid sample $S$
  of size $n =$
  $(\frac {2D} r)^{d_0}(d_0 \log(\frac {4D} r) + \log(\frac 1
  \delta))$, then with probability at least $1-\delta$ we have
  $\sup_{x \in \X} \dist{x , \NN_S(x)} \leq r$ and
  $\sup_{x \in \X} \dist{x, \NN_S(x)} \leq r^2$.
\end{lemma}

\begin{proof}
  By Lemma~\ref{lem:covering} there is a covering of $\X$ with balls
  of radius $r/2$ of size $(4D/r)^{d_0}$. For each ball $B$ in the
  cover, the probability that no sample point lands in $B$ is
  $(1-\pdist(B))^n \leq (1-(r/2D)^{d_0})^n \leq
  \exp(-n(r/2D)^{d_0})$. Let $E$ be the event that there exists at
  least one ball $B$ in our cover that does not contain one of the $n$
  sample points.  Applying the union bound over the balls in the
  cover, we have that
  $\prob(E) \leq (4D/r)^{d_0}\exp(-n(r/2D)^{d_0})$.  Setting
  $n = (2D/r)^{d_0}(d_0 \log(4D/r) + \log(1/\delta)) = O\bigl(
  (\frac{D}{r})^{d_0} (d_0 \log \frac{D}{r} + \log
  \frac{1}{\delta})\bigr)$, we have that $\prob(E) < \delta$. When the
  bad event $E$ does not occur, every ball in our covering contains at
  least one sample point. Since every point $x \in \X$ belongs to at
  least one ball in our covering and each ball has diameter $r$, we
  have $\sup_{x \in \X} \dist{x , \NN_S(x)} \leq r$.
\end{proof}

\paragraph{Bounding $\beta(S)$}
In this section, we bound the bias term $\beta$ under the condition
that the optimal clustering satisfies the following probabilistic
Lipschitz condition:

\begin{definition}[Probabilistic Lipschitzness] Let $(\X, d)$ be some
  metric space of diameter $D$ and let
  $\phi:[0,1]\rightarrow[0,1]$. $f:\X\rightarrow [k]$ is
  $\phi$-Lipschitz with respect to some distribution $\pdist$ over
  $\X$, if $\forall\lambda \in [0, 1]$:
  $ \prob_{x\sim\pdist} \big[ \exists y: \ind{f(x) \neq f(y)} \text{
    and } {\dist{x,y}} \leq {\lambda D} \big] \leq \phi(\lambda) $
\end{definition}

\lemmaBeta*
\begin{proof}
  Suppose
  $\prob_{x\sim\pdist}(f^*(\NN_S(x)) \neq f^*(x)) \leq \epsilon$.

  Given a clustering $f$ of the entire space $\X$, define the
  restriction $f_S:S \rightarrow {k \choose p}$ to be $f_S(x) = f(x)$
  for $x \in S$. First, we show that the cluster sizes of $f^*_S$ can
  be bounded. Recall that the weights of cluster $i$ in a clustering
  $f$ of $\X$ and the extension of a clustering $g$ of the sample $S$
  are $\prob_{x\sim \pdist}(i \in f(x))$ and
  $\prob_{x\sim \pdist}(i \in \bar g(x))$, respectively. By the
  triangle inequality,
  $| \prob_{x\sim\pdist}( i \in \bar{f^*_S}(x)) -
  \prob_{x\sim\pdist}(i \in f^*(x)) | \leq
  \prob_{x\sim\pdist}(\bar{f^*_S}(x) \neq f^*(x) ) =
  \prob_{x\sim\pdist}(f^*(\NN_S(x)) \neq f^*(x))$ and this is at most
  $\epsilon$, by our assumption.

  Consider $\beta_1(S)$ for $k$-median. Since
  $f^* \in F(\ell + 2\epsilon, L - 2\epsilon)$, we have that
  $f^*_S \in G_n(\ell - \epsilon, L + \epsilon)$ and so
  \begin{align*}
    \beta_1& (S)
           \leq \Qmed(\bar{f^*_S}, c^*) - \Qmed(f^*, c^*) \\
           &= \expect_{x\sim\pdist} \bigg[ \sum_{i \in f^*(\NN_S(x))} \norm{x-c(i)} - \sum_{i' \in f^*(x)} \norm{x - c(i')} \bigg] \\
           &= \expect_{x\sim\pdist} \bigg[ \sum_{i \in f^*(\NN_S(x))} \norm{x-c(i)} - \sum_{i' \in f^*(x)} \norm{x - c(i')} \, \bigg|\, f^*(\NN_S(x)) = f^*(x) \bigg] \prob_{x \sim \pdist}(f^*(\NN_S(x)) = f^*(x))\\
           &+ \expect_{x\sim\pdist} \bigg[ \sum_{i \in f^*(\NN_S(x))} \norm{x-c(i)} - \sum_{i' \in f^*(x)} \norm{x - c(i')} \,\bigg|\, f^*(\NN_S(x)) \neq f^*(x) \bigg] \prob_{x \sim \pdist}(f^*(\NN_S(x)) \neq f^*(x))\\
  \end{align*}
  The first conditional expectation is zero, since the assignments for
  $x$ and $x'$ are identical. The second conditional expectation is at
  most $pD$, since the cost for any given point can change by at most
  $pD$ when we assign it to new centers. It follows that
  $\beta_1(S) \leq pD\epsilon$. For $k$-means, the second expectation
  is at most $pD^2$.
  
  It remains to show that show that
  $\prob_{x\sim\pdist}(f^*(\NN_S(x)) \neq f^*(x)) \leq \epsilon$ for
  big enough $n$.  Lemma \ref{lem:NNC} lists the conditions when this
  is true.
\end{proof}

We require the following lemma for nearest neighbor classification,
similar in spirit to that of Urner et al \citep{urner2013plal}. Note
that since $f$ is a set of $p$ elements, this lemma holds for
multi-label nearest neighbor classification.

\begin{lemma}
  \label{lem:NNC}
  Let $\pdist$ be a measure on $\reals^q$ with support $\X$ of
  diameter $D$. Let the labeling function, $f$ be $\phi$-PL.  For any
  accuracy parameter $\epsilon$ and confidence parameter $\delta$, if
  we see a sample $S$ of size at least
  \begin{itemize}
  \item
    $\frac{2}{\epsilon} \Big\lceil
    \frac{\sqrt{q}}{\phi\inv(\epsilon/2))}\Big\rceil^q \Big(q \log
    \lceil \frac{\sqrt{q}}{\phi\inv(\epsilon/2)}\rceil + \log
    \frac{1}{\delta}\Big)$ in the general case
  \item
    $\Big(\frac{2}{\phi\inv(\epsilon)}\Big)^{d_0} \Big(d_0 \log\frac
    {4} {\phi\inv(\epsilon)} + \log \frac{1}{\delta}\Big)$ when
    $\pdist$ is a doubling measure of dimension $d_0$
  \end{itemize}
  then nearest neighbor classification generalizes well. That is, with
  probability at least $1-\delta$ over the draw of $S$, the error on a
  randomly drawn test point,
  $\prob_{x\sim\pdist}(f(x) \neq f(\NN_S(x))) \leq \epsilon$.
\end{lemma}

\begin{proof}
  Let $\lambda = \phi\inv(\epsilon)$.  We know that most of $\X$ can
  be covered using hypercubes in the general case, as in
  Lemma~\ref{lem:general_covering} or entirely covered using balls in
  the case when $\pdist$ is a doubling measure, as in
  Lemma~\ref{lem:covering}, both of diameter $\lambda D$.  In case we
  have cubes in the cover, we shall use a ball of the same diameter
  instead. This does not change the sample complexity, since a cube is
  completely contained in a ball of the same diameter.

  Formally, let $\mathcal{C}$ be the covering obtained from
  Lemma~\ref{lem:general_covering} or Lemma~\ref{lem:covering},
  depending on whether or not the measure is a doubling measure.
  Define $\mathcal{B}(x)$ to be the set of all the balls from
  $\mathcal{C}$ that contain the point $x$.  A point will only be
  labeled wrongly if it falls into a ball with no point from $S$, or a
  ball that contains points of other labels.  Hence,
  \begin{align*}
    \prob_{x\sim\pdist} (f(\NN_S(x)) \neq f(x))
    \leq \prob_{x\sim\pdist} (\forall C\in \mathcal{B}(x): S\cap C = \emptyset) \\
        + \prob_{x\sim\pdist} (\exists y \in \bigcup_{C \in \mathcal{B}(x)} C: f(y) \neq f(x))
  \end{align*}
  Since each ball is of diameter $\lambda D$, the second term is at
  most
  $\prob_{x\sim\pdist} (\exists y \in B(x, \lambda D): f(y) \neq
  f(x))$.  By the PL assumption, this is at most
  $\phi(\lambda) = \epsilon$, independent of the covering used.

  For the first term, our analysis will depend on which covering we
  use:
  \begin{itemize}

  \item From Lemma~\ref{lem:general_covering}, we know that all but
    $1-\epsilon$ fraction of the space is covered by the covering
    $\mathcal{C}$. When the sample is of size
    $O(\frac 1 \epsilon(\frac{\sqrt{q}}{\lambda})^q(q \log
    (\frac{\sqrt{q}}{\lambda}) + \log \frac{1}{\delta}))$, each
    $C \in \mathcal{C}$ sees a sample point. For a sample this large,
    the first term is $\leq\epsilon$.  Substituting $\epsilon$ with
    $\epsilon/2$ completes this part of the proof.

  \item When $\pdist$ is a doubling measure, we can do better. If
    every ball of the cover sees a sample point, the first term is
    necessarily zero. From the proof of Lemma~\ref{lem:hpcovering}, we
    know that if we draw a sample of size
    $n = (2/\lambda)^{d_0}(d_0 \log(4/\lambda) + \log(1/\delta))$
    samples, then every ball of the cover sees a sample point with
    probability at least $1-\delta$ over the draw of $S$. This
    completes the proof.
  \end{itemize}
\end{proof}

\section{Details for the Experiments} \label{app:experiments}
\paragraph{Experimental System Setup} We now describe the distributed
implementation used for the experiments. We start one worker process
on each of the available processing cores. First, a single worker
subsamples the data, clusters the subsample into $k$ clusters, and
then builds a random partition tree for fast nearest neighbor
lookup. The subsample, clustering, and random partition tree describe
a dispatching rule, which is then copied to every worker. Training the
system has two steps: first, the training data is dispatched to the
appropriate workers, and then each worker learns a model for the
clusters they are responsible for. During the deployment phase, the
workers load the training data in parallel and send each example to
the appropriate workers (as dictated by the dispatch rule). To
minimize network overhead examples are only sent over the network in
batches of 10,000. During the training phase, each worker calls either
Liblinear or an L-BFGS solver to learn a one-vs-all linear classifier
for each of their clusters. For testing, the testing data is loaded in
parallel by the workers and the appropriate workers are queried for
predictions.

\paragraph{Details of LSH-baseline} The LSH family used by our LSH
baseline is the concatenation of $t$ random projections followed by
binning. Each random projection is constructed by sampling a direction
$u$ from the standard Gaussian distribution in $\reals^d$. An example
$x$ is then mapped to the index $\lfloor u^{\top}x/w \rfloor$, where
$w$ is a scale parameter. Two points $x$ and $y$ map to the same bin
if they agree under all $t$ hash functions. In our experiments, the
parameter $t$ was set to 10 and $w$ is chosen so that hashing the
training data results in approximately $2k$ non-empty bins. We tried
several other values of $t$ and $w$ but performance did not change
significantly.

\paragraph{Details for Synthetic Data Distribution} The synthetic
distribution used in Section~\ref{sec:expt} is an equal-weight mixture
of 200 Gaussians in $\reals^{20}$ with means chosen uniformly at
random from the unit cube $[0,1]^{20}$. Each Gaussian is associated
with one of 30 class labels. To decide the class labels, we construct
a hierarchical clustering of the Gaussian centers using complete
linkage and assign labels using a simple recursive procedure: each
internal node of the subtree is associated with a range of permissible
labels. The permissible labels assigned to the left and right children
of the parent node are a partition of the parent's assigned labels,
and the number of labels they receive is proportional to the number of
leaves in the subtree. If a subtree has only one permissible label,
both children are given that label. Finally, each leave chooses a
label uniformly at random from the set of permissible labels it was
assigned (in many cases, there will only be one). This labeling
strategy results in nearby Gaussians having similar labels.

\paragraph{Inception Network} The specific architecture for the neural
network we used when constructing the feature representations for the
CIFAR-10 dataset can be found here:
\url{https://github.com/dmlc/mxnet/blob/master/example\\/notebooks/cifar10-recipe.ipynb}.

\paragraph{Hardware} The experiments for MNIST-8M, CIFAR10, and the
CTR datasets were performed on a cluster of 15 machines, each with 8
Intel(R) Xeon(R) cores of clock rate 2.40 GHz and 32GB shared memory
per machine. The experiments for the large synthetic experiment were
performed on AWS. We used clusters of 8, 16, 32, and 64 m3.large EC2
instances, each with an Intel (R) Xeon E5-2670 v2 processors and 7.5
GB memory per machine.

\paragraph{Clustering algorithm selection} In
Section~\ref{sec:stability} we showed that $k$-means++ will find
high-quality balanced clusterings of the data whenever a natural
stability condition is satisfied. Since the $k$-means++ algorithm is
simple and scalable, it is a good candidate for implementation in real
systems. In this section we present an empirical study of the quality
of the clusterings produced by $k$-means++ for the datasets used in
our experiments. For each of the datasets used in our learning
experiments, we find that the clustering obtained by $k$-means++ is
very competitive with the $LP$ rounding techniques. We also include a
synthetic dataset designed specifically so that $k$-means++ with
balancing heuristics will not perform as well as the LP-rounding
algorithms.

We compare the clustering algorithms on two metrics: (1) the $k$-means
objective value of the resulting clustering and (2) the mean
per-cluster class distribution entropy. Since the LP rounding
algorithms may violate the replication by a factor of 2, we use an
averaged version of the $k$-means objective
\[
Q(f, c) = \sum_{i = 1}^n \frac{1}{|f(x_i)|} \sum_{j \in f(x_i)} \norm{x - c_j}^2,
\]
which provides a fair comparison when $|f(x_i)|$ is not exactly $p$
for all points. The second metric is the mean per-cluster class
distribution entropy which measures how well the clusters are
capturing information about the class labels. Each cluster has an
empirical distribution over the class labels and this distribution
will have low entropy when the cluster contains mostly examples from a
small number of classes. Therefore, when the average
class-distribution entropy per cluster is small, we expect the
learning problems for each cluster to be simpler than the global
learning problem, which should lead to improved accuracy. Formally,
given a dataset $(x_1, y_1), \dots, (x_n, y_n)$ with
$y_i \in \{1, \dots, M\}$ and a clustering $(f,c)$, we compute
\[
H(f,c) = -\frac{1}{k} \sum_{j=1}^k \sum_{y=1}^M p_{j,y} \log_2(p_{j,y}),
\]
where $p_{j,y}$ is the fraction of the points in cluster $j$ that
belong to class $y$.

\paragraph{Results} We run the $k$-means++ algorithm with balancing
heuristics described in Section~\ref{sec:expt}, the LP-rounding
algorithm for the $k$-means objective, and the LP-rounding algorithm
for the $k$-median objective. For each dataset, we randomly subsample
700 points and run the algorithm for values of $k$ in
$\{8, 16, 32, 64, 128\}$ with $p=2$ and compute the above
metrics. This is averaged over 5 runs. The $k$-means objective values
are shown in Figure~\ref{fig:kmeansobjective} and the mean per-cluster
class entropies are shown in Figure~\ref{fig:classentropy}.

For every dataset, the $k$-means++ algorithm finds a clustering of the
data that has better $k$-means objective value than the LP-rounding
algorithms, and for all but the large values of $k$, the per-cluster
class entropies are also smaller for $k$-means++.  It is interesting
to note that the LP-rounding algorithm for $k$-median achieves
\emph{better} $k$-means objective that the LP-rounding algorithm for
$k$-means!  This might be explained by the smaller approximation
factor for $k$-median.  The balancing heuristics for $k$-means++ never
resulted in the algorithm outputting more than $k$ clusters (though
for $k=128$, it output on average 16 too few clusters). Finally, the
LP-rounding algorithms assigned the majority of points to only 1
center (which is allowed under our bi-criteria analysis). This reduces
the cluster sizes, which explains why the per-cluster class entropy is
lower for the LP rounding algorithms when $k$ is large. These results
justify the use of $k$-means++ in our distributed learning
experiments.

\begin{figure*}
  \centering
  \subfigure[][{\scriptsize MNIST-8M}]{\includegraphics[width=0.40\textwidth]{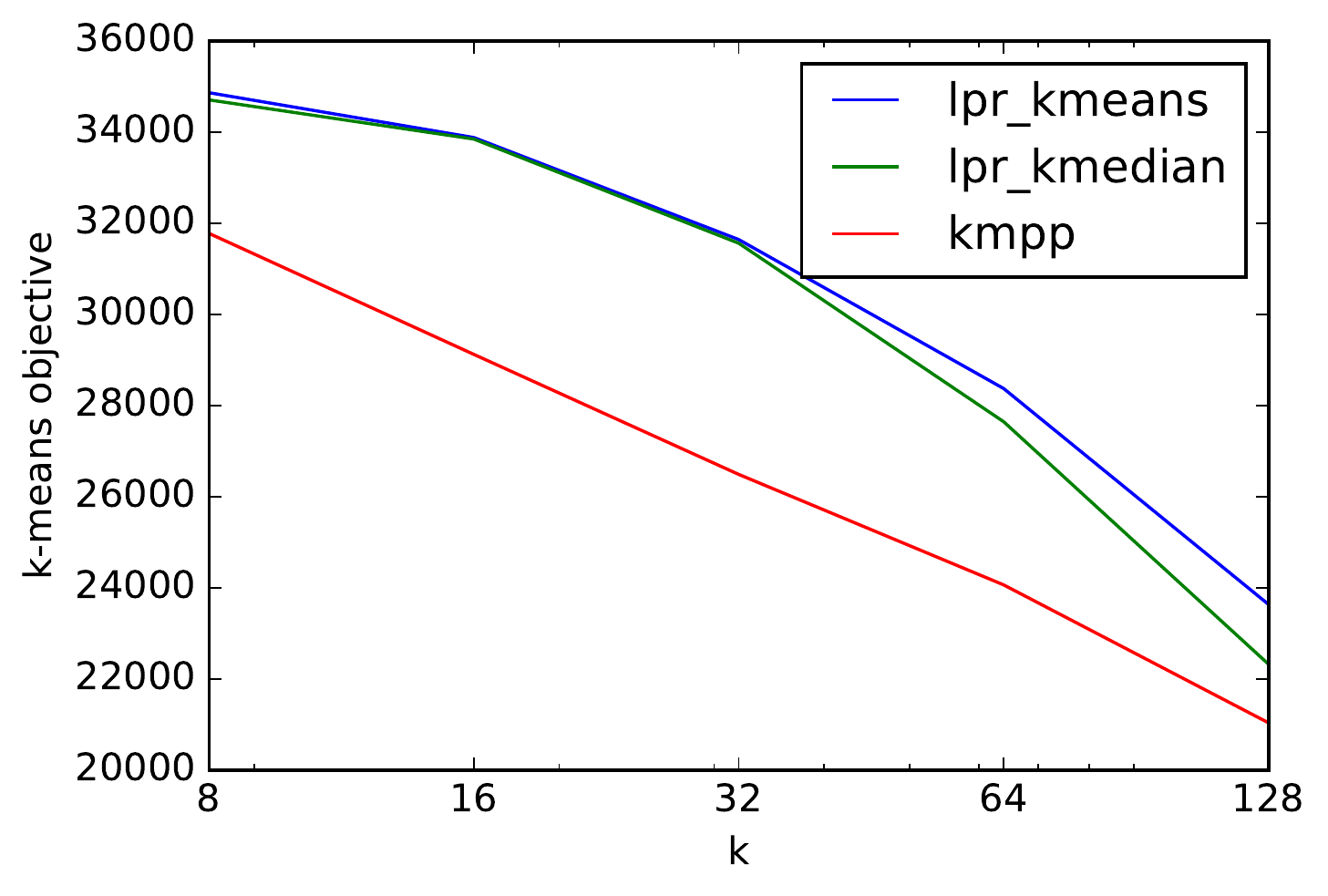}}%
  \subfigure[{\scriptsize CIFAR-10 Early Features}]{\includegraphics[width=0.40\textwidth]{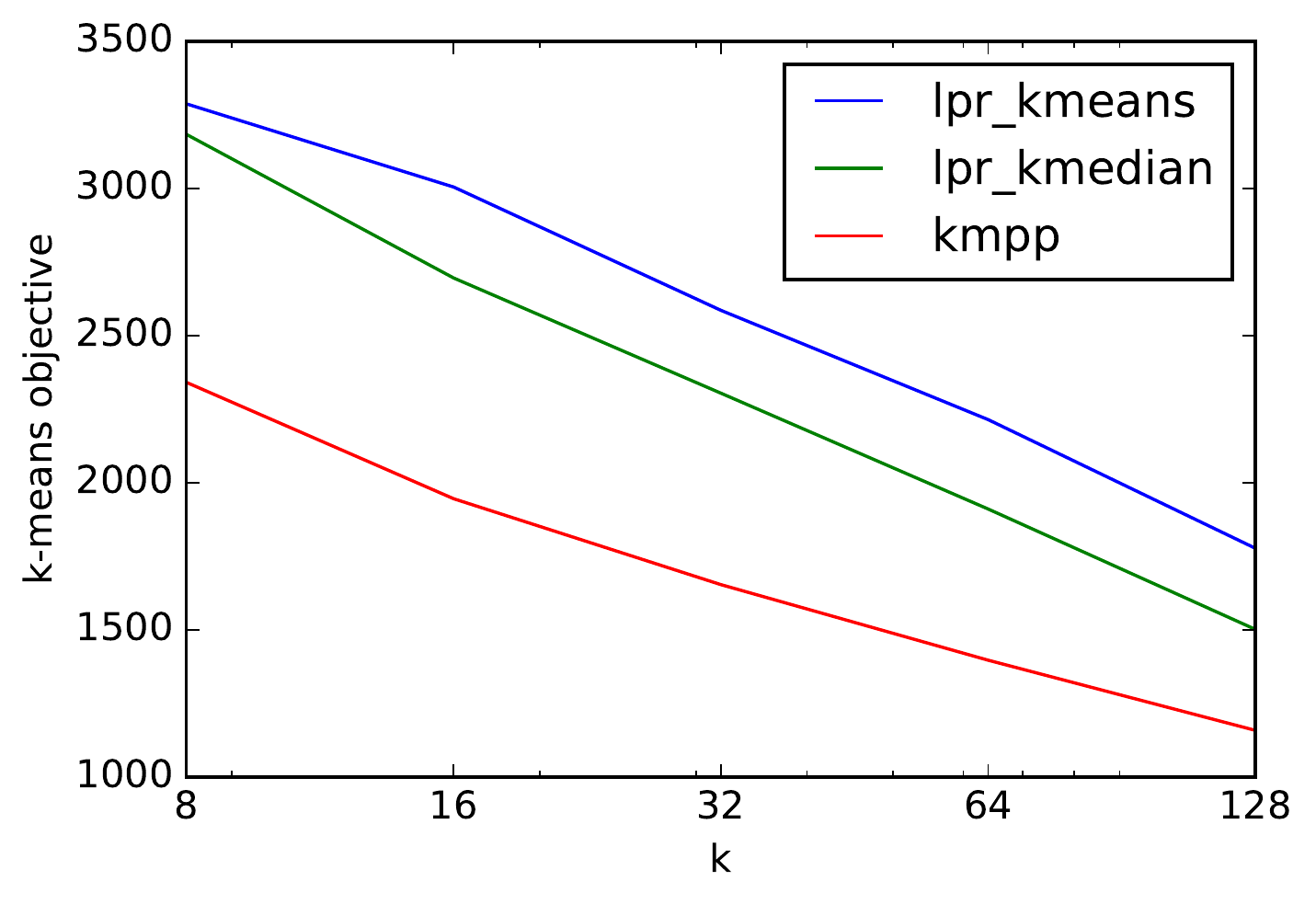}}\\%
  \subfigure[{\scriptsize CIFAR-10 Late Features}]{\includegraphics[width=0.40\textwidth]{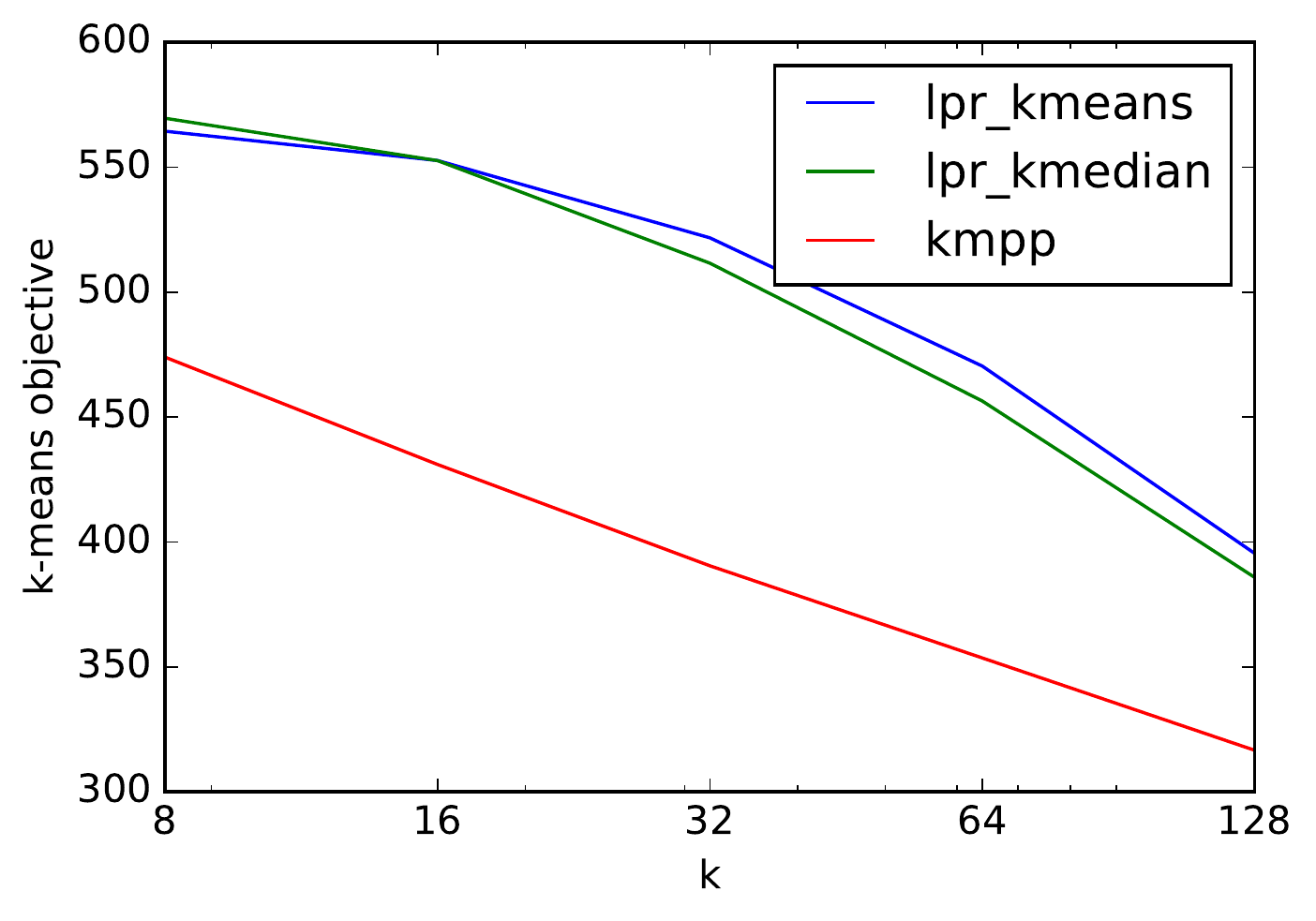}}%
  \subfigure[][{\scriptsize Synthetic Dataset}]{\includegraphics[width=0.40\textwidth]{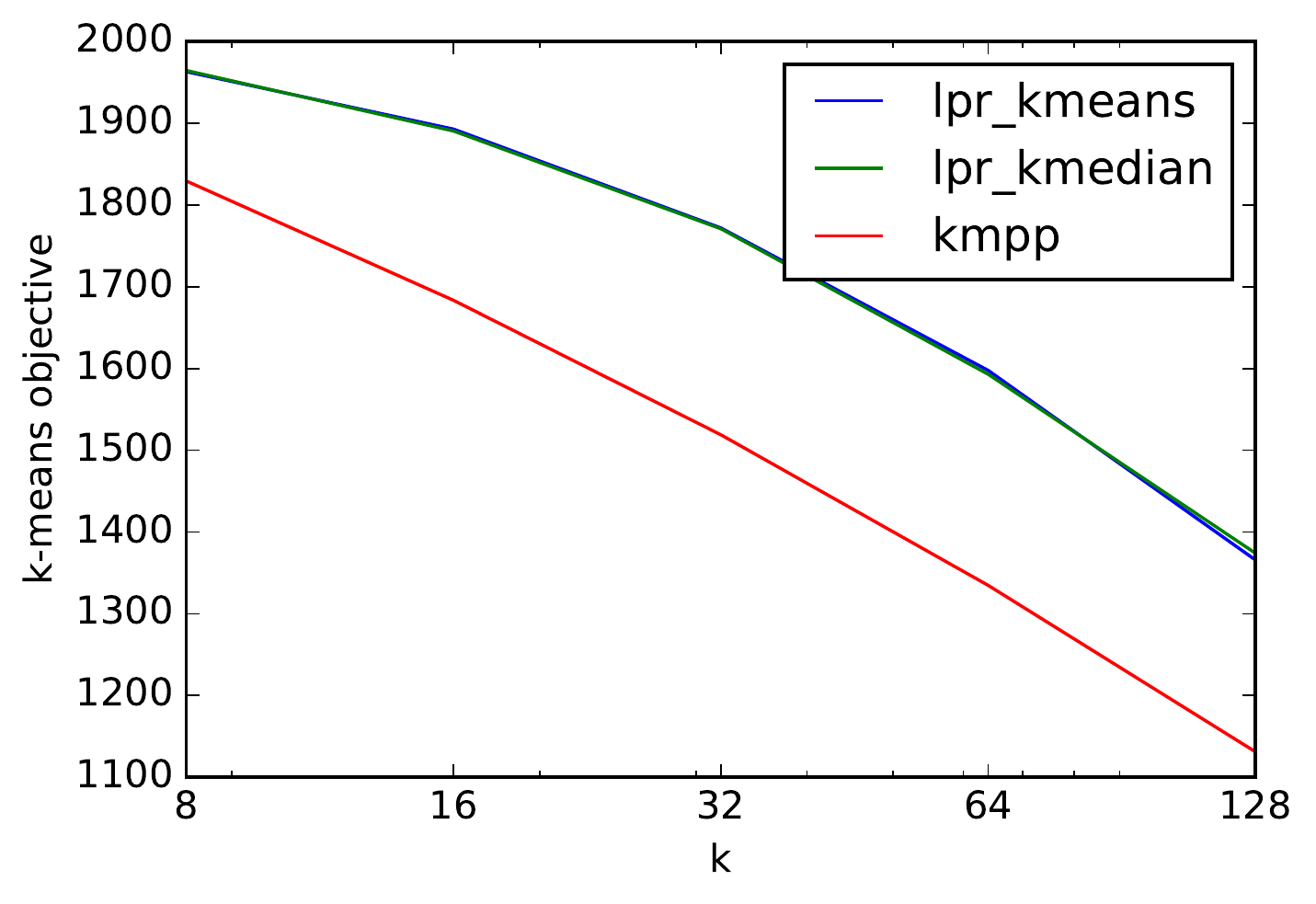}}%
\caption{Comparison of the $k$-means objective value.}
\label{fig:kmeansobjective}
\end{figure*}

\begin{figure*}
  \centering
  \subfigure[][{\scriptsize MNIST-8M}]{\includegraphics[width=0.40\textwidth]{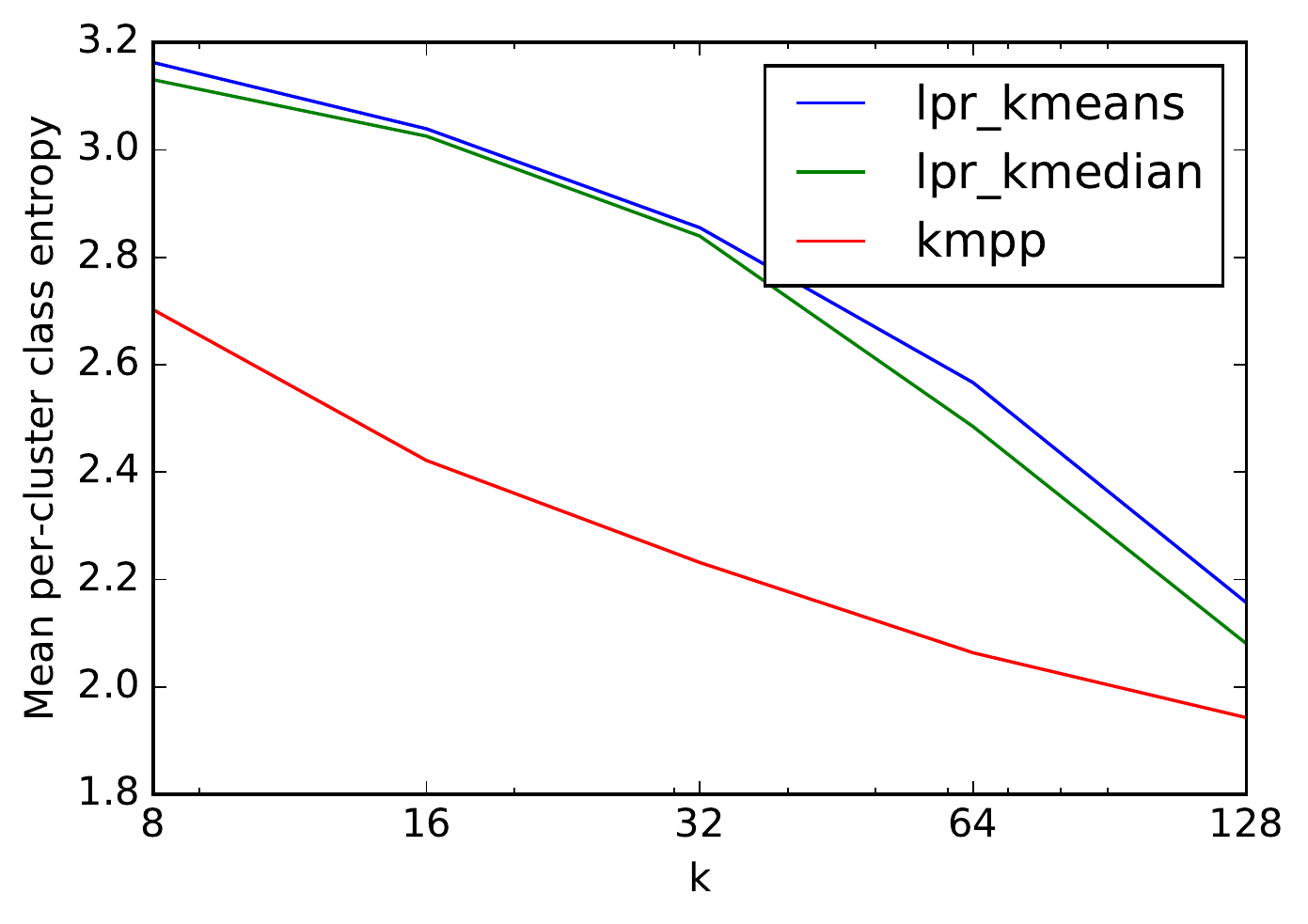}}%
  \subfigure[{\scriptsize CIFAR-10 Early Features}]{\includegraphics[width=0.40\textwidth]{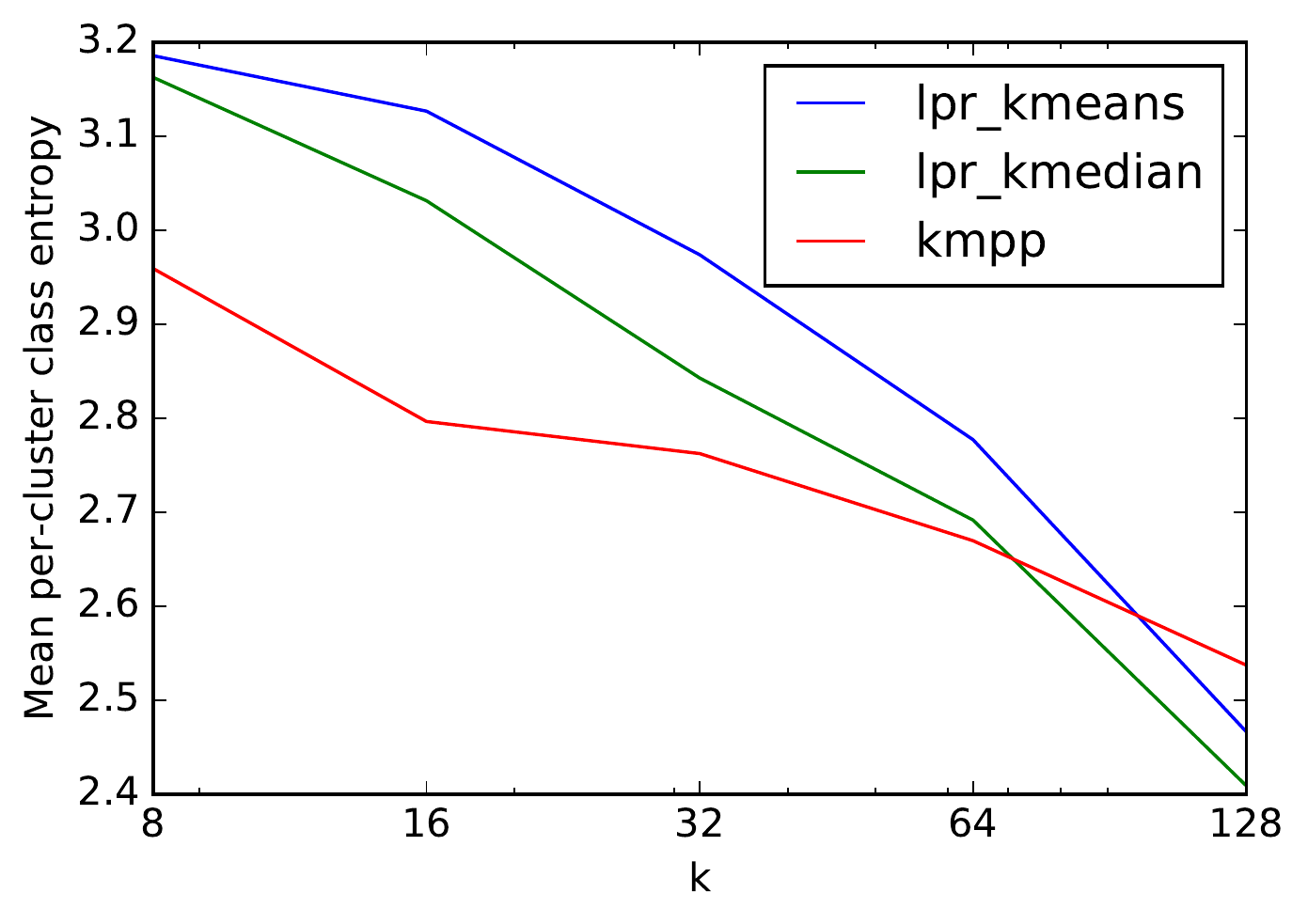}}\\%
  \subfigure[{\scriptsize CIFAR-10 Late Features}]{\includegraphics[width=0.40\textwidth]{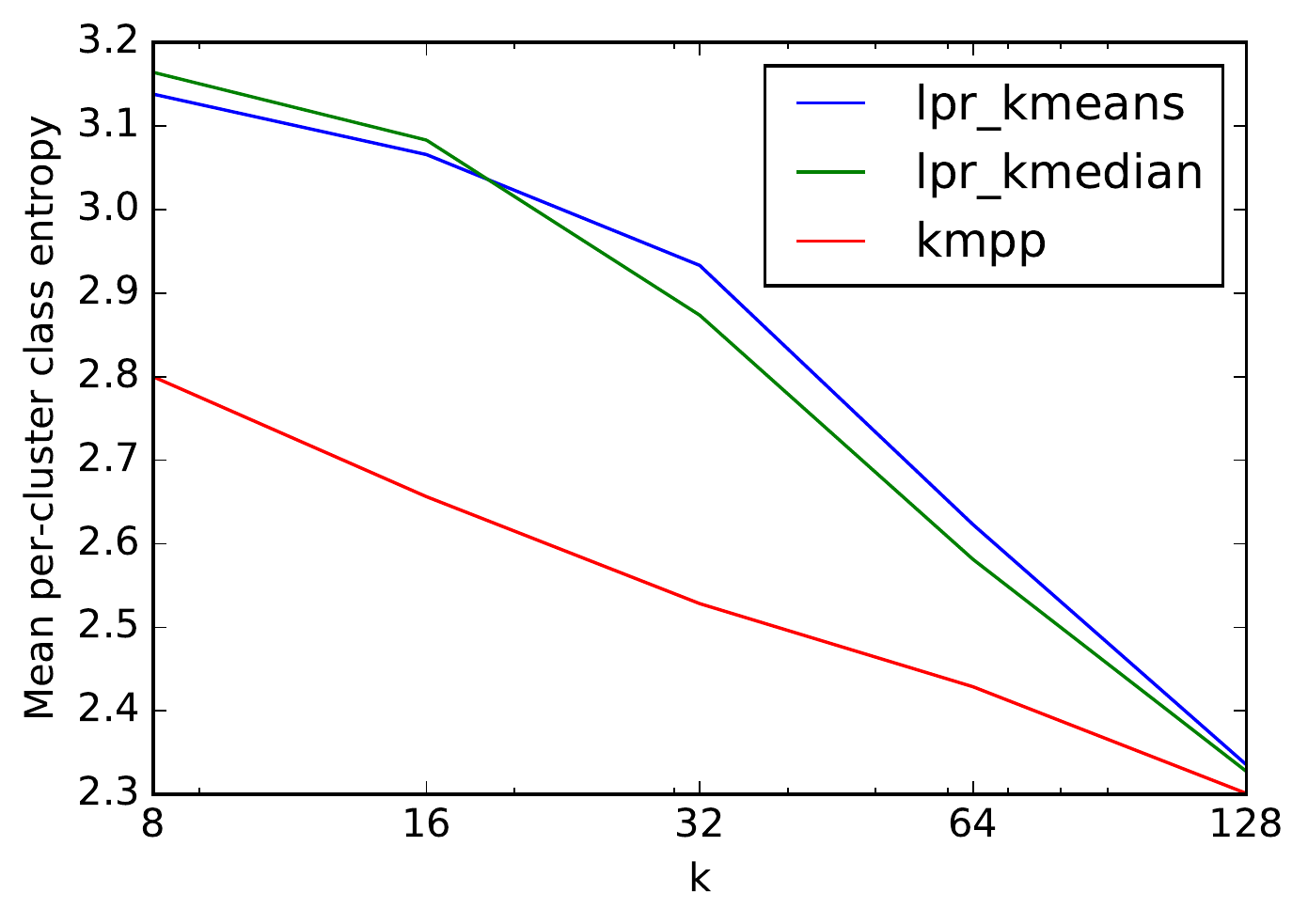}}%
  \subfigure[][{\scriptsize Synthetic Dataset}]{\includegraphics[width=0.40\textwidth]{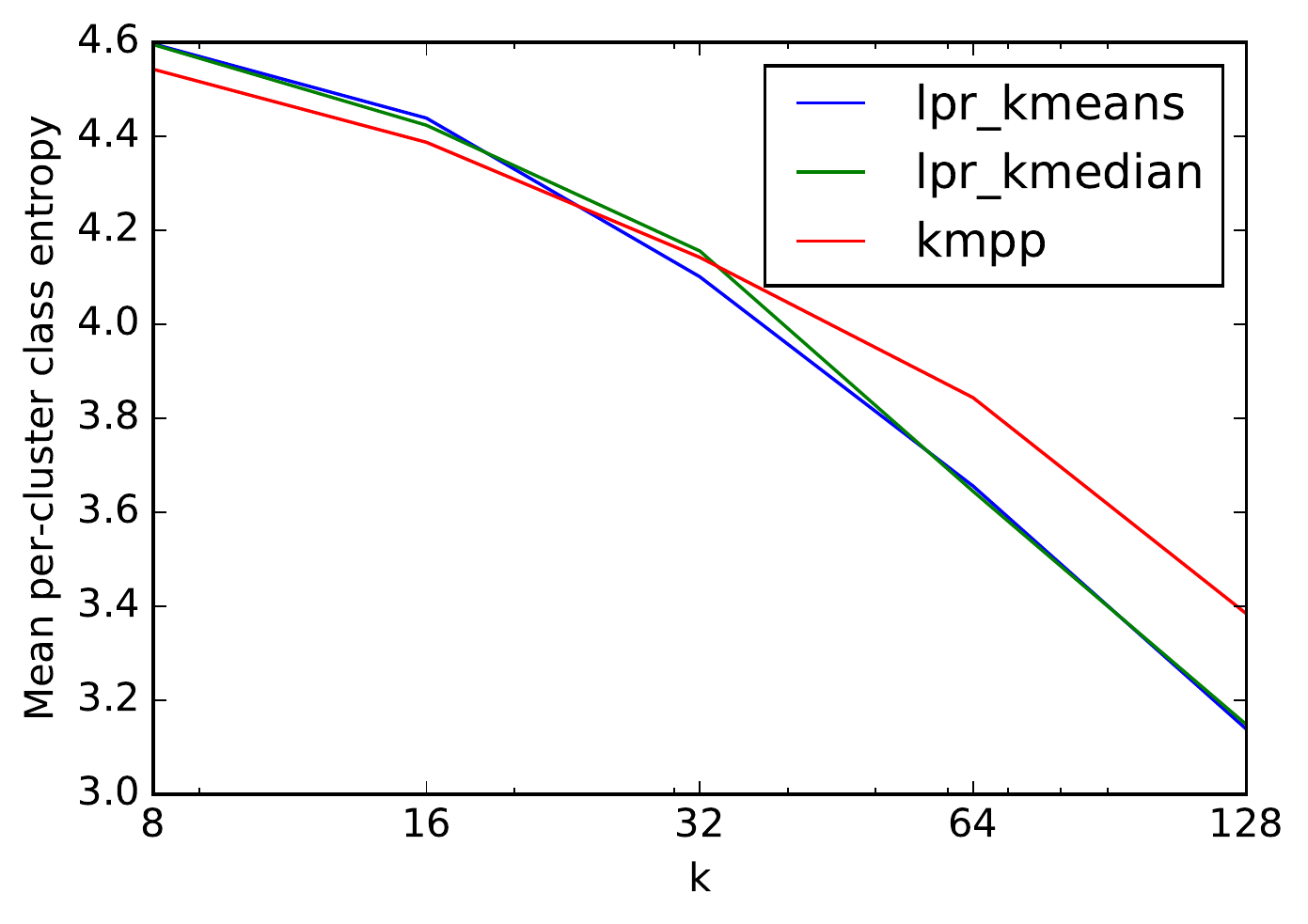}}%
\caption{Comparison of the mean per-cluster class distribution entropy.}
\label{fig:classentropy}
\end{figure*}

\paragraph{Additional Synthetic Distribution for Bounded Partition
  Trees} In this section we present an additional synthetic
distribution for which our algorithm significantly outperforms the
balanced partition tree. The data distribution is uniform on the
100-dimensional rectangle
$[0,10] \times [0,10] \times [0,1] \times \dots \times [0,1]$, where
the first two dimensions have side length 10 and the rest have side
length 1. The class of an example depends only on the first 2
coordinates, which are divided in to a regular $4 \times 4$ grid with
one class for each grid cell, giving a total of 16
classes. Figure~\ref{fig:bptsample} shows a sample from this
distribution projected onto the first two dimensions. We use either
balanced partition trees or our algorithm using $k$-means++ to
partition the data, and then we learn a linear one-vs-all SVM on each
subset. If a subset is small enough to only intersect with one or two
grid cells, then the learning problem is easy. If the subset
intersects with many grid cells, there is not usually a low-error
one-vs-all linear classifier.

\begin{figure}
  \centering
  \includegraphics[width=0.4\columnwidth]{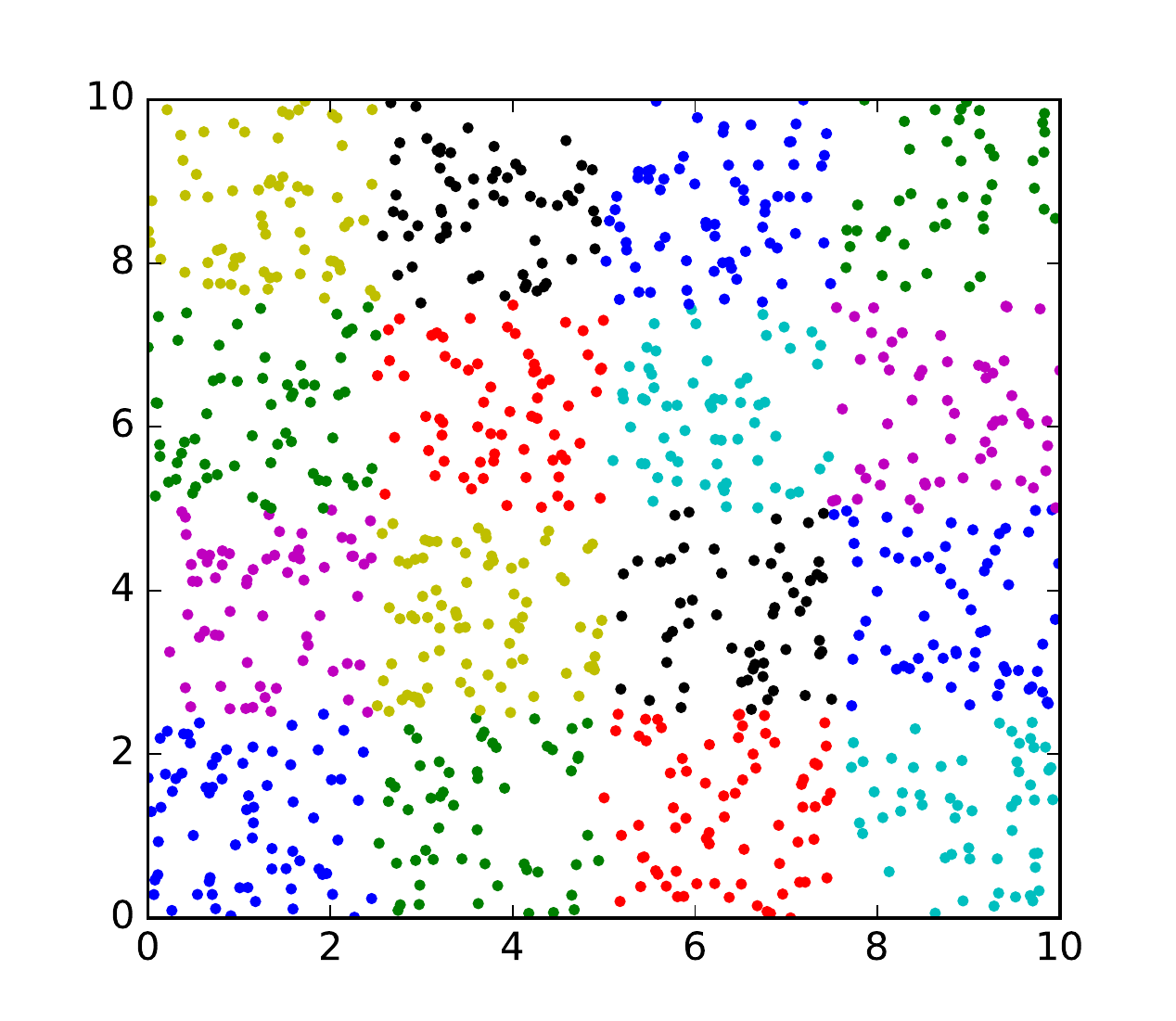}  
  \caption{Scatter plot of the data after projecting onto the first
    two coordinates (note: some colors are reused)}
  \label{fig:bptsample}
\end{figure}

Intuitively, balanced partition trees fail on this dataset because in
order for them to produce a good partitioning, they need to repeatedly
split on one of the first two dimensions. Any time the balanced
partition tree splits on another dimension, the two resulting learning
problems are identical but with half the data. On the other hand,
clustering-based approaches will naturally divide the data into small
groups, which leads to easier learning problems. The accuracies for
the balanced partition trees and $k$-means++ are shown in
Figure~\ref{fig:badbpt}. Our method is run with parameters $p = 1$,
$\ell = 1/(2k)$, and $L = 2/k$.

\begin{figure}
  \centering
  \includegraphics[width=0.3\columnwidth]{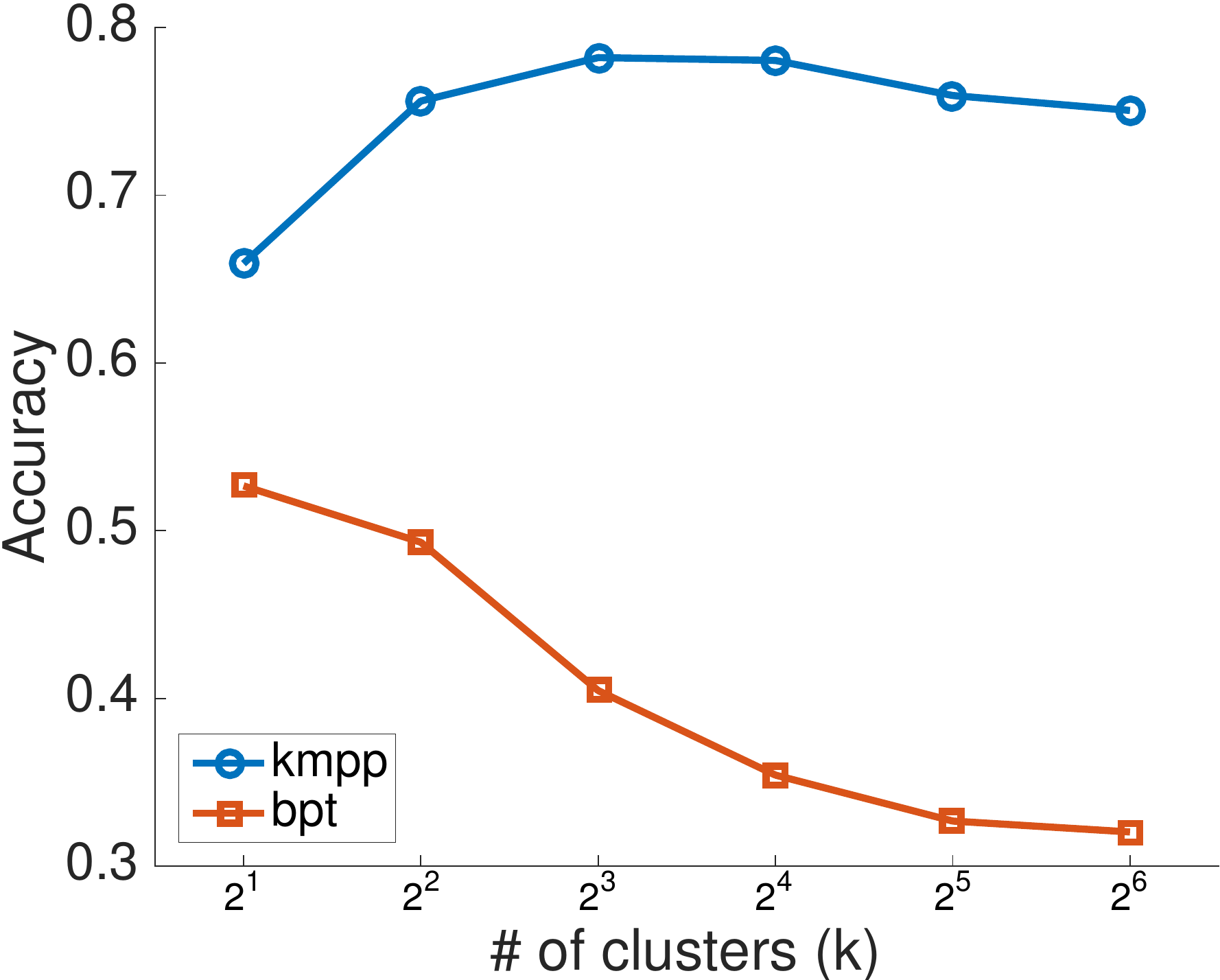}
  \caption{Accuracy of learning using balanced partition trees and
    $k$-means++}
  \label{fig:badbpt}
\end{figure}

\section{Comparison of Clustering Algorithms}\label{app:comparison}
In this section, we empirically and theoretically compare the LP
rounding and $k$-means++ algorithms on a data distribution designed
specifically to show that in some cases, the LP rounding algorithm
gives higher performance.

The synthetic distribution is a mixture of two Gaussians in
$\reals^2$, one centered at $(0,0)$ and the other centered at
$(10,0)$. We set the balancing constraints to be $\ell = 1/10$ and
$L = 1$ so that no cluster can contain fewer than 10\% of the
data. The mixing coefficient for the Gaussian at $(10,0)$ is set to
$0.8 \ell = 0.08$, so in a sample this Gaussian will not contain
enough data points to form a cluster on its own. In an optimal
$2$-clustering of this data with the given constraints, cluster
centered at $(10,0)$ will steal some additional points from the
Gaussian centered at $(0,0)$. Running the $k$-means++ algorithm,
however, will produce a clustering that does not satisfy the capacity
constraints, and the merging heuristic described in
Section~\ref{sec:expt} will simply merge the points into one
cluster. The following labeling function is designed so that there is
no globally accurate one-vs-all linear classifier, but for which there
is an accurate classifier for each cluster in the optimal clustering.
\[
  f(x) = \begin{cases}
    -1 & \hbox{if $x_1 \leq 0$} \\
    1 & \hbox{if $x_1 \in (0,5]$} \\
    2 & \hbox{otherwise}.
  \end{cases}
\]

\begin{figure}[h]
  \centering
  \includegraphics[width=0.7\columnwidth]{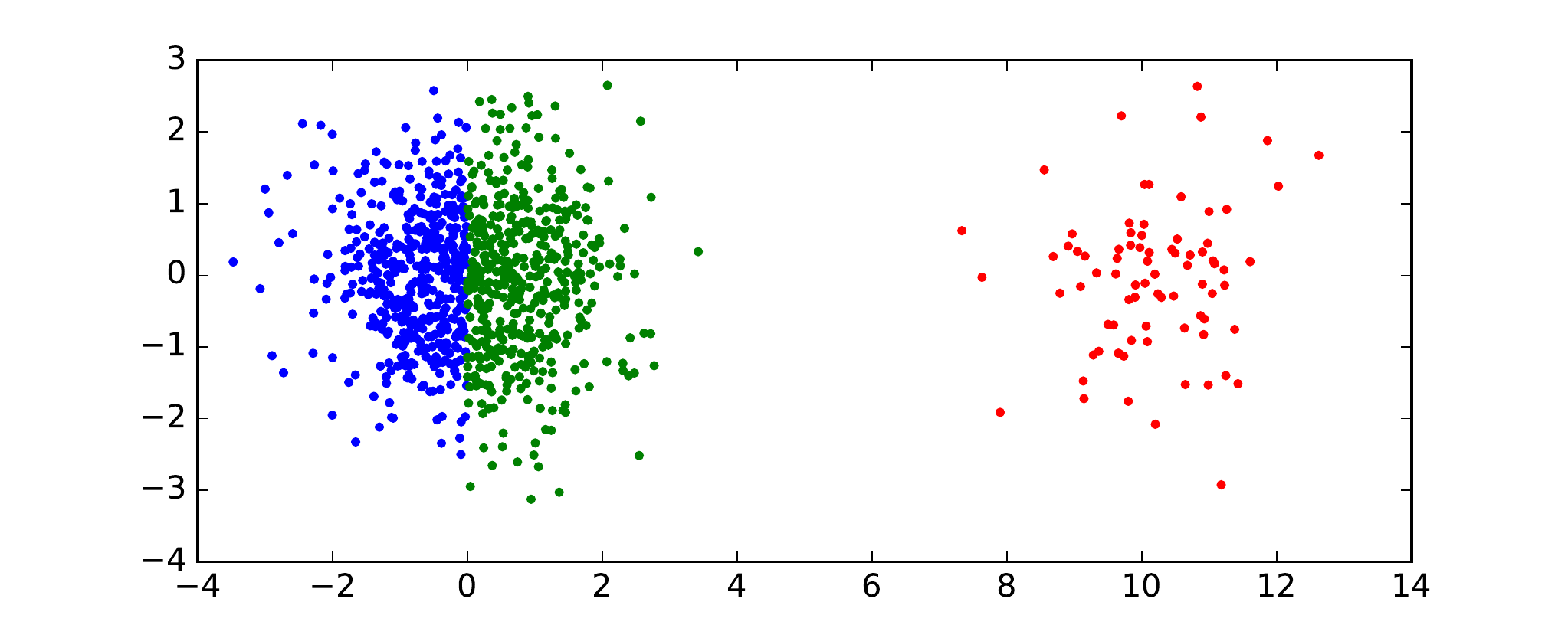}
  \caption{A dataset which has two well defined clusters that are not
    balanced in size.}
  \label{fig:badkmpp_data}
\end{figure}

The LP rounding algorithm requires the replication parameter $p$ to be
at least two, so we run both algorithms with $p=2$ and $k=4$, in which
case the above intuitions still hold but now each of the clusters is
assigned two centers instead of one. Figure~\ref{fig:badkmpp_data}
shows a sample drawn from this distribution labeled according to the
above target function.

We evaluate the LP rounding algorithm, $k$-means++, and an algorithm
that optimally solves the clustering problem by solving the
corresponding integer program (this is only feasible for very small
input sizes). Other aspects of the experiment are identical to the
large scale learning experiments described in
Section~\ref{sec:expt}. In all cases, we set the parameters to be
$k = 4$, $p = 2$. The training size is 10,000, the testing size 1,000,
and the clustering is performed on a sample of size 200. Running
$k$-means++ results in accuracy {\bf 0.768} (averaged over 500 runs),
using the LP rounding algorithm results in accuracy {\bf 0.988}, and
exactly solving the capacitated clustering IP results in accuracy {\bf
  0.988}. Since the LP rounding and IP based algorithms are
deterministic, we did not run the experiment multiple times. The
accuracy of $k$-means++ does not significantly increase even if we
cluster the entire sample of training 10,000 points rather than a
small sample of size $200$. This experiment shows that, while the
$k$-means++ heuristic works effectively for many real-world datasets,
it is possible to construct datasets where its performance is lower
than the LP rounding algorithm.

With a modification to this distribution, but with the same intuition,
we can prove there is a point set in which the LP rounding algorithm
outperforms $k$-means++.

\begin{lemma}
  \label{lem:outperform}
  There exists a distribution such that with constant probability,
  classification using the $k$-median LP rounding algorithm as a
  dispatcher outperforms the $k$-means++ algorithm with a balancing
  heuristic.
\end{lemma}
\begin{proof}
  The point set is as follows. All points are on a line, and there are
  three groups of points, group $A$, $B$, and $C$. Two points in the
  same group are at distance zero.  Group $A$ is distance 1 from group
  $B$, and group $B$ is distance $10$. Group $A$ is distance $101$
  from group $C$. Group $A$ contains 112 points, group $B$ contains
  111 points, and group $C$ contains 1 point. Set $k=2$, $n\ell=112$,
  and $L=1$. For now, $p=1$.  All points in $A$, $B$, and $C$ are
  labeled $-1$, 0, and 1, respectively.

  Then the optimal $k$-median 2-means cluster is to put centers at $A$
  and $B$. Then the points at $A$ and $B$ pay zero, and the points at
  $C$ pay $10\cdot 10=100$ to connect to the center at $B$. So the
  total $k$-median score is 100.

  The LP rounding algorithm is guaranteed to output an
  11-approximation, so it must output a 2-clustering with score
  $\leq 110$ (it only works when $p\geq 2$, but we will modify the
  proof for $p\geq 1$ at the end).  Note, the centers must stay at $A$
  and $B$, because if a center is not at (wlog) $A$, then the center
  must be distance at least 1 away from $A$, which means the score of
  this clustering is $\geq 111$.  Now we know the LP algorithm is
  guaranteed to output centers at $A$ and $B$.  Then the clusters must
  be $A$ and $B\cup C$, because this assignment minimizes the flow in
  the graph the LP algorithm uses to assign points to clusters.
  Therefore, each cluster has $\leq 2$ labels, which can easily be
  classified using a linear separator.

  Now we consider the $k$-means++ algorithm.  First we calculate the
  probability that a center is placed in group $C$.  Note, there is
  zero probability that both centers are in the same group.  So this
  probability is the complement of the probability the centers fall in
  $A$ and $B$.  By a simple calculation, the probability is
  $1-\frac{112}{224}\cdot\frac{1\cdot 111}{1\cdot 111+11\cdot 1}-
  \frac{111}{224}\cdot\frac{1\cdot 112}{1\cdot 112+10\cdot 1}=.09016$.
  However, if a center is placed in group $C$, then the clusters will
  be $A\cup B$ and $C$, which means the $k$-means++ balancing
  heuristic will combine both clusters into a single cluster.  Then,
  there are 3 groups of points with different labels on a line, so a
  linear separator must classify at least one point incorrectly.
\end{proof}

The proof of Lemma \ref{lem:outperform} can be modified for
$k$-means. The probability that $k$-means outputs a bad clustering is
inversely proportional to the approximation factor of the LP
algorithm, but since the LP algorithm has constant-factor
approximation ratios, the probability is constant.  The proof can also
be modified for $p=2$ similar to the explanation in the experimental
evaluation.  Set $p=2$ and $k=4$, and the problem stays the same,
since the optimal clustering puts two centers at $A$ and two centers
at $B$.

\end{document}